\documentclass[mcgd]{ipart}

\usepackage[utf8]{inputenc} 
\usepackage[T1]{fontenc}    
\usepackage{hyperref}       
\usepackage{url}            
\usepackage{booktabs}       
\usepackage{amsfonts}       
\usepackage{nicefrac}       
\usepackage{microtype}      
\usepackage{graphicx}
\usepackage[square,numbers]{natbib}
\usepackage{doi}

\usepackage{caption}
\usepackage{xspace}
\usepackage{multirow}
\usepackage{algorithm}
\usepackage{algpseudocode}

\startlocaldefs
\DeclareMathOperator{\arcosh}{arcosh}

\DeclareMathOperator{\artanh}{artanh}

\makeatletter
\DeclareRobustCommand\onedot{\futurelet\@let@token\@onedot}
\def\@onedot{\ifx\@let@token.\else.\null\fi\xspace}
 
\def\ie{\emph{i.e}\onedot}

\def\etal{\emph{et al}\onedot}
\makeatother

\newtheorem{definition}{Definition}[section]
\newtheorem{proposition}{Proposition}[section]
\newtheorem{theorem}{Theorem}[section]
\newtheorem{corollary}{Corollary}[theorem]

\newtheorem{problem}{Problem}
\endlocaldefs

\pubyear{2025}
\volume{4}
\issue{2}
\firstpage{75}
\lastpage{139}

\begin{document}

\begin{frontmatter}

\title{Hyperbolic Optimal Transport}

\begin{aug}

\author{\fnms{Yan Bin} \snm{Ng}\ead[label=e1]{yanbng@cs.stonybrook.edu}}
\address{Stony Brook University\\\printead{e1}}
\author{\fnms{Xianfeng} \snm{Gu}\ead[label=e2]{gu@cs.stonybrook.edu}}
\address{Stony Brook University\\\printead{e2}}
\end{aug}

\begin{abstract}
The optimal transport (OT) problem aims to find the most efficient mapping between two probability distributions under a given cost function, and has diverse applications in many fields such as machine learning, computer vision and computer graphics. 
However, existing methods for computing optimal transport maps are primarily developed for Euclidean spaces and the sphere. 
In this paper, we explore the problem of computing the optimal transport map in hyperbolic space, which naturally arises in contexts involving hierarchical data, networks, and multi-genus Riemann surfaces.
We propose a novel and efficient algorithm for computing the optimal transport map in hyperbolic space using a geometric variational technique by extending methods for Euclidean and spherical geometry to the hyperbolic setting. 
We also perform experiments on synthetic data and multi-genus surface models to validate the efficacy of the proposed method.
\end{abstract}




\end{frontmatter}

\section{Introduction}
\label{introduction} 

Optimal transport (OT) has emerged as a powerful mathematical framework with applications across various fields, including economics \cite{galichon2018optimal}, machine learning \cite{peyre2019computational, montesuma2024recent}, computer vision \cite{bonneel2023survey}, and engineering \cite{chen2021optimal}. 
It seeks to find the most efficient way of transporting one distribution of mass to another, minimizing the transportation cost according to a given cost function. 
While much of the theory and applications of optimal transport have been developed in Euclidean space or the sphere, there is a notable gap in the literature when it comes to computing the optimal transport map in hyperbolic space. This is particularly useful in contexts involving hierarchical data, networks, and multi-genus Riemann surfaces, where hyperbolic geometry naturally arises.
 
Hyperbolic geometry, a non-Euclidean geometry characterized by constant negative curvature, has found numerous applications in areas like network theory, computer vision, and machine learning, where data often resides on curved manifolds. 
Notably, hyperbolic geometry allows us to embed any finite tree into a finite hyperbolic space, such that distances are preserved approximately \cite{gromov1987hyperbolic}, making it an ideal tool to model hierarchical structures efficiently.
Additionally, according to the Uniformization theorem (Theorem 4.4.1 \cite{jost2006compact}), the universal cover of a compact Riemann surface is conformally equivalent to the sphere, the complex plane or the Poincaré disk. 
This implies that if we can compute the optimal transport map in hyperbolic space, we can also compute optimal transport maps on multi-genus surfaces.
However, the unique properties of hyperbolic space, such as exponential growth of distances and a rich structure of geodesics, pose significant challenges in extending classical OT theory to this setting. 

This paper investigates the problem of computing the optimal transport map in hyperbolic space.  
Formally, we define the semi-discrete hyperbolic optimal transportation problem on the $m$-dimensional hyperbolic space $\mathbb{H}^m$ as follows:

\begin{problem}[Semi-discrete Hyperbolic Optimal Transport Problem]
Let $ M = \mathbb{H}^m/\Gamma $ be an $ m $-dimensional compact hyperbolic manifold where $\Gamma$ is a group of isometries on $\mathbb{H}^m$, $ \mu $ be a probability measure on $ M $ with a continuous density function $ d\mu = f d\sigma_M $ with respect to the Riemannian measure $ \sigma_M $, and $ \nu = \sum_{i=1}^k \nu_i \delta_{p_i} $ be a discrete measure defined on the point set $ \{p_1, p_2, \dots, p_k\} \subset M $, satisfying $ \min_{1 \leq i \leq k} \nu_i > 0 $ and $ \mu(M) = \sum_{i=1}^k \nu_i $. 
Given a hyperbolic cost function $c: M \times M \to [0,\infty)$ defined as $c(x,y) = \ln \cosh d_M(x,y)$, where $d_M$ is the geodesic distance function (Riemannian metric function) of $M$, solve
\[
\inf \left\{ \int_{M} \ln \cosh d_M(x, T(x)) \, d\mu(x) : T_\# \mu = \nu  \right\}
\]
where $T_\# \mu$ is defined by $T_\# \mu(p_i) := \mu(T^{-1}(p_i)), \, \forall p_i \in M$, \ie $\mu(T^{-1}(p_i)) = \nu_i$.
\end{problem}

Our approach is based on \cite{zeng2022geometric} and extends results from Euclidean \cite{gu2013variational} and spherical \cite{cui2019spherical} optimal transport methods, adapting them to the hyperbolic setting by generalizing the classical Minkowski problem to the hyperbolic case using a geometric variational method based on Alexandrov's convex geometry theory \cite{alexandrov2005convex}. 

The geometric variational theory of optimal transport problems studies the intrinsic connection between the optimal transport problem in the semi-discrete format and the discrete Minkowski problem in convex differential geometry. 
In fact, there are two measures on a convex surface: the surface measure and the Gaussian curvature measure. 
Based on Alexandrov's theory, it can be shown that the geometry of the convex surface is determined by the optimal transport map between these two measures. Therefore, solving the optimal transport map can be reduced to the problem of constructing the convex surface.

The goal of this paper is to extend the geometric variational principles of optimal transport problems in Euclidean space and on the sphere \cite{gu2013variational, cui2019spherical}, to hyperbolic manifolds, using variational methods, convex differential geometry, and hyperbolic geometry theory.

We summarize the contributions of this work as follows:
\begin{itemize}
\item We formulate the hyperbolic optimal transport problem by generalizing the Euclidean and spherical semi-discrete optimal transport problem to the hyperbolic setting.
\item We propose a novel and efficient method to compute the optimal transport map in hyperbolic space via a geometric variational approach based on the Minkowski problem.
\item We evaluate the efficacy and efficiency of our proposed method by experiments on toy data and multi-genus surface models.
\end{itemize}

The rest of the paper is organized as follows: section 2 discusses a brief background on optimal transport theory in Euclidean space and key geometrical aspects of hyperbolic space relevant to optimal transport. Section 3 provides a review of current research on computing optimal transport maps and various applications using hyperbolic geometry. In section 4, we present our algorithms for computing the hyperbolic optimal transport map, followed by numerical experiments in section 5. Finally, we conclude with a discussion of future directions and open problems in this area.
\section{Background}
\label{background} 

\subsection{Optimal Transport}
In this section, we present the background of optimal transport theory that provides the theoretical framework for our work. More details can be found in \cite{santambrogio2015optimal}.

\subsubsection{Monge's Formulation}
Consider probability measures $\mu$ and $\nu$ defined on metric spaces $X$ and $Y$ respectively, and a cost function $c: X \times Y \to [0, \infty)$ defined on the product space $X \times Y$. The objective of Monge's formulation is to solve for a map $T$ that attains the infimum:

\begin{equation} 
\inf \left\{ \left. \int_X c(x,T(x)) \, \mathrm{d}\mu(x) \right| T_\#\mu = \nu \right\} ,
\label{eq:monge}
\end{equation} 

where $T_\#\mu$ denotes the pushforward of $\mu$ and is defined as $T_\#\mu(E) = \mu(T^{-1}(E))$ for every $E \subset Y$. The transportation map that achieves the infimum is called the optimal transport map. Monge's formulation is difficult to solve because of the measure preserving constraint $T_\#\mu = \nu$. In particular, it does not allow the splitting of mass from $X$ to $Y$ by the transportation map. For example, when $\mu$ is a Dirac measure and $\nu$ is not, then the problem does not have a solution.

\subsubsection{Kantorovich's Formulation}

Kantorovich \cite{kantorovich1948on} attempted to solve the problem by relaxing the constraint to allow splitting of mass. Kantorovich's formulation aims to find a transportation plan $\gamma$, defined as a probability measure on the product space $X \times Y$, to realize the following:
\begin{equation} 
\text{(KP)} \qquad \inf \left\{ \left. \int_{X \times Y} c(x,y) \, \mathrm{d}\gamma(x,y) \right| \gamma \in \Pi(\mu,\nu) \right\},
\label{eq:kantorovich}
\end{equation} 

where $\Pi(\mu,\nu)$ denotes the set of transportation plans satisfying the constraints $(\pi_x)_\# \gamma = \mu$ and $(\pi_y)_\# \gamma = \nu$ where $\pi_x$ and $\pi_y$ are projections from $X \times Y$ to $X$ and $Y$ respectively. In this formulation, $\gamma(A \times B)$ gives the amount of mass moving from $A$ to $B$ and allows for the splitting of mass. This is a relaxation of Monge's formulation.

\subsubsection{Dual Form}

Kantorovich's formulation is usually solved in its dual form. Let $\varphi: X \to \mathbb{R}$ and $\psi: Y \to \mathbb{R}$ be bounded and continuous functions. We consider the following optimization problem:
\begin{equation} 
\text{(DP)} \qquad \sup \left\{ \left. \int_X \varphi(x) \, \mathrm{d}\mu(x) + \int_Y \psi(y) \, \mathrm{d}\nu(y) \right| \varphi(x) + \psi(y) \leq c(x,y) \right\},
\label{eq:dual}
\end{equation}

where the supremum is taken over functions $\varphi \in L^1 (\mu)$ and $\psi \in L^1 (\nu)$. 

In order to see what the solution space of (DP) looks like, we introduce the concept of $c$-transforms.

\begin{definition}
Given a function $\chi: X \to \overline{\mathbb{R}}$, we define its $c$-transform $\chi^c: Y \to \overline{\mathbb{R}}$ as follows:
\begin{equation}
\chi^c(y) =  \inf_{x \in X} \{ c(x,y) - \chi(x) \}.
\end{equation}

Similarly, we define the $\overline{c}$-transform of $\zeta: Y \to \overline{\mathbb{R}}$ as:
\begin{equation}
\zeta^{\overline{c}}(x) =  \inf_{y \in Y} \{ c(x,y) - \zeta(y) \}.
\end{equation}
\end{definition}

We say that a function $\varphi$ defined on $X$ is $c$-concave if there exists $\zeta : Y \to \overline{\mathbb{R}}$ such that $\varphi = \zeta^{\overline{c}}$, denoted as $\varphi \in c - \text{conc}(X)\). Similarly, a function $\psi$ defined on $Y$ is $\overline{c}$-concave if there exists $\chi : X \to \overline{\mathbb{R}}$ such that $\psi = \chi^c$, denoted as $\psi \in \overline{c} - \text{conc}(Y)$.

We then have the existence result:
\begin{proposition}[Proposition 1.11 \cite{santambrogio2015optimal}]
Suppose that X and Y are compact and c is continuous. Then there exists a solution $(\varphi,\psi)$ of (DP) that has the form $\varphi \in c - \text{conc}(X), \psi \in \overline{c} - \text{conc}(Y)$ and $\psi = \varphi^c$. In particular,
\begin{equation}
\label{eq:dp-exist}
\begin{split}
\max \text{(DP)} &= \max_{\varphi \in c - \text{conc}(X)} \int_X \varphi d\mu + \int_Y \varphi^c d\nu \\
&= \max_{\psi \in \overline{c} - \text{conc}(Y)} \int_X \psi^{\overline{c}} d\mu + \int_Y \psi d\nu.
\end{split}
\end{equation}
\end{proposition}

The next theorem shows that (KP) and (DP) are equivalent.
\begin{theorem}[Theorem 1.39 \cite{santambrogio2015optimal}]
\label{thm:duality}
Suppose that $X$ and $Y$ are Polish spaces and $c: X \times Y \to \mathbb{R}$ is uniformly continuous and bounded. Then (DP) admits a solution $(\varphi,\varphi^c)$ and max(DP) = min(KP).
\end{theorem}

\subsubsection{Brenier's Theorem}
In the case of the quadratic cost function $c(x,y) = \frac{1}{2}|x - y|^2$, Brenier's Theorem \cite{brenier1991polar} claims the existence of an optimal transport map $T$ that can be written as the gradient of a convex function $u$, \ie  $T^*(x) = x - \nabla \varphi^*(x) = \nabla (\frac{x^2}{2} - \varphi^*(x)) = \nabla u(x)$. It can also be shown that if $\mu = f(x) \mathrm{d}x$ and $\nu = g(y) \mathrm{d}y$, then solving for the convex function $u$ is equivalent to solving the Monge-Ampère equation:

\begin{equation}
\det \left(D^2 u(x) \right) = \frac{f(x)}{g(\nabla u(x))}.
\label{eq:brenier}
\end{equation}

\subsubsection{Wasserstein Distance}
Wasserstein distance is used as a way to measure the distance between two probability distributions. 
Let $\mu$ and $\nu$ be probability measures defined on the metric space $X$ and $d$ be the metric on $X$. 
Then the Wasserstein $p$-distance between $\mu$ and $\nu$ for some $p \geq 1$ can be defined as:
\begin{equation}
W_p(\mu, \nu) = \inf_{\gamma \in \Pi(\mu,\nu)} \left\{ \int_{X \times X} d(x,y)^p \, \mathrm{d}\gamma \right\}^{1/p}.
\label{eq:wasserstein}
\end{equation}

For $p=1$, this is just the total cost of the optimal transport plan of Kantorovich's formulation with cost function $c(x,y) = d(x,y)$.

\subsection{Hyperbolic Geometry}
Hyperbolic geometry is a non-Euclidean geometry characterized by the rejection of the parallel postulate of Euclidean geometry, which states that through any point not on a given line, there exists exactly one parallel line to the given line. In contrast, hyperbolic geometry asserts that through any point not on a given line, there are infinitely many lines that do not intersect the given line. This leads to a fundamentally different structure and a rich mathematical framework, which has become important in various fields including mathematics, physics, and computer science.

The study of hyperbolic geometry can be approached through several models, which provide various perspectives and tools for visualizing and analyzing the geometry. Two commonly used models are the hyperboloid model, and the Poincaré disk model. Each of these offers different conceptualizations of the hyperbolic plane, yet they all share the same underlying geometry.

\subsubsection{Hyperboloid Model}
The hyperboloid model of hyperbolic geometry is closely tied to the geometry of Minkowski spacetime, which can be understood in terms of the Lorentzian inner product. To define the hyperboloid model, we first define the Lorentzian inner product, which serves as the foundation for the geometry of Minkowski space.

The Lorentzian inner product is defined as follows:
\begin{equation}
\langle \mathbf{x}, \mathbf{y} \rangle_H = x_1 y_1 + x_2 y_2 + \cdots + x_m y_m - x_{m+1} y_{m+1} \quad  \forall \mathbf{x}, \mathbf{y} \in \mathbb{R}^{m+1}
\label{eq:lorentzian}
\end{equation}

This inner product can take positive, negative, or zero values depending on the relationship between the vectors. A vector with a negative inner product with itself is referred to as time-like, a vector with a positive inner product is space-like, and a vector with zero inner product is light-like.

Minkowski spacetime in $m+1$ dimensions is the vector space $\mathbb{R}^{m+1}$ equipped with the Lorentzian inner product and is denoted as $\mathbb{R}^m_1$.
In ${m+1}$-dimensional Minkowski space, the future cone is a region that consists of all possible future-directed, time-like vectors emanating from the origin.
Formally, the future cone as be defined as:
\begin{equation}
C_f = \{ \mathbf{x} \in \mathbb{R}^m_1 : \langle \mathbf{x}, \mathbf{x} \rangle _H < 0, x_{m+1} > 0 \}
\label{eq:future-cone}
\end{equation}

The hyperboloid model can be described as a pseudo-sphere in Minkowski space-time and consists of the set of points on the upper sheet of a one-sheeted hyperboloid that lies within the future cone, as shown in Figure \ref{fig:hyperboloid}.

The $m$-dimensional hyperbolic space $\mathbb{H}^m$ is represented by the set of points on the hyperboloid given by:
\begin{equation}
\mathbb{H}^m = \{ \mathbf{x} \in \mathbb{R}^m_1 : \langle \mathbf{x}, \mathbf{x} \rangle _H = -1, x_{m+1} > 0 \}
\label{eq:hyperboloid}
\end{equation}

\begin{figure}[h]
\begin{center}
\fbox{\rule{0pt}{2in}
\includegraphics[width=0.8\linewidth]{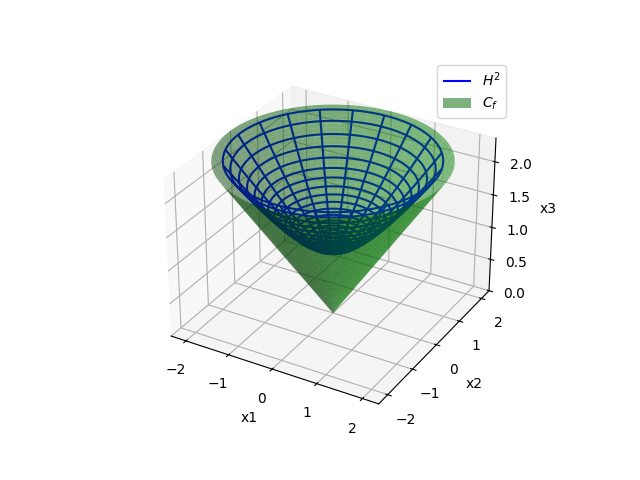}}
\end{center}
   \caption{Hyperboloid and Future cone}
\label{fig:hyperboloid}
\end{figure}

The tangent space is then given by:
\begin{equation}
T_\mathbf{x} \mathbb{H}^m = \{ \mathbf{v} \in \mathbb{R}^m_1 : \langle \mathbf{x}, \mathbf{v} \rangle _H = -1 \}
\label{eq:hyperboloid-tangent}
\end{equation}

The Riemannian metric on $\mathbb{H}^m$ can be defined by restricting Lorentzian inner product to the tangent space as follows:
\begin{equation} 
ds^2 = dx_1^2 + dx_2^2 + \cdots + dx_m^2 - dx_{m+1}^2   
\label{eq:hyperboloid-metric}
\end{equation}

To find the hyperbolic distance between two points on the hyperboloid, we use the Lorentzian inner product, which provides a measure of distance in the hyperboloid model.

\begin{equation} 
d_H(\mathbf{x}, \mathbf{y}) = \arcosh(-\langle \mathbf{x}, \mathbf{y} \rangle_H) \quad \forall \mathbf{x}, \mathbf{y} \in \mathbb{H}^m
\label{eq:hyperboloid-distance}
\end{equation}

The geodesics in the hyperboloid model, which are the equivalent of straight lines in hyperbolic space, correspond to intersections of the hyperboloid with planes passing through the origin of Minkowski space. These geodesics can be interpreted as segments of hyperbolas or straight lines in Minkowski space.

\subsubsection{Poincaré Disk Model}
The Poincaré Disk Model provides another representation of hyperbolic space within the unit disk in Euclidean space. This model is particularly effective for visualizing hyperbolic space in low-dimensional settings.

In the Poincaré disk model, the hyperbolic space is represented as the open unit disk $\mathbb{D}^m \subset \mathbb{R}^m$:
\begin{equation}  
\mathbb{D}^m = \{ \mathbf{x} \in \mathbb{R}^m : \| \mathbf{x} \| \leq 1 \} 
\label{eq:poincare}
\end{equation}

where $\|\mathbf{x}\|$ is the Euclidean norm.
The Poincaré disk model uses a Riemannian metric that differs from the Euclidean metric to reflect the intrinsic curvature of hyperbolic space. The metric is given by:
\begin{equation} 
ds^2 = \frac{4 \|d\mathbf{x}\|^2}{(1 - \|\mathbf{x}\|^2)^2}
\label{eq:poincare-metric}
\end{equation}

where $\|d\mathbf{x}\|^2$ is the usual Euclidean metric. 
The metric effectively scales distances near the boundary of the disk, reflecting the fact that, in hyperbolic geometry, points close to the boundary are infinitely far apart in terms of hyperbolic distance.

The geodesics on the Poincaré disk are represented by segments of Euclidean circles orthogonal to the boundary of the disk, or straight lines passing through the origin.

Given two distinct points $p$ and $q$ on the disk, we can find a unique geodesic connecting them that intersects the disk boundary at ideal points $a$ and $b$. 
The hyperbolic distance between $p$ and $q$ is then given by:
\begin{equation} 
d_H(p,q) = \ln \frac{|a-q||p-b|}{|a-p||q-b|}
\label{eq:poincare-distance}
\end{equation}

where $|\cdot|$ represents the Euclidean distance on the disk.

In the special case where one of the points is the origin and the Euclidean distance between the points is $r$, the hyperbolic distance can be written as:
\begin{equation} 
\ln \left( \frac{1 + r}{1 - r} \right) = 2 \artanh r
\label{eq:poincare-distance2}
\end{equation}

The Poincaré disk model is related to the hyperboloid model through a projection map onto the hyperplane $x_{m+1}=0$ by intersecting the hyperboloid with a line drawn through $(0,\cdots,0,-1)$, as shown in Figure \ref{fig:projection}.
Given Cartesian coordinates $(x_1, \cdots, x_m, x_{m+1})$ on the hyperboloid and $(y_1, \cdots, y_m)$ on the Poincaré disk, the relationship between the points are:
\begin{equation}
y_i = \frac{x_i}{1 + x_{m+1}} \qquad i = 1..m
\label{eq:poincare-hyperboloid}
\end{equation}

\begin{equation}
(x_i, x_{m+1}) = \frac{(2y_i, 1+\sum y_i^2)}{1-\sum y_i^2} \qquad i = 1..m
\label{eq:hyperboloid-poincare}
\end{equation}

The two models are equivalent in terms of their underlying hyperbolic geometry, and the transformation between them preserves the hyperbolic structure, but not the Euclidean distances.

\begin{figure}[h]
\begin{center}
\fbox{\rule{0pt}{2in}
\includegraphics[width=0.8\linewidth]{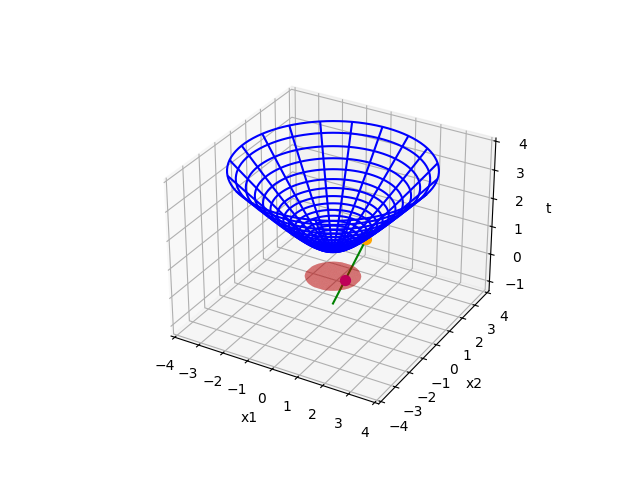}}
\end{center}
   \caption{Relating the Hyperboloid model and Poincaré disk}
\label{fig:projection}
\end{figure}

\section{Related Work}
\label{related-work}

\subsection{Optimal Transportation}
Traditional algorithms for computing optimal transport maps, such as those based on Kantorovich's theory \cite{kantorovich1948on, kantorovich2006problem}, Brenier's theorem \cite{brenier1991polar}, or the Benamou and Brenier numerical method \cite{benamou2000computational} have demonstrated excellent theoretical properties but are often computationally expensive, especially for high-dimensional or large-scale problems. 
Recent research has focused on developing novel algorithms and computational methods to make OT more feasible in practical applications. 
Note that there are several works that computes the OT cost or Wasserstein distance such as \cite{arjovsky2017wasserstein} and \cite{xie2020fast} but those do not compute the OT map explicitly and are not discussed in this section.

A key advancement in making OT computation more efficient is the introduction of entropic regularization to the classical optimal transport problem. 
This regularization allows the problem to be transformed into a form that can be solved using Sinkhorn-Knopp’s matrix scaling algorithm, significantly speeding up computation \cite{cuturi2013sinkhorn}.
Another similar approach is the work \cite{pooladian2023minimax} that leverages a computationally efficient estimator based on entropic optimal transport to estimate the OT map between two distributions. 
The inclusion of entropy regularization simplifies the computation and provides a robust solution in high-dimensional settings. 
However, this approach comes at the cost of precision and does not guarantee the existence of an OT map. 

Another line of research to overcome the problem of computational complexity is sliced optimal transport, where high-dimensional optimal transport problems are reduced to solving 1D optimal transport problems over projections on random one dimensional lines \cite{rabin2011wasserstein, bonneel2015sliced, liutkus2019sliced}.
Unlike entropic regularization, this approach does not converge to the true OT problem, and the result does not converge to the true solution even as the number of projection increases. 
Thus, while this approach offers substantial improvements in computational efficiency, it does so at the expense of a reduction in accuracy.

The use of neural networks to compute OT maps has recently gained attention as a way to circumvent the computational bottlenecks of classical methods. Makkuva \etal \cite{makkuva2020optimal} proposed a method to estimate the optimal transport map by solving the dual optimization problem using input convex neural networks (ICNNs) \cite{amos2017input}. 
Building on this work, Korotin \etal \cite{korotin2021wasserstein} extended this approach to non-minimax optimization, achieving better scaling and faster convergence rates. 
Rout \etal \cite{rout2021generative} proposed a neural-based method to compute the optimal transport map for generative modeling directly in the ambient space for the Wasserstein-2 cost.
Similarly, the work \cite{gushchin2023entropic} focuses on neural algorithms for computing the entropic OT plan between continuous probability distributions accessed by samples.
These neural network-based methods are becoming increasingly popular due to their flexibility and ability to generalize across various domains, but are limited by their poor expressiveness and scalability \cite{korotin2021neural}.

An alternative approach to neural network-based method is the use of convex geometry and Alexandrov theory. AE-OT \cite{an2019ae} integrates an autoencoder to map an input image into a latent representation and then applies a geometric computational method based on \cite{gu2013variational} to compute the OT map from the Gaussian distribution to the latent distribution. 
This framework allows for the efficient computation of transport maps in settings where the underlying distributions are highly structured, such as in generative models. 
AE-OT-GAN \cite{an2020ae} extends this approach by incorporating Generative Adversarial Networks (GANs), resulting in improved quality of generated samples, and providing a powerful tool for image generation and transformation tasks.

While much of the research on OT focuses on Euclidean spaces, optimal transport on non-Euclidean spaces, specifically the sphere, has also been studied. 
The spherical OT problem has been addressed by several authors in \cite{cui2019spherical} who use a geometric variational approach to solve OT on the sphere, employing convex energy minimization techniques. 
This approach ensures that the solution respects the geometry of the sphere, leading to accurate transport maps in spherical coordinates. 
In a related study, a PDE-based approach for solving OT on the sphere was developed, focusing on squared geodesic and logarithmic costs \cite{hamfeldt2022convergence}. 
Additionally, recent work by Quellmalz \etal \cite{quellmalz2023sliced} introduces novel transforms, such as the vertical slice transform and the semicircle transform, for optimal transport on the 2-sphere, leading to varied optimal transport solutions for spherical geometry.

While significant progress has been made in computing OT maps in Euclidean and spherical geometry, challenges remain in extending these methods to handle hyperbolic geometry while remaining efficient and scalable. 

\subsection{Hyperbolic Geometry}
Recently, hyperbolic geometry have emerged as a promising framework for handling hierarchical data, which is prevalent in many applications, such as natural language processing (NLP), computer vision, and graph-based data. 
Hyperbolic geometry offers advantages in representing tree-like structures over traditional Euclidean-based approaches, and several methods have been proposed for computing representations in this non-Euclidean space.

One of the foundational works in this area is \cite{nickel2017poincare} by Nickel \etal, which proposes embedding symbolic data into hyperbolic space, specifically the n-dimensional Poincaré ball, to capture both hierarchy and similarity. This embedding method is highly effective for learning hierarchical relationships, and it has been shown to outperform Euclidean embeddings in tasks requiring the preservation of hierarchical structure. 

In the domain of graph embeddings, hyperbolic geometry has also gained attention due to its ability to represent complex networks more naturally than Euclidean spaces, such as the Multi-Relational Poincaré model (MuRP) proposed by Balazevic \etal \cite{balazevic2019multi}. 
MuRP embeds multi-relational graph data into the Poincaré ball, allowing the model to capture multiple simultaneous hierarchies and outperform traditional Euclidean embeddings in tasks involving multi-relational data.

In addition to these neural embedding methods, Alvarez-Melis \etal \cite{alvarez2020unsupervised} introduced a novel approach for unsupervised hierarchy matching using OT over hyperbolic spaces. 
This method utilizes optimal transport in hyperbolic space to align hierarchical data, such as WordNet or ontologies, and outperforms traditional Euclidean-based alignment techniques. 

The integration of hyperbolic geometry with deep learning models has also been explored in the context of neural networks. Ganea \etal \cite{ganea2018hyperbolic} developed hyperbolic versions of deep learning tools, including logistic regression and neural networks, using the Möbius gyrovector space formalism and the Poincaré model. 
In the same vein, Khrulkov \etal \cite{khrulkov2020hyperbolic} demonstrated that hyperbolic image embeddings provide a better alternative to Euclidean embeddings for computer vision tasks such as image classification, retrieval, and few-shot learning. 

Further advancements in \cite{ermolov2022hyperbolic} by Ermolov \etal introduced a hyperbolic-based model for metric learning. 
By using a vision transformer with output embeddings mapped to hyperbolic space, this method enables better learning of spatial hierarchies in visual data. 

In the domain of multi-modal learning, Desai \etal \cite{desai2023hyperbolic} introduced MERU, a contrastive model that generates hyperbolic image-text representations to better capture the hierarchical relationships between visual and linguistic concepts. 
This method leverages the natural geometry of hyperbolic space to align visual and textual data in a shared space, outperforming Euclidean-based methods in tasks like image classification and image-text retrieval. 

Lastly, Chami \etal \cite{chami2020low} proposed a class of hyperbolic knowledge graph embedding models that combine hyperbolic reflections and rotations with attention mechanisms to simultaneously capture both hierarchical and logical patterns in knowledge graphs.

\section{Hyperbolic Optimal Transport}
\label{method}

\subsection{Hyperbolic Legendre Duality}
In this section, we first introduce the concepts and basic properties of Fuchsian convex bodies and Gauss curvature measures, then establish the theory of hyperbolic Legendre duality, and finally analyze the combinatorial structure of $\Gamma$-convex polyhedra and their duals.

\subsubsection{Fuchsian convex bodies and Gauss curvature measure}
Let $I(\mathbb{R}^m_1)$ be the isometry group of the $(m+1)$-dimensional Minkowski spacetime $\mathbb{R}^m_1$, which is the group of linear mappings that preserve the Lorentzian inner product. 
Let $I^+(\mathbb{R}^m_1)$ be a subgroup of $I(\mathbb{R}^m_1)$, where each element of the subgroup preserves the future light cone $C_f$. 
From Proposition A.2.4 in \cite{benedetti1992lectures}, we know that the isometry group $I(\mathbb{H}^m)$ of the $m$-dimensional hyperbolic space $\mathbb{H}^m$ is the restriction of $I^+(\mathbb{R}^m_1)$ to $\mathbb{H}^m$, so $I(\mathbb{H}^m) \cong I^+(\mathbb{R}^m_1)$. 
Let $\mathcal{F}$ be the set of all discrete, compact, free subgroups of $I^+(\mathbb{R}^m_1)$ acting on $\mathbb{H}^m$. Then, for any $\Gamma \in \mathcal{F}$, the quotient manifold $\mathbb{H}^m / \Gamma$ is a compact $m$-dimensional hyperbolic manifold.

\begin{definition}
\label{def:fuchsian}
Let $\mathcal{C}$ be a closed convex proper subset of the future cone $C_f$, and suppose there exists a subgroup $\Gamma \in \mathcal{F}$ such that $\Gamma \cdot \mathcal{C} = \mathcal{C}$. Then, $\mathcal{C}$ is called a Fuchsian convex body in the Minkowski spacetime $\mathbb{R}^m_1$. 

The boundary of a Fuchsian convex body is called the Fuchsian convex surface, denoted as $\partial \mathcal{C}$. 

A Fuchsian convex body that is invariant under the action of the group $\Gamma$ is called a $\Gamma$-convex body.
\end{definition}

A $\Gamma$-convex body $\mathcal{P}$ is called a $\Gamma$-convex polyhedron if there exist $k$ points $x_1, x_2, \dots, x_k \in \mathbb{H}^m$ (which are not pairwise collinear) and $k$ positive real numbers $\rho_1, \rho_2, \dots, \rho_k$ such that:
\begin{equation}
\label{eq:polyhedron}
\mathcal{P} = \left\{ z \in C_f \mid \langle z - \rho_i^{-1} \gamma x_i, \gamma x_i \rangle_H \leq 0, \ \forall 1 \leq i \leq k, \ \forall \gamma \in \Gamma \right\}.
\end{equation}

Fuchsian convex bodies are the generalization of convex bodies in Euclidean space to Minkowski space-time. Each pseudo-sphere $\mathbb{H}^m_r = \{ x \in \mathbb{R}^m_1 \mid \langle x, x \rangle_H = -r, x_{m+1} > 0 \}$ (where $r > 1$) is a Fuchsian convex body under the action of the group $I^+(\mathbb{R}^m_1)$. Given a finite set of points in the future cone $C_f$ and a subgroup $\Gamma \in \mathcal{F}$, the convex hull of the orbits of these points under the action of $\Gamma$ is a $\Gamma$-convex polyhedron.

\begin{figure}[h]
\begin{center}
\fbox{\rule{0pt}{2in}
\includegraphics[width=0.8\linewidth]{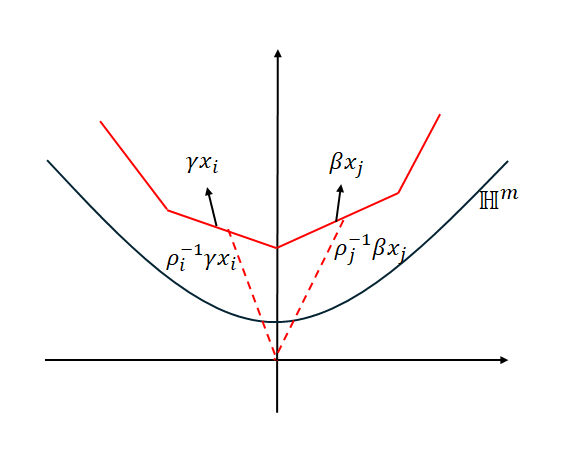}}
\end{center}
   \caption{$\Gamma$-convex polyhedron}
\label{fig:convex-polyhedron}
\end{figure}

Note that the supporting hyperplane of a set $A$ at a point $x \in A$ is a hyperplane $H$ that satisfies $x \in A \cap H$ and $A$ is entirely contained on one side of $H$. The following proposition shows that a Fuchsian convex body can be obtained as the upper envelope of all of its supporting hyperplanes.

\begin{proposition}[\cite{fillastre2013fuchsian} \cite{bertrand2014prescription}]
\label{prop:fuchsian}
Let $\mathcal{C}$ be a Fuchsian convex body. Then:
\begin{enumerate}
    \item $\mathcal{C}$ is not contained in any hyperplane with positive codimension.
    \item Every point on the boundary $\partial \mathcal{C}$ has a supporting hyperplane.
    \item All supporting hyperplanes of $\mathcal{C}$ are space-like.
    \item $\mathcal{C}$ is contained in the future side of any supporting hyperplane.
    \item For any $x \in \mathcal{C}$ and $\lambda \geq 1$, $\lambda x \in \mathcal{C}$.
\end{enumerate}
\end{proposition}

\begin{proposition} [Proposition 2.15 \cite{bertrand2014prescription}]
\label{prop:projection}
Let $\mathcal{C}$ be a Fuchsian convex body. Then the radial projection $p: \partial \mathcal{C} \to \mathbb{H}^m$ is a homeomorphism, where
\begin{equation}
\label{eq:projection}
\rho(x) = \sqrt{-\langle p^{-1}(x), p^{-1}(x) \rangle_H}, \quad \forall x \in \mathbb{H}^m.
\end{equation}
\end{proposition}

From Proposition \ref{prop:projection}, it follows that $p^{-1}(x) = \rho(x) x$. Therefore, the radial function $\rho$ is continuous on $\mathbb{H}^m$, and it is greater than zero and less than infinity. In addition, the radial function $\rho$ is invariant under the action of the group $\Gamma$, \ie
\begin{equation}
\label{eq:radial-invariance}
\rho(\gamma x) = \rho(x), \quad \forall x \in \mathbb{H}^m, \ \forall \gamma \in \Gamma.
\end{equation}

Let $\mathcal{C}$ be a $\Gamma$-convex body. We say that $y \in \mathbb{H}^m$ is the inward unit normal vector of $\mathcal{C}$ at the point $z_0 \in \partial \mathcal{C}$ if $y$ is the normal vector of a supporting hyperplane at $z_0$. From Proposition \ref{prop:fuchsian}, we know that for $z_0 = \rho(x_0) x_0 \in \partial \mathcal{C}$, $\langle z - z_0, y \rangle_H \leq 0, \ \forall z \in \partial \mathcal{C}$,
which implies that
\begin{equation}
\rho(x_0) \langle x_0, y \rangle_H \geq \rho(x) \langle x, y \rangle_H, \quad \forall x \in \mathbb{H}^m.
\end{equation}

Conversely, if the above inequality holds, then the hyperplane
\begin{equation}
\label{eq:hyperplane}
\left\{\frac{\langle x_0, y \rangle_H}{\langle x, y \rangle_H} \rho(x_0) x \mid x \in \mathbb{H}^m \right\}
\end{equation}

is the supporting hyperplane at $z_0 = \rho(x_0) x_0$ with the inward unit normal vector $y$.

Based on the above, we define the set of subnormal vectors for a point $x \in \mathbb{H}^m$ as the set of all inward unit normal vectors corresponding to $z = \rho(x) x$.

\begin{definition}
\label{def:subnormal}
Let $\mathcal{C}$ be a $\Gamma$-convex body, and $\partial \mathcal{C} = \{ \rho(x) x : x \in \mathbb{H}^m \}$. The set of subnormal vectors for a point $x \in \mathbb{H}^m$ is defined as
\begin{equation}
\label{eq:subnormal}
\partial \rho(x) = \{ y \in \mathbb{H}^m \mid \rho(x) \langle x, y \rangle_H \geq \rho(z) \langle z, y \rangle_H, \quad \forall z \in \mathbb{H}^m \}.
\end{equation}

The Gauss map of $\mathcal{C}$ is defined as the multivalued map $\mathcal{G} = \partial \rho \circ p: \partial \mathcal{C} \to \mathbb{H}^m$, that is,
\begin{equation}
\label{eq:gauss-map}
\mathcal{G}(z) = \partial \rho(x), \quad \forall z = \rho(x) x \in \partial \mathcal{C}.
\end{equation}
\end{definition}

Note that $\langle \rho(z) z - \rho(x) x, y \rangle_H \leq 0$ if and only if $\langle \rho(z) \gamma z - \rho(x) \gamma x, \gamma y \rangle_H \leq 0, \quad \forall \gamma \in \Gamma$. Therefore, $\partial \rho(\gamma x) = \gamma \partial \rho(x)$. Thus, we have the following proposition.

\begin{proposition}[Lemma 2.19 \cite{bertrand2014prescription}]
\label{prop:gauss-curvature}
Let $\mathcal{C}$ be a $\Gamma$-convex body. For any Borel set $U \subset \partial \mathcal{C}$, both $p(U)$ and $\mathcal{G}(U)$ are Borel subsets of $\mathbb{H}^m$. 

Moreover, the map $\pi_{\Gamma} \circ \partial \rho: \mathbb{H}^m \to \mathbb{H}^m / \Gamma$ is well-defined and invariant under the action of the group $\Gamma$, where $\pi_{\Gamma}: \mathbb{H}^m \to \mathbb{H}^m / \Gamma$ is the covering map. 

Therefore, $\pi_{\Gamma} \circ \partial \rho$ induces a map $\mathcal{G}_{\Gamma}: \mathbb{H}^m / \Gamma \to \mathbb{H}^m / \Gamma$ on the hyperbolic manifold $\mathbb{H}^m / \Gamma$, such that if $U \subset \mathbb{H}^m / \Gamma$ is a Borel set, then $\mathcal{G}_{\Gamma}(U)$ is also a Borel set.
\end{proposition}

In order to define Gauss curvature measure, we also require the following proposition.
\begin{proposition}[Lemma 2.21 \cite{bertrand2014prescription}]
Let $C$ be a Fuchsian convex set and $\Gamma$ its related subgroup of isometries. Then, there exists a unique canonical Borel measure $\sigma_{\mathbb{H}^m / \Gamma}$ on $\mathbb{H}^m / \Gamma$, and its total mass equals $\sigma_{\mathbb{H}^m}(D)$, where $\sigma_{\mathbb{H}^m}$ is the Riemannian measure on $\mathbb{H}^m$ and $D$ is any convex, locally finite, fundamental domain for $\Gamma$. Subsequently, $\sigma_{\mathbb{H}^m}(D)$ is denoted by $\text{Vol}(\mathbb{H}^m / \Gamma)$.
\end{proposition}

\begin{definition}
\label{def:gauss-curvature}
Let $\mathcal{C}$ be a $\Gamma$-convex body. The Gauss curvature measure $\mu_{\mathcal{C}}$ of $\mathcal{C}$ is defined as
\begin{equation}
\label{eq:gauss-curvature}
\mu_{\mathcal{C}}(U) = \sigma_{\mathbb{H}^m / \Gamma}(\mathcal{G}_{\Gamma}(U)), \quad \forall \text{ Borel set } U \subset \mathbb{H}^m / \Gamma.
\end{equation}
\end{definition}

From Proposition \ref{prop:gauss-curvature}, it follows that the Gauss curvature measure is well-defined on $\Gamma$-convex bodies. In particular,
$\mu_{\mathcal{C}}(\mathbb{H}^m / \Gamma) = \sigma_{\mathbb{H}^m / \Gamma}(\mathbb{H}^m / \Gamma) = \text{Vol}(\mathbb{H}^m / \Gamma).$
If $\mathcal{C}$ is a $\Gamma$-convex polytope, then $\mu_{\mathcal{C}}$ is discrete, and its support set is the set of all vertices of $\mathcal{C}$.

The Minkowski problem in Minkowski space investigates whether there exists a Fuchsian body in Minkowski space whose Gauss curvature measure is equal to a given probability measure.

\subsubsection{Hyperbolic Legendre Dual}
\begin{definition}
\label{def:legendre-dual}
Let $\rho : \mathbb{H}^m \to (0, +\infty)$ be a positive hyperbolic function. Its hyperbolic Legendre dual is a hyperbolic function $\rho^* : \mathbb{H}^m \to (0, +\infty)$, defined as
\begin{equation}
\label{eq:legendre-dual}
\rho^*(y) = \sup_{x \in \mathbb{H}^m} \frac{-1}{\rho(x) \langle x, y \rangle_H}, \quad \forall y \in \mathbb{H}^m.
\end{equation}

Let $\mathcal{C}$ be a subset of $C_f$, and its hyperbolic Legendre dual $\mathcal{C}^*$ is a subset of $C_f$, defined as
\begin{equation}
\label{eq:legendre-dual-set}
\mathcal{C}^* = \{ y \in C_f \mid \langle y, z \rangle_H \leq -1, \, \forall z \in \mathcal{C} \}.
\end{equation}
\end{definition}

\begin{proposition}[Lemma 2.6 \cite{fillastre2013fuchsian}]
\label{prop:support-plane}
Let $\mathcal{C}$ be a $\Gamma$-convex body. For each $y \in \mathbb{H}^m$, the supremum $\sup_{x \in \mathbb{H}^m} \rho(x) \langle x, y \rangle_H$
is negative and is attained at some point $x \in \mathbb{H}^m$. Furthermore, for every point $y \in \mathbb{H}^m$, $y$ is the unique inward unit normal vector to a supporting hyperplane of $\mathcal{C}$.
\end{proposition}

From Proposition \ref{prop:support-plane}, it follows that the dual function $\rho^*$ is well-defined on $\mathbb{H}^m$, and its boundary is greater than zero and less than infinity. Moreover, $\rho^*$ is invariant under the action of the group $\Gamma$, \ie,
\begin{equation}
\rho^*(\gamma y) = \rho^*(y), \quad \forall y \in \mathbb{H}^m, \, \forall \gamma \in \Gamma.
\end{equation}

\begin{proposition}[Lemma 2.30 \cite{bertrand2014prescription}]
\label{prop:legendre-dual}
Let $\mathcal{C}$ be a $\Gamma$-convex body. Then, $\mathcal{C}^*$ is also a $\Gamma$-convex body, and $(\mathcal{C}^*)^* = \mathcal{C}$.

Furthermore, the radial function of $\partial \mathcal{C}^*$ is $\rho^*$, and
\begin{equation}
\rho(x) = \sup_{y \in \mathbb{H}^m} \frac{-1}{\rho^*(y) \langle y, x \rangle_H}, \quad \forall x \in \mathbb{H}^m.
\end{equation}
\end{proposition}

From Proposition \ref{prop:legendre-dual}, it follows that $\rho$ and $\rho^*$ are hyperbolic Legendre duals of each other. This shows that the hyperbolic Legendre dual is the hyperbolic space counterpart of the Legendre dual in Euclidean convex analysis theory.

From a geometric perspective, the $\Gamma$-convex body $\mathcal{C}$ is closely related to its hyperbolic Legendre dual $\mathcal{C}^*$. Specifically, $\mathcal{C}^*$ is the convex hull of all the supporting hyperplanes of $\mathcal{C}$, and $\mathcal{C}$ is the convex hull of all the supporting hyperplanes of $\mathcal{C}^*$. For any point $\rho^*(y_0) y_0 \in \partial \mathcal{C}^*$, by Proposition \ref{prop:support-plane}, there exists a unique supporting hyperplane on $\mathcal{C}$ with $y_0$ as the inward unit normal vector, and its supporting point is $\rho(x_0) x_0 \in \partial \mathcal{C}$, and we have $y_0 \in \partial \rho(x_0)$, \ie,
\begin{equation}
\rho^*(y_0) = \frac{-1}{\rho(x_0) \langle x_0, y_0 \rangle_H} \text{ and }  \rho^*(y_0) \geq \frac{-1}{\rho(x) \langle x, y_0 \rangle_H}, \quad \forall x \in \mathbb{H}^m.
\end{equation}

Conversely, for any point $\rho(x_0) x_0 \in \partial \mathcal{C}$, there exists a unique supporting hyperplane of $\mathcal{C}^*$ with $x_0$ as the inward unit normal vector, and its supporting point is $\rho(y_0) y_0 \in \partial \mathcal{C}^*$. Moreover, we have $x_0 \in \partial \rho^*(y_0)$, \ie,
\begin{equation}
\rho(x_0) = \frac{-1}{\rho^*(y_0) \langle y_0, x_0 \rangle_H} \text{ and }  \rho(x_0) \geq \frac{-1}{\rho^*(y) \langle y, x_0 \rangle_H}, \quad \forall y \in \mathbb{H}^m.
\end{equation}

\begin{figure}[h]
\begin{center}
\fbox{\rule{0pt}{2in}
\includegraphics[width=0.8\linewidth]{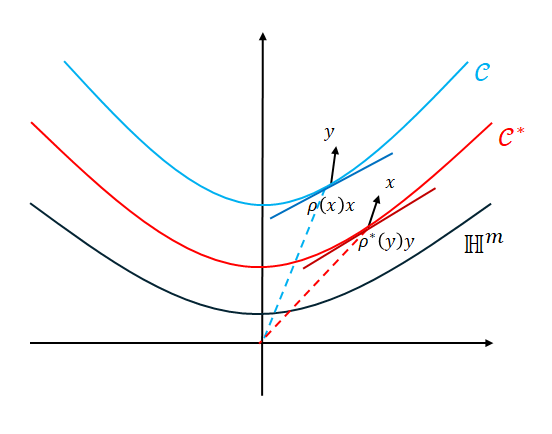}}
\end{center}
   \caption{Hyperbolic Legendre Dual}
\label{fig:legendre-dual}
\end{figure}

Let $\mathcal{P}$ be a $\Gamma$-convex polyhedron. From Definition \ref{def:fuchsian}, it follows that there exist points $x_1, x_2, \dots, x_k \in \mathbb{H}^m$ and $\rho = (\rho_1, \rho_2, \dots, \rho_k) \in \mathbb{R}^k_+$ such that $\mathcal{P}$ is the hyperbolic Legendre dual of the point set $\{ \rho_i \gamma x_i \mid 1 \leq i \leq k, \gamma \in \Gamma \}$. In fact,
\begin{equation}
\label{eq:polyhedron-dual}
\begin{split}
\mathcal{P} &= \{ z \in C_f \mid \langle z - \frac{1}{\rho_i} \gamma x_i, \gamma x_i \rangle_H \leq 0, \forall 1 \leq i \leq k, \forall \gamma \in \Gamma \} \\
&= \{ z \in C_f \mid \langle z, \rho_i \gamma x_i \rangle_H \leq -1, \forall 1 \leq i \leq k, \forall \gamma \in \Gamma \} \\
&= \{ \rho_i \gamma x_i \mid 1 \leq i \leq k, \gamma \in \Gamma \}^*.
\end{split}
\end{equation}

The following two propositions are important properties of $\Gamma$-convex polyhedrons.
\begin{proposition}[Lemma 3.3 \cite{bertrand2014prescription}]
\label{prop:polyhedron-dual}
The hyperbolic Legendre dual of a $\Gamma$-convex polyhedron is also a $\Gamma$-convex polyhedron.
\end{proposition}

\begin{proposition}[Lemma 4.2 \cite{fillastre2013fuchsian}]
\label{prop:polyhedron-face}
An $m$-dimensional $\Gamma$-convex polyhedron is locally finite and has countably many faces, each of which is an $(m - 1)$-dimensional convex polyhedron.
\end{proposition}

Inspired by Euclidean convex geometry \cite{schneider2013convex}, we define the concepts of convex hull and upper envelope in Minkowski spacetime $\mathbb{R}^m_1$. Since both $\mathbb{R}^m_1$ and $\mathbb{E}^{m+1}$ are defined in the real vector space $\mathbb{R}^{m+1}$, the convex hull in $\mathbb{R}^m_1$ is analogous to the Euclidean convex hull, and the upper envelope is obtained by replacing the Euclidean inner product with the Lorentzian inner product.

\begin{definition}
The convex hull (positive hull) of a set of points $A = \{ x_i \in \mathbb{R}^m_1 \mid i \in I \}$ in $\mathbb{R}^m_1$ is the set of all convex (positive) combinations of any finite number of elements of $A$, denoted as $\text{Conv}(A) = \text{Conv} \{ x_i \mid i \in I \}$.
\end{definition}

The convex hull of a set of points $A$ is the boundary of the positive hull of the set $A$, and they are isomorphic.

\begin{definition}
The upper envelope of the family of functions $\{ f_i : \mathbb{H}^m \to \mathbb{R} \mid i \in I \}$ in $\mathbb{R}^m_1$ is the graph of the following function:
\begin{equation}
f(y) = \sup_{i \in I} f_i(y), \quad \forall y \in \mathbb{H}^m,
\end{equation}
denoted as $\text{Env}(f) = \text{Env} \{ f_i \mid i \in I \}$. 
\end{definition}

Now let the $\Gamma$-convex polyhedron $\mathcal{P}$ be given by equation \ref{eq:polyhedron}, and $\mathcal{U}$ be the boundary of $\mathcal{P}$. If $\rho^*$ is the radial function of $\mathcal{U}$, then:
\begin{equation}
\rho^*(y) = \sup_{1 \leq i \leq k, \gamma \in \Gamma} \frac{-1}{\rho_i \langle \gamma x_i, y \rangle_H}.
\end{equation}

This implies that $\mathcal{U}$ is the upper envelope of the family of functions 
$\{ \pi_{i, \gamma}(y) = \frac{-1}{\rho_i \langle \gamma x_i, y \rangle_H} \mid 1 \leq i \leq k, \gamma \in \Gamma \}$ 
in $\mathbb{R}^m_1$. 
For each face of $\mathcal{U}$, there exists $1 \leq i \leq k$ and $\gamma \in \Gamma$ such that its unit normal vector can be represented as $\gamma x_i$, denoted as $F_{i, \gamma}$.

By Proposition \ref{prop:polyhedron-face}, $F_{i, \gamma}$ is an $(m - 1)$-dimensional convex polyhedron and is contained in the hyperplane 
$\pi_{i, \gamma} = \left\{ \frac{-1}{\rho_i \langle \gamma x_i, y \rangle_H}  \mid y \in \mathbb{H}^m \right\}$.

The radial projection $U_{i, \gamma}$ of the face $F_{i, \gamma}$, can be written as:
\begin{equation}
\label{eq:face-projection}
U_{i, \gamma} = \left\{ y \in \mathbb{H}^m \mid \frac{-1}{\rho_i \langle \gamma x_i, y \rangle_H} \geq \frac{-1}{\rho_j \langle \beta x_j, y \rangle_H}, \forall 1 \leq j \leq k, \forall \beta \in \Gamma  \right\}.
\end{equation}

To simplify the notation, we omit the subscript $\gamma$ when $\gamma = 1$. For example, $U_{i, 1} = U_i$. Note that $U_{i, \gamma} = p(F_{i, \gamma}) = \gamma p(F_i) = \gamma U_i$. Thus, by Proposition \ref{prop:projection}, the radial projection of $\mathcal{U}$ induces a cell decomposition $\mathcal{D}$ on $\mathbb{H}^m$,
\begin{equation}
\mathbb{H}^m = \bigcup_{i=1}^k \bigcup_{\gamma \in \Gamma} U_{i, \gamma} = \bigcup_{i=1}^k \bigcup_{\gamma \in \Gamma} \gamma U_i.
\end{equation}

Furthermore, $\mathcal{D}$ induces a hyperbolic weighted Delaunay triangulation $\mathcal{T}$, whose vertex set is $\{ \gamma x_i | 1 \leq i \leq k, \gamma \in \Gamma \}$. Two cells $U_{i, \gamma}, U_{j, \beta} \in \mathcal{D}$ intersect in an edge if and only if the points $\gamma x_i$ and $\beta x_j$ are connected by an edge in $\mathcal{T}$, \ie,
\begin{equation}
\label{eq:face-nonempty}
U_{i, \gamma} \cap U_{j, \beta} \neq \varnothing \iff \gamma x_i \sim \beta x_j \in \mathcal{T}.
\end{equation}

Let $\mathcal{C}$ be the positive hull of the point set $\{ \rho_i \gamma x_i \mid 1 \leq i \leq k, \gamma \in \Gamma \}$ in $\mathbb{R}^m_1$, which is a $\Gamma$-convex set. From equation \ref{eq:polyhedron-dual}, we have $\mathcal{P} = \mathcal{C}^*$. From Proposition \ref{prop:legendre-dual}, we obtain $\mathcal{C} = (\mathcal{C}^*)^* = \mathcal{P}^*$. Therefore, by Proposition \ref{prop:polyhedron-dual}, $\mathcal{C}$ is also a $\Gamma$-convex polyhedron. All the vertices of $\mathcal{C}$ are contained in the point set $\{ \rho_i \gamma x_i \mid 1 \leq i \leq k, \gamma \in \Gamma \}$.

\begin{figure}[h]
\begin{center}
\fbox{\rule{0pt}{2in}
\includegraphics[width=0.8\linewidth]{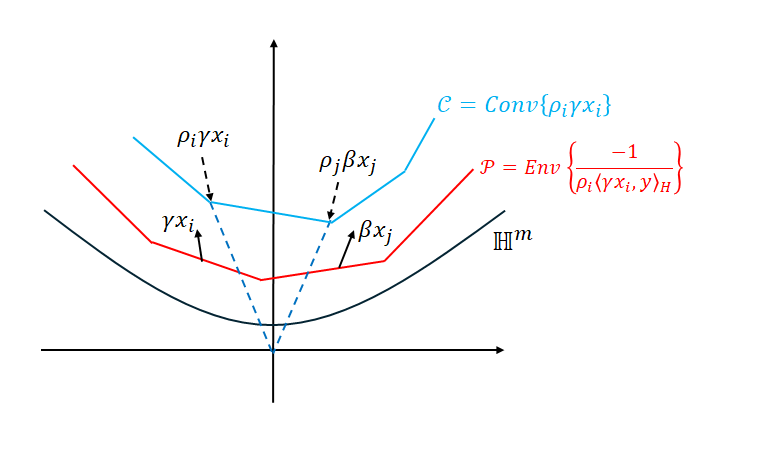}}
\end{center}
   \caption{Convex Hull $\mathcal{C}$ and upper envelope $\mathcal{P}$}
\label{fig:convex-hull}
\end{figure}

Next, we have the following proposition:
\begin{proposition}
\label{prop:vertices}
The set of all vertices of $\mathcal{C}$ is $\{ \rho_i \gamma x_i \mid 1 \leq i \leq k, \gamma \in \Gamma \}$ if and only if $ \text{int}(U_i) \neq \varnothing $,  $\forall 1 \leq i \leq k $.
\end{proposition}

\begin{proof}
Assume that $\text{int}(U_i) \neq \varnothing$ for all $1 \leq i \leq k$. According to equation \ref{eq:face-projection}, for any $ 1 \leq i \leq k $, there exists $y_i \in \text{int}(U_i)$ such that
\begin{equation}
\rho_i \langle x_i, y_i \rangle_H > \rho_j \langle \gamma x_j, y_i \rangle_H, \forall 1 \leq j \leq k, \forall \gamma \in \Gamma \text{ s.t. } \gamma x_j \neq x_i.
\end{equation}

Thus, $\langle \rho_j \gamma x_j - \rho_i x_i, y_i \rangle_H < 0, \forall 1 \leq j \leq k, \forall \gamma \in \Gamma $ such that $ \gamma x_j \neq x_i $. This implies that all other points $ \rho_j \gamma x_j $ lie on the future side of the hyperplane passing through the point $ \rho_i x_i $, described by $ \langle z - \rho_i x_i, y_i \rangle_H = 0 $. Therefore, $ \rho_i x_i $ is a vertex of $ \mathcal{C} $. Since $ \Gamma $ acts invariantly on $ \mathcal{C} $, the set of all vertices of $ \mathcal{C} $ is $ \{ \rho_i \gamma x_i | 1 \leq i \leq k, \gamma \in \Gamma \} $.

Assume that the set of all vertices of $ \mathcal{C} $ is $ \{ \rho_i \gamma x_i | 1 \leq i \leq k, \gamma \in \Gamma \} $. For $ 1 \leq i \leq k $, at the point $ \rho_i x_i $, there exists a supporting hyperplane $ P_i $, with an inward unit normal vector $ y_i \in \mathbb{H}^m $, such that all other points $ \rho_j \gamma x_j $ lie on the future side of $ P_i $. This implies that $ \langle \rho_j \gamma x_j - \rho_i x_i, y_i \rangle_H < 0 $, for all $ 1 \leq j \leq k $, and for all $ \gamma \in \Gamma $ such that $ \gamma x_j \neq x_i $. Therefore, $ \text{int}(U_i) \neq \varnothing, \forall 1 \leq i \leq k $.
\end{proof}

For simplicity, we still use $ \rho $ to denote the radial function of $ \partial \mathcal{C} $. According to Definition \ref{def:subnormal} and Proposition \ref{prop:support-plane}, the subnormal map $ \partial \rho $ induces a subnormal cell decomposition $ \mathcal{S} $ of $ \mathbb{H}^m $,
\begin{equation}
\mathbb{H}^m = \bigcup_{i=1}^k \bigcup_{\gamma \in \Gamma} \partial \rho(\gamma x_i) = \bigcup_{i=1}^k \bigcup_{\gamma \in \Gamma} \gamma \partial \rho(x_i).
\end{equation}

\begin{proposition}
\label{prop:isomorphic}
The cell decomposition $ \mathcal{D} $ and the subnormal cell decomposition $ \mathcal{S} $ are isomorphic.
\end{proposition}

\begin{proof}
From equation \ref{eq:face-projection}, we know that for any $ y \in \mathbb{H}^m $,
\begin{equation}
\label{eq:isomorphic}
\begin{split}
y \in U_{i,\gamma} &\iff \rho_i \langle \gamma x_i, y \rangle_H \geq \rho_j \langle \beta x_j, y \rangle_H, \quad \forall 1 \leq j \leq k, \beta \in \Gamma \\
&\iff \rho_i \langle \gamma x_i, y \rangle_H \geq \langle z, y \rangle_H, \quad \forall z \in \mathcal{C} \\
&\iff y \in \partial \rho (\gamma x_i).
\end{split}
\end{equation}
Thus, $ U_{i,\gamma} = \partial \rho(\gamma x_i), \forall 1 \leq i \leq k, \gamma \in \Gamma$.
\end{proof}

Proposition \ref{prop:isomorphic} shows that $\mathcal{P}$, $\mathcal{U}$, $\mathcal{D}$, and $\mathcal{S}$ are isomorphic to each other.

\begin{proposition}
\label{prop:isomorphic-TC}
Assume that $ \text{int}(U_i) \neq \varnothing, \forall 1 \leq i \leq k $. The hyperbolic weighted Delaunay triangulation $ \mathcal{T} $ and the convex hull $ \mathcal{C} $ are isomorphic.
\end{proposition}

\begin{proof}
From Proposition \ref{prop:vertices}, we know that the set of all vertices of $ \mathcal{C} $ is $ \{ \gamma x_i \mid 1 \leq i \leq k, \gamma \in \Gamma \} $, and the set of vertices of $ \mathcal{T} $ is also this set. Therefore, based on equations \ref{eq:face-nonempty} and \ref{eq:isomorphic}, we have:
\begin{equation}
\begin{split}
\gamma x_i \sim \beta x_j \in \mathcal{T} &\iff U_{i, \gamma} \cap U_{j, \beta} \neq \varnothing \\
&\iff \partial \rho(\gamma x_i) \cap \partial \rho(\beta x_j) \neq \varnothing \\
&\iff \rho_i \gamma x_i \sim \rho_j \beta x_j \in \mathcal{C}.
\end{split}
\end{equation}
\end{proof}

\subsection{Geometric Variational Principle for Hyperbolic Optimal Transport}
In this section, we analyze the optimal transport problem on a hyperbolic manifold using the perspective of convex differential geometry. 
By studying the hyperbolic Legendre duality theory, we investigate the equivalence between the hyperbolic optimal transport problem and the Minkowski problem in Minkowski spacetime, and derive the geometric variational principle for the hyperbolic optimal transport map. 
Based on Kantorovich duality, we prove that the Kantorovich functional is twice differentiable and concave, and that its gradient and Hessian matrix have integral expressions.

Let the hyperbolic transport cost function $ c: M \times M \to \mathbb{R} $ on the hyperbolic manifold $ M $ be defined as
\begin{equation}
c(x, y) = \ln \cosh \, d_M(x, y).
\end{equation}

Here, $ d_M $ is the geodesic distance function (Riemannian metric function) of $ M $. We call the optimal transport problem on the hyperbolic manifold the optimal transport problem under this hyperbolic transport cost function. Next, we consider constructing convex functions in the universal covering space of the hyperbolic manifold to perform geometric variational analysis of the hyperbolic optimal transport problem.

Let $ M = \mathbb{H}^m / \Gamma $ be an $m$-dimensional compact hyperbolic manifold, and let $ \mu $ be a probability measure on $ M $ that is absolutely continuous with respect to the Riemannian measure $ \sigma_M $. Let $ \nu = \sum_{i=1}^k \nu_i \delta_{p_i} $ be a discrete measure, where $ \{ p_1, p_2, \dots, p_k \} \subset M $ and $ \mu(M) = \sum_{i=1}^k \nu_i $.

According to the definition of absolute continuity, we know that there exists a measurable function $ f: M \to \mathbb{R} $ such that $ d\mu = f \, d\sigma_M $. Therefore, the composition $ f_H := f \circ \pi_\Gamma $ is a measurable function on $ \mathbb{H}^m $. Hence, we can define the measure $ \mu_H $ on $ \mathbb{H}^m $, with the density function given by $d\mu_H = f_H \, d\sigma_H$.

Now, we select a fundamental domain $D \subset \mathbb{H}^m$ for the group $\Gamma$ acting on $\mathbb{H}^m$, and points $x_1, x_2, \dots, x_k \in \overline{D}$ such that $\pi_\Gamma(x_i) = p_i$, for all $1 \leq i \leq k$.
For any $\varphi = (\varphi_1, \varphi_2, \dots, \varphi_k) \in \mathbb{R}^k$, we define $\mathcal{C}_\varphi$ as the convex hull in $\mathbb{R}^m_1$ of the set of points $\{ e^{-\varphi_i} \gamma x_i \mid 1 \leq i \leq k, \gamma \in \Gamma \}$. Let $\mathcal{P}_\varphi$ be the $\Gamma$-convex polyhedron defined by the set $\{ z \in C_f \mid \langle z - e^{\varphi_i} \gamma x_i, \gamma x_i \rangle_H \leq 0, \forall 1 \leq i \leq k, \forall \gamma \in \Gamma \}$.
According to the hyperbolic Legendre duality theory established in the previous section, $\mathcal{P}_\varphi$ and $\mathcal{C}_\varphi$ are mutually hyperbolic Legendre dual $\Gamma$-convex polyhedrons.

Let $\mathcal{U}_\varphi$ be the boundary of $\mathcal{P}_\varphi$, and its radial function $\rho^*$ is given by
\begin{equation}
\label{eq:radial-dual}
\rho^*(y) = \sup_{1 \leq i \leq k, \gamma \in \Gamma} \frac{-e^{\varphi_i}}{\langle \gamma x_i, y \rangle_H}, \forall y \in \mathbb{H}^m.
\end{equation}

This implies that $\mathcal{U}_\varphi$ is the upper envelope of the family of hyperplanes in $\mathbb{R}^m_1$: 
\begin{equation}
\pi_{i,\gamma}(y) = \frac{-e^{\varphi_i}}{\langle \gamma x_i, y \rangle_H}, 1 \leq i \leq k, \gamma \in \Gamma.
\end{equation}

For each face of $\mathcal{U}_\varphi$, denoted as $F_{i, \gamma}(\varphi)$, there exists $1 \leq i \leq k$ and $\gamma \in \Gamma$ such that its unit normal vector can be expressed as $\gamma x_i$. 
By Proposition \ref{prop:polyhedron-face}, we know that $F_{i, \gamma}(\varphi)$ is an $m$-dimensional compact convex polyhedron and is contained in the hyperplane $\pi_{i, \gamma} = \left\{ \frac{-e^{\varphi_i}}{\langle \gamma x_i, y \rangle_H} y \mid y \in \mathbb{H}^m \right\}$.
According to equation \ref{eq:face-projection}, we know that the radial projection of $F_{i, \gamma}(\varphi)$ can be expressed as
\begin{equation}
\label{eq:face-projection-phi}
U_{i, \gamma}(\varphi) = \{ y \in \mathbb{H}^m \mid e^{-\varphi_i} \langle \gamma x_i, y \rangle_H \geq e^{-\varphi_j} \langle \beta x_j, y \rangle_H, \forall 1 \leq j \leq k, \forall \beta \in \Gamma \}.
\end{equation}

Thus, the radial projection of $\mathcal{U}_\varphi$ induces a cell decomposition $\mathcal{D}_\varphi$ on $\mathbb{H}^m$,
\begin{equation}
\label{eq:cell-decomp}
\mathbb{H}^m = \bigcup_{i=1}^k \bigcup_{\gamma \in \Gamma} U_{i, \gamma}(\varphi) = \bigcup_{i=1}^k \bigcup_{\gamma \in \Gamma} \gamma U_i(\varphi).
\end{equation}

Since the discrete group $\Gamma$ is finitely generated, we can index all elements of $\Gamma$ using an index set $I$. Let $E_j^i \subset I$ be the set of indices $k \in I$ such that $U_{j, \gamma_k}(\varphi)$ and $U_i(\varphi)$ intersect in $\mathbb{H}^m$ along an edge $U_{i, jk}$. Since each edge $U_{i, jk}$ is the radial projection of an edge of the face $F_i$ and is contained in the intersection line of $\mathbb{H}^m$ and the hyperplane $\{ y \in \mathbb{R}^m_1 \mid e^{-\varphi_i} \langle x_i, y \rangle_H = e^{-\varphi_j} \langle \gamma_k x_j, y \rangle_H \}$, it follows that $U_{i, jk}$ is a geodesic segment in $\mathbb{H}^m$. Therefore, each cell $U_i(\varphi)$ is a hyperbolic convex polyhedron in $\mathbb{H}^m$. This implies that each set $E_j^i$ is either finite or empty.

\subsubsection{Geometric Variational Principle}
We first prove a property of the covering map $\pi_\Gamma$.
\begin{proposition}
\label{prop:cover-inj}
The covering map $\pi_\Gamma$ restricted to $U_i(\varphi) \setminus \cup_{k \in E_i^i} U_{i, ik}$ is injective.
\end{proposition}

\begin{proof}
Suppose there are two points $y_1, y_2 \in U_i(\varphi)$ such that $\pi_\Gamma(y_1) = \pi_\Gamma(y_2)$. Then, there exists a non-trivial element $\gamma \in \Gamma$ such that $y_2 = \gamma y_1$. According to equation \ref{eq:face-projection-phi}, since $y_1, \gamma y_1 \in U_i(\varphi)$, we can deduce that
\begin{equation}
\label{eq:cover-inj-proof1}
e^{-\varphi_i} \langle x_i, y_1 \rangle_H \geq e^{-\varphi_j} \langle \beta x_j, y_1 \rangle_H, \forall 1 \leq j \leq k, \beta \in \Gamma,
\end{equation}
\begin{equation}
\label{eq:cover-inj-proof2}
e^{-\varphi_i} \langle x_i, \gamma y_1 \rangle_H \geq e^{-\varphi_j} \langle \beta x_j, \gamma y_1 \rangle_H, \forall 1 \leq j \leq k, \beta \in \Gamma,
\end{equation}
Let $j = i$ and $\beta = \gamma^{-1}$ in equation \ref{eq:cover-inj-proof1}, and  let $j = i$ and $\beta = \gamma$ in equation \ref{eq:cover-inj-proof2}, we obtain
\begin{equation}
e^{-\varphi_i} \langle x_i, y_1 \rangle_H \geq e^{-\varphi_i} \langle \gamma^{-1} x_i, y_1 \rangle_H,
\end{equation}
\begin{equation}
e^{-\varphi_i} \langle \gamma^{-1} x_i, y_1 \rangle_H = e^{-\varphi_i} \langle x_i, \gamma y_1 \rangle_H \geq e^{-\varphi_i} \langle \gamma x_i, \gamma y_1 \rangle_H = e^{-\varphi_i} \langle x_i, y_1 \rangle_H.
\end{equation}

This implies that $e^{-\varphi_i} \langle x_i, y_1 \rangle_H = e^{-\varphi_i} \langle \gamma^{-1} x_i, y_1 \rangle_H$ and $e^{-\varphi_i} \langle \gamma x_i, \gamma y_1 \rangle_H = e^{-\varphi_i} \langle x_i, \gamma y_1 \rangle_H$.

Therefore, $y_1 \in U_i(\varphi) \cap U_{i, \gamma^{-1}}(\varphi)$ and $y_2 \in U_i(\varphi) \cap U_{i, \gamma}(\varphi)$. Thus, there exist $k_1, k_2 \in E_i^i$ such that $y_1 \in U_{i, ik_1}$ and $y_2 \in U_{i, ik_2}$. Therefore, the covering map $\pi_\Gamma$ restricted to $U_i(\varphi) \setminus \cup_{k \in E_i^i} U_{i, ik}$ is injective.
\end{proof}

For any $1 \leq i \leq k$, let the covering projection of the cell $U_i(\varphi)$ be denoted as $W_i(\varphi) := \pi_\Gamma(U_i(\varphi))$, with its $\mu$-measure defined as $\omega_i(\varphi) := \mu(W_i(\varphi))$. 
By Proposition \ref{prop:cover-inj}, the covering map $\pi_\Gamma$ is an isometry on $\text{int}(U_i(\varphi))$ and maps it to $\text{int}(W_i(\varphi))$. 
According to Corollary 5.14 in \cite{lee2018introduction}, we know that $\pi_\Gamma$ maps each geodesic edge of $U_i(\varphi)$ to a geodesic edge of $W_i(\varphi)$. 
Thus, $\pi_\Gamma$ maps the boundary of $U_i(\varphi)$ to the boundary of $W_i(\varphi)$. 
The two cells $W_i(\varphi)$ and $W_j(\varphi)$ intersect in a geodesic edge $W_{ij}$ if and only if $E_j^i \neq \varnothing $.
Note that every hyperplane of codimension 1 in $\mathbb{H}^m$ has measure zero under the Riemannian measure $\sigma_H$, hence
\begin{equation}
\label{eq:measure-zero}
\sigma_M (\partial W_i(\varphi)) = \sigma_H (\partial U_i(\varphi)) = 0, \forall 1 \leq i \leq k,
\end{equation}
\begin{equation}
\omega_i(\varphi) = \mu(W_i(\varphi)) = \mu_H (U_i(\varphi)), \forall 1 \leq i \leq k.
\end{equation}

Furthermore, $W_i(\varphi)$ is a convex region with finitely many geodesic edges, and
\begin{equation}
\label{eq:cell}
\begin{split}
W_i(\varphi) &= \pi_\Gamma \left\{ y \in \mathbb{H}^m \mid e^{-\varphi_i} \langle x_i, y \rangle_H \geq e^{-\varphi_j} \langle \beta x_j, y \rangle_H, \forall 1 \leq j \leq k, \beta \in \Gamma \right\} \\
&= \{ p \in M \mid \ln \cosh d_M (p, p_i) - \varphi_i \leq \ln \cosh d_M (p, p_j) - \varphi_j, \forall 1 \leq j \leq k \}.
\end{split}
\end{equation}

Next, we will prove that the collection of all $W_i(\varphi)$ induces a cell decomposition on $M$.
\begin{proposition}
\label{prop:m-decomp}
Given a vector $\varphi = (\varphi_1, \varphi_2, \dots, \varphi_k) \in \mathbb{R}^k$, the collection of all cells $W_i(\varphi)$, $1 \leq i \leq k$, induces a cell decomposition $\mathcal{W}_\varphi$ on $M$,
\begin{equation}
\label{eq:m-decomp}
M = \bigcup_{i=1}^k W_i(\varphi).
\end{equation}
Furthermore, $\mathcal{W}_\varphi$ is independent of the choice of the fundamental domain $D$ of the group $\Gamma$ acting on $\mathbb{H}^m$ and the representatives $x_i \in \overline{D}$, $1 \leq i \leq k$.
\end{proposition}

\begin{proof}
Given any $p \in M$, there exists a $y \in \overline{D}$ such that $\pi_\Gamma(y) = p$. From the cell decomposition $\mathcal{D}_\varphi$ given by equation \ref{eq:cell-decomp}, we know that there exists $1 \leq i \leq k$ and $\gamma \in \Gamma$ such that $y \in U_{i,\gamma}(\varphi)$, hence $p \in \pi_\Gamma(U_{i,\gamma}(\varphi)) = \pi_\Gamma(U_i(\varphi)) = W_i(\varphi)$.

If for $1 \leq i < j \leq k$ we have $p \in W_i(\varphi) \cap W_j(\varphi)$, then there exist $y_1 \in U_i(\varphi)$ and $y_2 \in U_j(\varphi)$ such that $\pi_\Gamma(y_1) = \pi_\Gamma(y_2) = p$. 
Therefore, there exists an element $\gamma \in \Gamma$ such that $y_2 = \gamma y_1$, which implies that $y_1 \in U_{j, \gamma^{-1}}$. 
Hence, $y_1 \in \partial U_i \cap \partial U_{j, \gamma^{-1}}$. Note that $\pi_\Gamma(U_{j, \gamma^{-1}}) = \pi_\Gamma(U_j) = W_j(\varphi)$, so $p \in \partial W_i(\varphi) \cap \partial W_j(\varphi)$. 
Therefore, the collection of all cells $W_i(\varphi)$ induces a cell decomposition on $M$.

Assume we select another fundamental domain $D'$ of the group $\Gamma$ acting on $\mathbb{H}^m$, and representatives $x'_i \in \overline{D'}$, $1 \leq i \leq k$. Then, for $1 \leq i \leq k$, there exists $\gamma_i \in \Gamma$ such that $x'_i = \gamma_i x_i$, and thus
\begin{equation}
\begin{split}
U_{x'_i}(\varphi) &= \{ y \in \mathbb{H}^m \mid e^{-\varphi_i} \langle x'_i, y \rangle_H \geq e^{-\varphi_j} \langle \beta x'_j, y \rangle_H, \forall 1 \leq j \leq k, \beta \in \Gamma \} \\
&= \{ y \in \mathbb{H}^m \mid e^{-\varphi_i} \langle \gamma_i x_i, y \rangle_H \geq e^{-\varphi_j} \langle \beta \gamma_j x_j, y \rangle_H, \forall 1 \leq j \leq k, \beta \in \Gamma \} \\
&= \{ y \in \mathbb{H}^m \mid e^{-\varphi_i} \langle \gamma_i x_i, y \rangle_H \geq e^{-\varphi_j} \langle \zeta x_j, y \rangle_H, \forall 1 \leq j \leq k, \zeta \in \Gamma \} \\
&= U_{i, \gamma_i}(\varphi).
\end{split}
\end{equation}

Thus, we have $\pi_\Gamma(U_{x'_i}(\varphi)) = \pi_\Gamma(U_{i, \gamma_i}(\varphi)) = W_i(\varphi)$. This shows that the cells $W_i(\varphi)$ is independent of the choice of fundamental domain $D$ and representatives $x_i \in \overline{D}, 1 \leq i \leq k$.
\end{proof}

The following proposition gives an important property regarding the hyperbolic transport cost of $\mathcal{W}_\varphi$.

\begin{proposition}
\label{prop:ot-cost}
Let $M = \mathbb{H}^m/\Gamma$ be an $m$-dimensional compact hyperbolic manifold, and let $\mu$ be a measure on $M$ that is absolutely continuous with respect to the Riemannian measure $\sigma_M$. Given $\varphi = (\varphi_1, \varphi_2, \dots, \varphi_k) \in \mathbb{R}^k$, let $\mathcal{W}_\varphi$ be the cell decomposition given by equation \ref{eq:m-decomp}. For any cell decomposition $\mathcal{D}$ of $M$, where $M = \bigcup_{i=1}^{k} X_i$ and $\mu(X_i) = \mu(W_i(\varphi)), \forall 1 \leq i \leq k$, the hyperbolic transport cost of $\mathcal{W}_\varphi$ is no greater than the hyperbolic transport cost of $\mathcal{D}$.
\end{proposition}

\begin{proof}
According to equation \ref{eq:cell}, we know that for any point $p \in W_i(\varphi) \cap X_j$, we have $\ln \cosh d_M(p, p_i) - \varphi_i \leq \ln \cosh d_M(p, p_j) - \varphi_j$, thus

\begin{equation}
\begin{split}
&\sum_{i=1}^k \int_{W_i(\varphi)} \left( \ln \cosh d_M(p, p_i) - \varphi_i \right) d\mu \\
= &\sum_{i=1}^k \sum_{j=1}^k \int_{W_i(\varphi) \cap X_j} \left( \ln \cosh d_M(p, p_i) - \varphi_i \right) d\mu \\
\leq &\sum_{i=1}^k \sum_{j=1}^k \int_{W_i(\varphi) \cap X_j} \left( \ln \cosh d_M(p, p_j) - \varphi_j \right) d\mu \\
= &\sum_{j=1}^k \int_{X_j} \left( \ln \cosh d_M(p, p_j) - \varphi_j \right) d\mu.
\end{split}
\end{equation}

Note that $\sum_{i=1}^k \int_{W_i(\varphi)} \varphi_i d\mu = \sum_{i=1}^k \mu(W_i(\varphi)) \varphi_i = \sum_{j=1}^k \mu(X_j) \varphi_j = \sum_{j=1}^k \int_{X_j} \varphi_j d\mu$.

Therefore, the hyperbolic transport cost of $\mathcal{W}_\varphi$ is no greater than the hyperbolic transport cost of $\mathcal{D}$.
\end{proof}

Based on the above, we now prove the following theorem.

\begin{theorem}
\label{thm:1}
Let $M = \mathbb{H}^m / \Gamma$ be an $m$-dimensional compact hyperbolic manifold, and $\mu$ be a measure on $M$ absolutely continuous with respect to the Riemannian measure $\sigma_M$. 
Let $\nu = \sum_{i=1}^k \nu_i \delta_{p_i}$ be a discrete measure satisfying $\{p_1, p_2, \dots, p_k\} \subset M$ and $\mu(M) = \sum_{i=1}^k \nu_i$. 
Given $\varphi = (\varphi_1, \varphi_2, \dots, \varphi_k) \in \mathbb{R}^k$, let $\mathcal{W}_\varphi$ be the cell decomposition given by equation \ref{eq:m-decomp}. 
Then, the transport map $T: W_i(\varphi) \mapsto p_i$ is a hyperbolic optimal transport map from $(M, \mu)$ to $(M, \nu)$ if and only if
\begin{equation}
\label{eq:face-measure}
\omega_i(\varphi) = \nu_i, \quad \forall 1 \leq i \leq k.
\end{equation}
\end{theorem}

\begin{proof}
The theorem follows directly from Proposition \ref{prop:m-decomp} and \ref{prop:ot-cost}.
\end{proof}

Theorem \ref{thm:1} shows that the composition of radial projection and covering map of a $\Gamma$-convex polyhedron can induce a hyperbolic optimal transport map. 
Therefore, the existence of the hyperbolic optimal transport map is transformed into the construction of a $\Gamma$-convex polyhedron with a given measure on each face. 
This provides a finite-dimensional geometric variational principle for optimal transport problems on hyperbolic manifolds.

\subsubsection{Kantorovich Dual}
To prove the existence of the hyperbolic optimal transport map, we need to solve the finite-dimensional nonlinear system \ref{eq:face-measure}. 
In this section, we prove that the nonlinear system \ref{eq:face-measure} is precisely the existence condition for the extremal solution of the Kantorovich functional. 
Since the radial function $ \rho^* $ is invariant under the action of the group $ \Gamma $, we can decompose a function $ \rho^*_M $ on the manifold $ M $ such that $ \rho^*_M \circ \pi_{\Gamma} = \rho^* $. 
For $ \varphi = (\varphi_1, \varphi_2, \dots, \varphi_k) \in \mathbb{R}^k $, we identify $ \varphi $ as a finitely supported function defined on $ M $, that is, $ \varphi(p_i) = \varphi_i, \forall 1 \leq i \leq k $, and for other points $ p \in M $, $\varphi(p) = 0$.

By equation \ref{eq:radial-dual}, we have
\begin{align}
&\rho^*(y) = \sup_{1 \leq i \leq k, \gamma \in \Gamma} \frac{-e^{\varphi_i}}{\langle \gamma x_i, y \rangle_H}, \nonumber \\
\Leftrightarrow &- \ln \rho^*(y) = \inf_{1 \leq i \leq k, \gamma \in \Gamma} \ln \left( - \langle \gamma x_i, y \rangle_H \right) - \varphi_i, \nonumber \\
\Leftrightarrow &- \ln \rho_M^*(p) = \inf_{1 \leq i \leq k} \ln \cosh d_M (p_i, p) - \varphi_i, \label{eq:kd1} \\
\Leftrightarrow &\, \varphi_i = \inf_{p \in M} \ln \cosh d_M (p_i, p) - (- \ln \rho_M^*(p)). \label{eq:kd2}
\end{align}

Thus, the $\overline{c}$-transform of $ \varphi $ under the hyperbolic transport cost function $ c(x, y) = \ln \cosh d_M(x, y) $ is $ \varphi^{\overline{c}}(p) = - \ln \rho^*_M(p) $. Therefore, from equation \ref{eq:cell} we have
\begin{equation}
\begin{split}
\varphi_i + \varphi^{\overline{c}}(p) = c(p_i, p) &\Leftrightarrow - \ln \rho^*_M(p) = \ln \cosh d_M(p_i, p) - \varphi_i, \\
&\Leftrightarrow p \in W_i(\varphi).
\end{split}
\end{equation}

Based on Theorem \ref{thm:duality} and equation \ref{eq:dp-exist}, the Kantorovich functional for the semi-discrete hyperbolic optimal transport problem from $(M, \mu)$ to $(M, \nu)$ is given by:
\begin{equation}
\label{eq:kantorovich-functional}
\begin{split}
I(\varphi) &= \int_M \varphi^{\overline{c}} \, d\mu + \int_M \varphi \, d\nu \\
&= \sum_{i=1}^k \int_{W_i(\varphi)} c(p_i,p) - \varphi_i \, d\mu + \sum_{i=1}^k \nu_i \varphi_i.
\end{split}
\end{equation}

By equation \ref{eq:measure-zero}, we can express $I(\varphi)$ as
\begin{equation}
I(\varphi) = \sum_{i=1}^k \int_{U_i(\varphi)} c_H(x_i,x) - \varphi_i \, d\mu_H + \sum_{i=1}^k \nu_i \varphi_i.
\end{equation}

\begin{figure}[h]
\begin{center}
\fbox{\rule{0pt}{2in}
\includegraphics[width=0.8\linewidth]{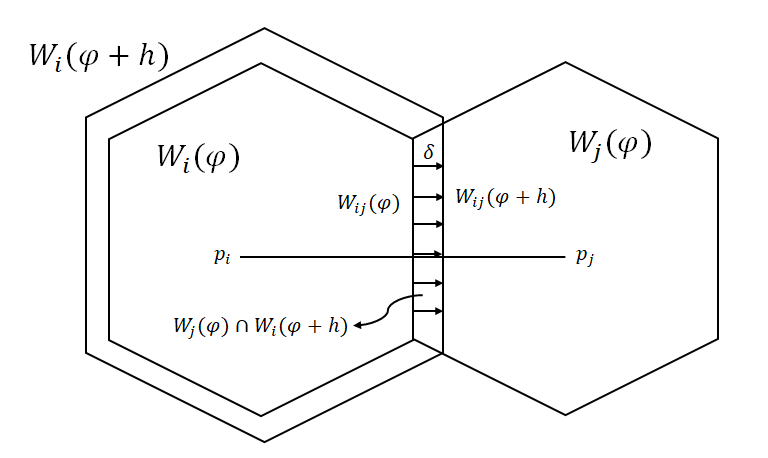}}
\end{center}
   \caption{Computing the cell measure}
\label{fig:cell-delta}
\end{figure}

To compute the gradient of $I(\varphi)$, we need to estimate the $\mu$-measure of the intersection edge between two adjacent cells. Suppose the intersection edge between two adjacent cells $W_i(\varphi)$ and $W_j(\varphi)$ is $W_{ij}(\varphi)$. Let $\lambda : M \to \mathbb{R}$ be an auxiliary function defined as:
\begin{equation}
\lambda(x) = c(p_i, x) - c(p_j, x), \forall x \in M.
\end{equation}

Thus, $W_{ij}(\varphi)$ is a level set of $\lambda$ where $\lambda(x) = \varphi_i - \varphi_j$, and is orthogonal to the gradient of $\lambda$,
\begin{equation}
\nabla \lambda(x) = \nabla_x c(p_i, x) - \nabla_x c(p_j, x),
\end{equation}

where $\nabla$ is the Levi-Civita connection of the manifold $M$.

The gradient flow of $\lambda$ is a smooth curve $ r : \mathbb{R} \to M $, defined by
\begin{equation}
\frac{d}{dt}r(x,t) = \frac{\nabla \lambda (x)}{|\nabla \lambda (x)|}, \quad r(x,0)= x.
\end{equation}

Note that $ \frac{d}{dt} \lambda(r(x,t)) = \langle \nabla \lambda, \dot{r} \rangle (r(x,t)) = |\nabla \lambda|(r(x,t)) $. 
This shows that the streamline $ r(x, \cdot) $ is a geodesic, with its initial position at $ x $ and the flow velocity being the unit velocity along the gradient direction $ \nabla \lambda(x) $. Given $ x \in M $ and $ t \in \mathbb{R} $, the geodesic distance from $ r(x, 0) $ to $ r(x, t) $ is $ |t| $.

For any $ \delta > 0 $, let $ \textbf{e}_i $ be the unit vector along the $ i $-th coordinate, and $ h = \delta \textbf{e}_i $. 
When $ \varphi $ changes to $ \varphi + h $, $ W_i(\varphi) $ changes to $ W_i(\varphi + h) $, and $ W_j(\varphi) \cap W_i(\varphi + h) \neq \varnothing $, as shown in Figure \ref{fig:cell-delta}.
For $ x \in W_{ij}(\varphi) $, suppose $ T(x) $ is the time period for the streamline $ r(x, \cdot) $ starting from $ W_{ij}(\varphi) $ to $ W_{ij}(\varphi + h) $. Then, there exists $ \xi = r(x, t_0) $, with $ t_0 \in [0, T(x)] $, such that
\begin{equation}
\begin{split}
\delta = \lambda(r(x,T(x))) - \lambda(r(x,0)) &= \int_0^{T(x)} \dot{\lambda}(r(x,T(x)) \, dt \\
&= \int_0^{T(x)} |\nabla \lambda|r(x,T(x)) \, dt \\
&= |\nabla \lambda(\xi)|T(x).
\end{split}
\end{equation}

Since $ \lambda $ is continuous and $ W_j(\varphi) \cap W_i(\varphi + h) $ is compact, $ \nabla \lambda $ is Lipschitz continuous on $ W_j(\varphi) \cap W_i(\varphi + h) $. Therefore, there exists $ C > 0 $ such that
\begin{equation}
\label{eq:grad-est}
|\nabla \lambda(x)| - C\delta \leq |\nabla \lambda(\xi)| \leq |\nabla \lambda(x)| + C\delta.
\end{equation}

Then, we have
\begin{equation}
T(x) = \frac{\delta}{|\nabla \lambda(\xi)|} = \frac{\delta}{|\nabla \lambda(x)|} + o(\delta).
\end{equation}

Therefore,
\begin{equation}
\begin{split}
&\mu(W_j(\varphi) \cap W_i(\varphi + h)) \\
= &\int_{W_j(\varphi) \cap W_i(\varphi + h)} f(y) \, d\sigma_M(y) \\
= &\int_{W_{ij}(\varphi)} T(x)f(x) \, d\sigma^{m-1}(x) \\
= &\int_{W_{ij}(\varphi)} \left( \frac{\delta}{|\nabla \lambda(x)|} + o(\delta) \right) f(x) \, d\sigma^{m-1}(x) \\
= &\, \delta \int_{W_{ij}(\varphi)} \frac{f(x)}{|\nabla_x c(p_i, x) - \nabla_x c(p_j, x)|} \, d\sigma^{m-1}(x) + o(\delta),
\end{split}
\end{equation}
where $ \sigma^{m-1} $ is the $(m-1)$-dimensional Hausdorff measure on $ M $. 

Based on the above, we obtain the partial derivative of the cell measure $ \omega_i(\varphi) $ as
\begin{equation}
\label{eq:omega-grad-ij}
\frac{\partial \omega_i(\varphi)}{\partial \varphi_j} = \frac{\partial \omega_j(\varphi)}{\partial \varphi_i} = - \int_{W_{ij}(\varphi)} \frac{f(x)}{|\nabla_x c(p_i, x) - \nabla_x c(p_j, x)|} \, d\sigma^{m-1}(x).
\end{equation}

If $ W_i(\varphi) $ and $ W_j(\varphi) $ are not adjacent, then $ W_{ij}(\varphi) = 0 $ and $ \frac{\partial \omega_j(\varphi)}{\partial \varphi_i} = 0 $, and the above formula still holds. Since $ \sum_{i=1}^k \omega_i(\varphi) = \mu(M) $, we have
\begin{equation}
\label{eq:omega-grad-i}
\frac{\partial \omega_i(\varphi)}{\partial \varphi_i} = - \sum_{j \neq i} \frac{\partial \omega_j(\varphi)}{\partial \varphi_i}.
\end{equation}

In conclusion, we obtain the following theorem:
\begin{theorem}
\label{thm:2}
Let $ M = \mathbb{H}^m / \Gamma $ be an $ m $-dimensional compact hyperbolic manifold, and the probability measure $ \mu $ on $ M $ has a continuous density function with respect to the Riemannian measure $ \sigma_M $, $ d\mu = f d\sigma_M $. 
Let $ \nu = \sum_{i=1}^k \nu_i \delta_{p_i} $ be a discrete measure, satisfying $ \min_{1 \leq i \leq k} \nu_i \geq 0 $, $ \{p_1, p_2, \dots, p_k\} \subset M $, and $ \mu(M) = \sum_{i=1}^k \nu_i $. 
Then, the Kantorovich functional $ I(\varphi) $ for the optimal transport problem on the hyperbolic space is a concave function on $ \mathbb{R}^k $. 
Furthermore, $ I(\varphi) $ is twice differentiable, and its gradient is
\begin{equation}
\label{eq:kantorovich-grad}
\text{grad} \, I(\varphi) = (\nu_1 - \omega_1(\varphi), \nu_2 - \omega_2(\varphi), \dots, \nu_k - \omega_k(\varphi)),
\end{equation}

with the second partial derivatives as
\begin{align}
\frac{\partial^2 I(\varphi)}{\partial \varphi_i^2} &= - \sum_{j \neq i} \frac{\partial \omega_i(\varphi)}{\partial \varphi_j}, \label{eq:kantorovich-d2i} \\
\begin{split}
\frac{\partial^2 I(\varphi)}{\partial \varphi_i \partial \varphi_j} &= - \frac{\partial \omega_i(\varphi)}{\partial \varphi_i} = - \frac{\partial \omega_j(\varphi)}{\partial \varphi_i} \\
&= \int_{W_{ij}(\varphi)} \frac{f(x)}{|\nabla_x c(p_i, x) - \nabla_x c(p_j, x)|} \, d\sigma^{m-1}(x), \label{eq:kantorovich-d2ij}
\end{split}
\end{align}
where $ \nabla $ is the Levi-Civita connection on $ M $, and $ \sigma_{m-1} $ is the $(m - 1)$-dimensional Hausdorff measure on $ M $.
\end{theorem}

\begin{proof}
Let $ \textbf{e}_i $ be the unit vector in the positive direction of the $ i $-th coordinate axis, $ \delta > 0 $, and $ h = \delta \textbf{e}_i $. Then, according to the definition of the Kantorovich functional in equation \ref{eq:kantorovich-functional}, we have
\begin{equation}
\begin{split}
I(\varphi + h) - I(\varphi) = &\, \delta \nu_i + \int_{W_i(\varphi) \cap W_i(\varphi + h)} (c(p_i,y) - \varphi_i - \delta) - (c(p_i,y) - \varphi_i) \, d\mu(y) \\
&+ \sum_{j \neq i} \int_{W_j(\varphi) \cap W_i(\varphi + h)} (c(p_i,y) - \varphi_i - \delta) - (c(p_j,y) - \varphi_j) \, d\mu(y).
\end{split}
\end{equation}

Let $ s(t) = \lambda(r(x, t)) - \lambda(r(x, 0)) $, then we have $ \frac{ds}{dt}(t) = |\nabla \lambda|(r(x, t)) $. Note that
\begin{equation}
\begin{split}
&\int_0^{T(x)} (c(p_i,r(x, t)) - \varphi_i - \delta) - (c(p_j,r(x, t)) - \varphi_j) \, dt \\
= &\int_0^{T(x)} \lambda(r(x,t)) - \lambda(r(x,0)) - \delta \, dt \\
= &\int_0^{\delta} s(t) - \delta \, dt \\
= &\int_0^{\delta} \frac{s - \delta}{\dot{s}} \, ds.
\end{split}
\end{equation}

From the gradient estimate \ref{eq:grad-est}, we know that
\begin{equation}
\int_0^{\delta} \frac{s - \delta}{\dot{s}} \, ds = \frac{1}{|\nabla \lambda(x)|} \int_0^{\delta} s - \delta \, ds + o(\delta) = -\frac{\delta^2}{2|\nabla \lambda(x)|} + o(\delta).
\end{equation}

Thus, for $j \neq i$ we have the estimate below
\begin{equation}
\label{eq:int-ji}
\begin{split}
&\int_{W_j(\varphi) \cap W_i(\varphi + h)} (c(p_i,y) - \varphi_i - \delta) - (c(p_j,y) - \varphi_j) \, d\mu(y) \\
= &\int_{W_{ij}(\varphi)} \int_0^{T(x)} \left[ (c(p_i,r(x, t)) - \varphi_i - \delta) - (c(p_j,r(x, t)) - \varphi_j) \right] f(x) \, dt \, d\sigma^{m-1}(x) \\
= &\int_{W_{ij}(\varphi)} -\frac{\delta^2}{2|\nabla \lambda(x)|} f(x) \, d\sigma^{m-1}(x) + o(\delta) \\
= &\, o(\delta).
\end{split}
\end{equation}

On the other hand, according to equations \ref{eq:omega-grad-ij} and \ref{eq:omega-grad-i}, we have
\begin{equation}
\begin{split}
\omega_i(\varphi) &= \mu \left( W_i(\varphi) \cap W_i(\varphi + h) \right) + \sum_{j \neq i} \mu \left( W_i(\varphi) \cap W_j(\varphi + h) \right) \\
&= \mu \left( W_i(\varphi) \cap W_i(\varphi + h) \right) + o(\delta).
\end{split}
\end{equation}

Therefore,
\begin{equation}
\label{eq:int-ii}
\begin{split}
&\int_{W_i(\varphi) \cap W_i(\varphi + h)} (c(p_i,y) - \varphi_i - \delta) - (c(p_i,y) - \varphi_i) \, d\mu(y) \\
= &-\delta \mu(W_i(\varphi) \cap W_i(\varphi + h)) \\
= &-\delta \omega_i(\varphi) + o(\delta). 
\end{split}
\end{equation}

By combining equations \ref{eq:int-ji} and \ref{eq:int-ii}, we obtain
\begin{equation}
I(\varphi + h) - I(\varphi) = \delta \nu_i - \delta \omega_i(\varphi) + o(\delta).
\end{equation}

This gives the proof of the gradient of the Kantorovich functional \ref{eq:kantorovich-grad}. The second partial derivatives of the Kantorovich functional \ref{eq:kantorovich-d2i} and \ref{eq:kantorovich-d2ij} are given by equations \ref{eq:omega-grad-ij} and \ref{eq:omega-grad-i}.

Note that for $ j \neq i $, $ \frac{\partial^2 I(\varphi)}{\partial \varphi_i \partial \varphi_j} \geq 0 $, and for $ 1 \leq i \leq k $, $ \sum_{j=1}^k \frac{\partial^2 I(\varphi)}{\partial \varphi_i \partial \varphi_j} = 0 $. 
We can deduce that the Hessian matrix $ \text{Hess}(I) $ of the Kantorovich functional is diagonally dominant and has a one-dimensional null space spanned by the vector $ (1, 1, \dots, 1) $. Therefore, $ \text{Hess}(I) $ is negative semi-definite. 
From this, we conclude that the Kantorovich functional is a concave function, and the proof is complete.
\end{proof}

From the proof of Proposition \ref{prop:cover-inj}, we know that the covering map $ \pi_\Gamma $ maps each geodesic edge of the cell $ U_{i,jk}(\varphi) $ isometrically to the boundary of $ W_{i,j}(\varphi) $. 
Thus, we have an equivalent formula to compute the Hessian matrix $ \text{Hess}(I(\varphi)) $ in the hyperbolic space $ \mathbb{H}^m $,

\begin{align}
\frac{\partial^2 I(\varphi)}{\partial \varphi_i^2} &= - \sum_{j \neq i} \frac{\partial^2 I(\varphi)}{\partial \varphi_i \partial \varphi_j}, \\
\frac{\partial^2 I(\varphi)}{\partial \varphi_i \partial \varphi_j} &= - \sum_{k \in E_i^j} \int_{U_{i,jk}(\varphi)} \frac{f_H(x)}{|\nabla_x^H c_H(x, x_i) - \nabla_x^H c_H(x, \gamma_k x_j)|} \, d\sigma^{m-1}(x),
\end{align}

where $ \sigma_H^{m-1} $ is the $(m - 1)$-dimensional Hausdorff measure on $ \mathbb{H}^m $, $ \nabla^H $ is the Levi-Civita connection on $ \mathbb{H}^m $, and $ f_H $ is the lift of $ f $ on $ \mathbb{H}^m $, satisfying $ f \circ \pi_\Gamma = f_H $. 
The index sets $ E^j_i $ and $ E^i_j $ are in one-to-one correspondence, \ie, $ U_i(\varphi) \cap U_{j,\gamma_k}(\varphi) \neq \varnothing \Leftrightarrow U_{i,\gamma^{-1}_k}(\varphi) \cap U_j(\varphi) \neq \varnothing $.

Using Proposition 7.7 in \cite{boumal2023introduction}, we can derive an explicit expression to compute the gradient of the cost function $ c_H(x, x_i) $ on $ \mathbb{H}^m $,
\begin{equation}
\nabla_x^H c_H(x, x_i) = \nabla_x^H \left( -\langle x, x_i \rangle_H \right) = (I + xx^TJ)J\frac{Jx_i}{\langle x, x_i \rangle_H} = \frac{x_i}{\langle x, x_i \rangle_H} + x,
\end{equation}
where $I$ is the identity matrix and $J = \text{diag}(1,\cdots, 1, -1)$.

Note that
\begin{equation}
\label{eq:edge-property}
e^{-\varphi_i} \langle x_i, x \rangle_H = e^{-\varphi_j} \langle \gamma_k x_j, x \rangle_H, \forall x \in U_{i,jk}(\varphi).
\end{equation}

Therefore, for any $x \in U_{i,jk}(\varphi)$, we have
\begin{equation}
\label{eq:edge-property2}
\frac{f_H(x)}{|\nabla_x^H c_H(x, x_i) - \nabla_x^H c_H(x, \gamma_k x_j)|} = \frac{-f_H(x) e^{-\varphi_i} \langle x_i, x \rangle_H}{|e^{-\varphi_i}x_i - e^{-\varphi_j}\gamma_k x_j|}.
\end{equation}

\begin{figure}[h]
\begin{center}
\fbox{\rule{0pt}{2in}
\includegraphics[width=0.8\linewidth]{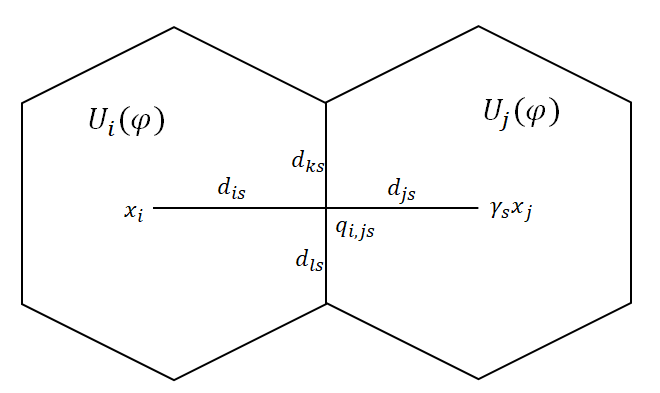}}
\end{center}
   \caption{Cells $U_i(\varphi)$ and $U_{j,\gamma_s}(\varphi)$}
\label{fig:cells}
\end{figure}

In the case of the two-dimensional hyperbolic plane $ \mathbb{H}^2 $, we have a special formula for the Hessian matrix $ \text{Hess}(I(\varphi)) $ given in the following corollary.
\begin{corollary}
Let $ M = \mathbb{H}^2/\Gamma $ be a two-dimensional compact hyperbolic manifold, with $ \mu = \sigma_M $. The cells $ U_i(\varphi) $ and $ U_{j,\gamma_s}(\varphi) $ intersect at the common edge $ U_{i,j_s}(\varphi) $, as shown in Figure \ref{fig:cells}. 
Let $ q_{i,js} $ be the intersection point of the cell $ U_{i,js}(\varphi) $ and the hyperbolic geodesic connecting $ x_i $ and $ \gamma_s x_j $. 
Then the partial derivatives of the cell measure $ \omega_i(\varphi) $ are given by
\begin{align}
\frac{\partial \omega_i(\varphi)}{\partial \varphi_j} &= \frac{\partial \omega_j(\varphi)}{\partial \varphi_i} = - \sum_{s \in E_i^j} \frac{\sinh d_{ks} + \sinh d_{ls}}{\tanh d_{is} + \tanh d_{js}}, \label{eq:omega-grad2d-ij} \\
\frac{\partial \omega_i(\varphi)}{\partial \varphi_i} &= - \sum_{j \neq i} \frac{\partial \omega_j(\varphi)}{\partial \varphi_i}, \label{eq:omega-grad2d-i}
\end{align}

where $ d_{is} = d_H(x_i, q_{i,js}) $, $ d_{js} = d_H(\gamma_s x_j, q_{i,js}) $, and $ d_{ks}, d_{ls} $ be the hyperbolic lengths of the two parts into which the common edge $ U_{i,js}(\varphi) $ is divided by the point $ q_{i,js} $.
\end{corollary} 

\begin{proof}
From equation \ref{eq:edge-property}, we have $ e^{-\varphi_i} \cosh d_{is} = e^{-\varphi_j} \cosh d_{js} $, thus
\begin{equation}
\begin{split}
&e^{2\varphi_i}|e^{-\varphi_i}x_i - e^{-\varphi_j}\gamma_s x_j|^2 \\
= &\cosh^{-2} d_{js} \langle (\cosh d_{js}) x_i - (\cosh d_{is}) \gamma_s x_j, (\cosh d_{js}) x_i - (\cosh d_{is}) \gamma_s x_j \rangle_H \\
= &\cosh^{-2} d_{js} \left( -\cosh^2 d_{is} + 2\cosh d_{is} \cosh d_{js} \cosh (d_{is} + d_{js}) - \cosh^2 d_{js} \right) \\
= &\cosh^2 d_{is} \left( \tanh d_{is} + \tanh d_{js} \right)^2.
\end{split}
\end{equation}

Using equation \ref{eq:edge-property2}, for any $x \in U_{i,jk}(\varphi)$, we have
\begin{equation}
\frac{1}{|\nabla_x^H c_H(x, x_i) - \nabla_x^H c_H(x, \gamma_k x_j)|} = \frac{\cosh d_H(x_i,x)}{\cosh d_{is} (\tanh d_{is} + \tanh d_{js})}.
\end{equation}

Therefore,
\begin{equation}
\begin{split}
&\int_{U_{i,js}(\varphi)} \frac{1}{|\nabla_x^H c_H(x, x_i) - \nabla_x^H c_H(x, \gamma_k x_j)|} \, d\sigma_H^{m-1}(x) \\
= &\int_{U_{i,js}(\varphi)} \frac{\cosh d_H(x_i,x)}{\cosh d_{is} (\tanh d_{is} + \tanh d_{js})} \, d\sigma_H^{m-1}(x) \\
= &\int_{U_{i,js}(\varphi)} \frac{\cosh d_H(q_{i,js},x)}{\tanh d_{is} + \tanh d_{js}}  \, d\sigma_H^{m-1}(x) \\
= &\int_{-d_{ks}}^{d_{ls}} \frac{\cosh y}{\tanh d_{is} + \tanh d_{js}} \, dy \\
= &\, \frac{\sinh d_{ks} + \sinh d_{ls}}{\tanh d_{is} + \tanh d_{js}}.
\end{split}
\end{equation}

The theorem then follows from equations \ref{eq:kantorovich-d2i} and \ref{eq:kantorovich-d2ij}.
\end{proof}

\subsubsection{Existence of the Optimal Transport Map on Hyperbolic Space}
Now we arrive at the main theorem of this work.
\begin{theorem}
\label{thm:3}
Let $ M = \mathbb{H}^m/\Gamma $ be an $ m $-dimensional compact hyperbolic manifold, and let $ \mu $ be a probability measure on $ M $ with a continuous density function $ d\mu = f d\sigma_M $ with respect to the Riemannian measure $ \sigma_M $. 
Let $ \nu = \sum_{i=1}^k \nu_i \delta_{p_i} $ be a discrete measure, satisfying $ \min_{1 \leq i \leq k} \nu_i > 0 $, $ \{p_1, p_2, \dots, p_k\} \subset M $, and $ \mu(M) = \sum_{i=1}^k \nu_i $. 
Then, there exists a height vector $ \varphi = (\varphi_1, \varphi_2, \dots, \varphi_k) \in \mathbb{R}^k $ such that $ \mu(W_i(\varphi)) = \nu_i $ for all $ 1 \leq i \leq k $. This height vector is unique up to adding a constant vector $ (c, c, \dots, c) $, and the height vector $ \varphi $ minimizes the following convex energy function
\begin{equation}
E(\varphi) = \int_0^{\varphi} \sum_{i=1}^k \mu(W_i(\varphi)) \, d\varphi_i - \sum_{i=1}^k \varphi_i \nu_i,
\end{equation}

in the admissible height space
\begin{equation}
\Phi = \{ \varphi \in \mathbb{R}^k \mid \mu(W_i(\varphi)) > 0, \forall 1 \leq i \leq k \}.
\end{equation}

Furthermore, among all measure-preserving transport maps $ T: (M, \mu) \to (M, \nu) $, the map $ T: W_i(\varphi) \mapsto p_i, \forall 1 \leq i \leq k $, minimizes the hyperbolic transport cost
\begin{equation}
\int_M \ln \cosh d_M(x, T(x)) d\mu.
\end{equation}
\end{theorem}

\begin{proof}
The proof of Theorem \ref{thm:3} proceeds in the following steps.
\begin{enumerate}
\item Prove the admissible height space
\begin{equation}
H = \{(\rho_1, \rho_2, \dots, \rho_k) \in \mathbb{R}^{k}_+ \mid \varphi = (\ln \rho_1, \ln \rho_2, \dots, \ln \rho_k), \omega_i(\varphi) > 0, \forall 1 \leq i \leq k\}
\end{equation}
is a non-empty convex open set, and there exists a diffeomorphism $ g: H \to \Phi $, where
\begin{equation}
\Phi = \{ \varphi \in \mathbb{R}^k \mid \omega_i(\varphi) > 0, \forall 1 \leq i \leq k \}.
\end{equation}
This shows that $ \Phi $ is a simply connected set.

\item Prove that the function 
\begin{equation}
E(\varphi) = \int_0^{\varphi} \sum_{i=1}^k \mu(W_i(\varphi)) \, d\varphi_i - \sum_{i=1}^k \varphi_i \nu_i,
\end{equation}
is well-defined and $ C^1 $-smooth on $ \Phi $.

\item Prove that $ E(\varphi) $ is convex on $ \Phi $, and strictly convex on 
\begin{equation}
\Phi_0 = \Phi \cap \left\{ \varphi \in \mathbb{R}^k \mid \sum_{1 \leq i \leq k} \varphi_i = 0 \right\}.
\end{equation}

\item By studying the restriction $ E|_{\Phi_0} $, prove that $ E(\varphi) $ has a minimum point on $ \Phi $. Thus, the map $ \omega: \Phi_0 \to \Psi $ is a diffeomorphism, where
\begin{equation}
\omega(\varphi) = (\omega_1(\varphi), \omega_2(\varphi), \dots, \omega_k(\varphi)), \, \forall \varphi \in \Phi_0,
\end{equation}
and
\begin{equation}
\Psi = \left\{ (\nu_1, \nu_2, \dots, \nu_k) \in \mathbb{R}^k \mid \min_{1 \leq i \leq k} \nu_i > 0, \sum_{i=1}^k \nu_i = \mu(M) \right\}.
\end{equation}
\end{enumerate}

Step 1: Since $ f $ is a positive continuous function and $ M $ is a compact manifold, $ f $ has a lower bound $ L > 0 $. Thus, $ f_H(x) = f \circ \pi_\Gamma(x) \geq L, \forall x \in \mathbb{H}^m $. 
Note that a convex set $ U \subset \mathbb{H}^m $ has a positive $ \sigma_H $-measure if and only if its interior $ \text{int}(U) $ is non-empty. 
Therefore, $ \omega_i(\varphi) > 0 $ is equivalent to $ U_i(\varphi) $ having a non-empty interior. 
This is also equivalent to the corresponding $ \Gamma $-convex polyhedron $ \mathcal{P}_\varphi $ having a non-empty interior on the face $ F_i(\varphi) $.

Let $ \zeta, \xi \in H $, and $ 0 < t < 1 $. We define $ \mathcal{P}_\zeta $ and $ \mathcal{P}_\xi $ as the sets of points $ \{ \zeta_i^{-1} \gamma x_i \mid 1 \leq i \leq k, \gamma \in \Gamma \} $ and $ \{ \xi_i^{-1} \gamma x_i \mid 1 \leq i \leq k, \gamma \in \Gamma \} $, respectively, which are the $ \Gamma $-convex polyhedrons obtained by applying the hyperbolic Legendre dual.
Thus, $ \partial \mathcal{P}_\zeta $ is the upper envelope of the family of functions $ \left\{ -\frac{\zeta_i}{\langle \gamma x_i, y \rangle_H} \right\}_{1 \leq i \leq k, \gamma \in \Gamma} $, and $ \partial \mathcal{P}_\xi $ is the upper envelope of the family of functions $ \left\{ -\frac{\xi_i}{\langle \gamma x_i, y \rangle_H} \right\}_{1 \leq i \leq k, \gamma \in \Gamma} $.
Therefore, the point set $ \{(t \zeta_i + (1 - t) \xi_i)^{-1} \gamma x_i \mid 1 \leq i \leq k, \gamma \in \Gamma \} $, after applying the hyperbolic Legendre dual, results in the $ \Gamma $-convex polyhedron which is the Minkowski sum of $ t \mathcal{P}_\zeta \oplus (1 - t) \mathcal{P}_\xi $.
Since $ \zeta, \xi \in H $, all the faces of $ \mathcal{P}_\zeta $ and $ \mathcal{P}_\xi $ have non-empty interiors. 
By the Brunn-Minkowski inequality (Theorem 7.1.1 \cite{schneider2013convex}), we conclude that all the faces of $ t \mathcal{P}_\zeta \oplus (1 - t) \mathcal{P}_\xi $ have non-empty interiors.
Therefore, $ t\zeta + (1 - t)\xi \in H $. This shows that $ H $ is a convex set. Furthermore, by definition, $ H $ is an open set.

Let $ \rho = (1, 1, \dots, 1) $, we claim that $ \rho \in H $. In fact, $ \varphi = (0, 0, \dots, 0) $, and for $ 1 \leq i \leq k $ we have
\begin{equation}
U_i(\varphi) = \{ y \in \mathbb{H}^m \mid \langle x_i, y \rangle_H \geq \langle \gamma x_j, y \rangle_H, \forall 1 \leq j \leq k, \forall \gamma \in \Gamma \}.
\end{equation}

Thus,
\begin{equation}
-1 = \langle x_i, x_i \rangle_H > \langle \gamma x_j, x_i \rangle_H, \forall 1 \leq j \leq k, \forall \gamma \in \Gamma \text{ s.t. } \gamma x_j \neq x_i.
\end{equation}

This shows that $ x_i $ is an interior point of $ U_i(\varphi) $, meaning that $ U_i(\varphi) $ has a non-empty interior. Therefore, $ H \neq \varnothing $.

Consider the map $ g : H \to \Phi $, defined by
\begin{equation}
g(\rho) = (\ln \rho_1, \ln \rho_2, \dots, \ln \rho_k), \, \forall \rho \in H.
\end{equation}

We can see that $ g $ is a diffeomorphism from $ H $ to $ \Phi $. Since convex sets in $ \mathbb{R}^k $ are simply connected, both $ H $ and $ \Phi $ are simply connected open sets. \\

Step 2: According to Theorem \ref{thm:2}, the differential form $ \eta = \sum_{i=1}^k (\omega_i(\varphi) - \nu_i) d\varphi_i $ is a closed form defined on the simply connected open set $ \Phi $. 
Therefore, it is exact, meaning the integral $ E(\varphi) = \int_0^\varphi \eta = \int_0^\varphi \sum_{i=1}^k \omega_i(\varphi) d\varphi_i - \sum_{i=1}^k \nu_i \varphi_i $ is well-defined as a $ C^1 $ smooth function, independent of the path chosen from 0 to $ \varphi $ in $ \Phi $.
Thus we have
\begin{align}
&E(\varphi + c(1, 1, \dots, 1)) = E(\varphi), \, \forall c \in \mathbb{R}, \\
&\text{grad } E(\varphi) = (\omega_1(\varphi) - \nu_1, \omega_2(\varphi) - \nu_2, \dots, \omega_k(\varphi) - \nu_k), \label{eq:grad-E} \\
&\text{Hess}(E(\varphi)) = \left[ \frac{\partial \omega_i(\varphi)}{\partial \varphi_j} \right].
\end{align}

Step 3: Notice that the condition $\omega_i(\varphi) > 0 $ leads to $ \frac{\partial \omega_i(\varphi)}{\partial \varphi_i} > 0. $ Then, according to Theorem \ref{thm:2}, we know that the Hessian matrix $ [h_{ij}] := \text{Hess}(E(\varphi)) $ is positive semi-definite and has a one-dimensional null space spanned by the vector $ (1, 1, \ldots, 1) $. Therefore, $ E(\varphi) $ is convex on $ \Phi $.

Assume $ v = (v_1, v_2, \dots, v_k) \in \mathbb{R}^k $ is a non-zero vector and that $ [h_{ij}] v = 0 $. 
Without loss of generality, let $ v_1 > 0 $ and $ |v_1| = \max_{1 \leq j \leq k} |v_j| $. 
Since $ h_{11} v_1 = - \sum_{j \neq 1} h_{1j} v_j $ and $ h_{11} = - \sum_{j \neq 1} h_{1j} $, we have $ \sum_{j \neq 1} h_{1j} (v_1 - v_j) = 0$.
From equation \ref{eq:kantorovich-d2ij}, we know that $ h_{1j} \leq 0, \forall j \neq 1$. We obtain that if $ h_{1j} \neq 0 $, then $ v_j = v_1 $. 
Therefore, the index set $ I = \{i \mid v_i = \max_{1 \leq j \leq k} |v_j|\} $ has the following property: if $ i_1 \in I $ and $ h_{i_1 i_2} \neq 0 $, then $ i_2 \in I $. 

Next, we claim that $ I = \{1, 2, \dots, k\} $. In fact, for any two indices $i \neq j$, we can find a sequence of indices $i_1 = i, i_2 \dots, i_s = j$ such that for each $1 \leq t \leq s - 1$, $U_{i_t}(\varphi)$ and $U_{i_{t+1}}(\varphi)$ intersect in an edge. Thus, $h_{i_t i_{t+1}} \neq 0$, for all $1 \leq t \leq s - 1$. 
Therefore, there exists a constant $c$ such that $v = c(1, 1, \dots, 1)$. This means that the null space of the Hessian matrix $\text{Hess}(E(\varphi))$ has dimension 1. This shows that $\text{Hess}(E(\varphi))$ is positive definite on $\Phi_0$, and thus $E(\varphi)$ is strictly convex on $\Phi_0$. \\

Step 4: We claim that $\Phi_0$ is bounded. Otherwise, there exists a sequence of vectors $\varphi^{(n)}$ in $\Phi_0$, and indices $1 \leq i, j \leq k$, such that $\lim_{n \to \infty} \varphi_i^{(n)} = \infty$ and $\lim_{n \to \infty} \varphi_j^{(n)} = -\infty$. Then, when $n$ is sufficiently large,
\begin{equation}
e^{-\varphi_i^{(n)}} \langle x_i, y \rangle_H > e^{-\varphi_j^{(n)}} \langle x_j, y \rangle_H, \, \forall y \in \mathbb{H}^m.
\end{equation}

This shows that $ U_j(\varphi^{(n)}) = \varnothing $, which contradicts the assumption that $ \varphi^{(n)} \in \Phi_0 $.

Now we extend $ E(\varphi) $ continuously to the closed set $ \overline{\Phi_0} $, and thus there exists a minimum point $ \varphi^0 \in \overline{\Phi_0} $.
We claim that $ \varphi^0 \in \Phi_0 $. Otherwise, if $ \varphi^0 \in \partial \Phi_0 $, the index set $ J = \{ i \mid \omega_i(\varphi^0) = 0 \} $ is non-empty.
By the convexity of the set $ \Phi_0 $, there exists a non-zero vector $ v \in \mathbb{R}^k $ such that $ \varphi^0 + t v \in \Phi_0 $ for sufficiently small $ t > 0 $.
Thus, $ \varphi^0 + t(v + c(1, 1, \dots, 1)) \in \Phi $ holds for sufficiently small $ t > 0 $. Therefore, by increasing $ c $, we can assume that all $ v_i > 0 $, and that $ \varphi^0 + t v \in \Phi $ holds for sufficiently small $ t > 0 $. 
Let the vector $ \delta \in \mathbb{R}^k $ satisfy: (1) $ \delta_i = v_i $, if $ \omega_i(\varphi^0) = 0 $; (2) $ \delta_i = 0 $, if $ \omega_i(\varphi^0) > 0 $.
By continuity, if $ \omega_i(\varphi^0) > 0 $, then $ \omega_i(\varphi^0 + t \delta) > 0 $ holds for sufficiently small $ t > 0 $. If $ \omega_i(\varphi^0) = 0 $, we claim that $ U_i(\varphi^0 + t v) \subset U_i(\varphi^0 + t \delta) $. 
In fact, for any $ y \in U_i(\varphi^0 + t v) $, and $ \forall 1 \leq j \leq k, \gamma \in \Gamma $, we have
\begin{equation}
\begin{split}
e^{-\varphi_i^0 - t \delta_i} \langle x_i, y \rangle_H &= e^{-\varphi_i^0 - t v_i} \langle x_i, y \rangle_H \\
&\geq e^{-\varphi_j^0 - t v_j} \langle \gamma x_j, y \rangle_H \\
&\geq e^{-\varphi_j^0 - t \delta_j} \langle \gamma x_j, y \rangle_H,
\end{split}
\end{equation}

Therefore, $ y \in U_i(\varphi^0 + t \delta) $. Thus, $ \varphi^0 + t \delta \in \Phi $. 
Since $ \varphi^0 $ is also a minimum point of $ E(\varphi) $ on the closed set $ \overline{\Phi} $, we have $ E(\varphi^0 + t \delta) \geq E(\varphi^0) $ for sufficiently small $ t > 0 $. 
This implies that
\begin{equation}
0 \leq \frac{d}{dt} \left. E\left( \varphi^0 + t \delta \right) \right|_{t=0} = \sum_{i=1}^k \left( \omega_i(\varphi^0) - \nu_i \right) \delta_i.
\end{equation}

On the other hand, from the construction of $ \delta $, we know that $\sum_{i=1}^k (\omega_i(\varphi^0) - \nu_i) \delta_i = - \sum_{i \in J} \nu_i v_i < 0$, which contradicts the previous equation. Therefore, $ \varphi^0 \in \Phi_0 $.

Let $ F(\varphi) = E(\varphi) + \sum_{i=1}^k \nu_i \varphi_i $, then $ \omega(\varphi) = (\omega_1(\varphi), \omega_2(\varphi), \dots, \omega_k(\varphi)) $ is the gradient map of $ F(\varphi) $ restricted to $ \Phi_0 $, \ie $ \omega(\varphi) = \text{grad } F(\varphi) = \text{grad } E(\varphi) + \nu $. 
On the one hand, because $ \text{Hess}(E(\varphi)) $ is symmetric and positive definite on $ \Phi_0 $, $ \omega(\varphi) $ is a local diffeomorphism from $ \Phi_0 $ to $ \Psi $. 
On the other hand, for any $ \nu \in \Psi $, since $ E(\varphi) = F(\varphi) - \sum_{i=1}^k \nu_i \varphi_i $ has a minimum point on $ \Phi_0 $, this shows that $ \omega(\varphi) $ is surjective. 
Therefore, $ \omega(\varphi) $ is a diffeomorphism from $ \Phi_0 $ to $ \Psi $.

In conclusion, we have shown that for any $ \nu \in \Psi $, $ E(\varphi) $ has a minimum value on $ \Phi $, and it is unique up to adding a constant vector $ (c, c, \dots, c) $. At the same time, such a minimum point satisfies the condition in equation \ref{eq:face-measure}.
By combining Theorem \ref{thm:1}, this completes the proof of Theorem \ref{thm:3}.
\end{proof}

In short, the above theorem states that there exists a height vector $ \varphi $, for which the radial projection of the corresponding convex polyhedron $ \mathcal{P}_\varphi $ gives the hyperbolic optimal transport map under the hyperbolic transportation cost function, as illustrated in Figure \ref{fig:method}. 

\begin{figure}[h]
\begin{center}
\fbox{\rule{0pt}{2in}
\includegraphics[width=0.8\linewidth]{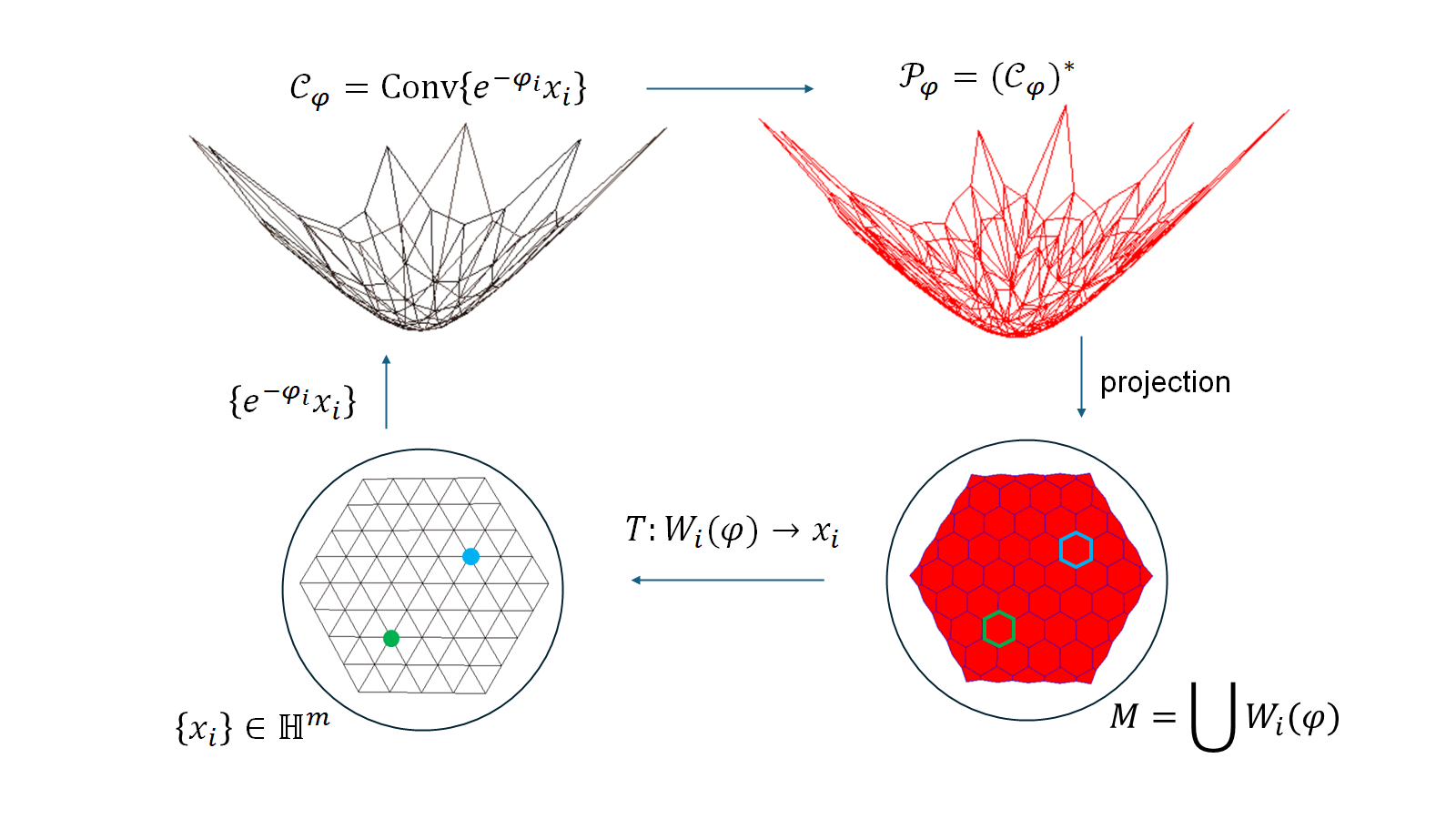}}
\end{center}
   \caption{Geometric variational method}
\label{fig:method}
\end{figure}

Next, we will prove the discrete case of Theorem 1.3 from \cite{bertrand2014prescription}, which is stated in the following theorem.

\begin{theorem}
\label{thm:4}
Let $M = \mathbb{H}^m / \Gamma$ be an $m$-dimensional compact hyperbolic manifold, where $\sigma_M$ is the Riemannian measure on $M$, and $\nu = \sum_{i=1}^k \nu_i \delta_{p_i}$ is a discrete measure with $\min_{1 \leq i \leq k} \nu_i > 0$, $\{ p_1, p_2, \dots, p_k \} \subset M$, and $\sigma_M(M) = \sum_{i=1}^k \nu_i$. 
Then, there exists a homeomorphism $\psi: \mathbb{H}^m \to \mathbb{R}^m_1$ that is invariant under the action of the group $\Gamma$, such that $\psi(\mathbb{H}^m)$ is the boundary of a convex polyhedron in $\mathbb{R}^m_1$ with a discrete Gauss curvature measure $\nu$, and this homeomorphism is unique up to homothety.
\end{theorem}

\begin{proof}
To prove this theorem, we assume $ \mu = \sigma_M $, $ \nu = \sum_{i=1}^k \nu_i \delta_{p_i} $ with $ \min_{1 \leq i \leq k} \nu_i > 0 $, $ \{ p_1, p_2, \dots, p_k \} \subset M $, and $ \sum_{i=1}^k \nu_i = \sigma_M(M) $. 
Let $ \mathcal{C}_\varphi $ be the convex hull of the set $ \{ e^{-\varphi_i} \gamma x_i \mid 1 \leq i \leq k, \gamma \in \Gamma \} $ in $ \mathbb{R}^m_1 $. 
We claim that the Gauss curvature measure $ \mu_{\mathcal{C}} $ of the $ \Gamma $-convex polyhedron is discrete, and its support set is the set of points $ e^{-\varphi_i} \gamma x_i $ for all $ 1 \leq i \leq k $ and $ \gamma \in \Gamma $.
In fact, from definition \ref{def:gauss-curvature} and Proposition \ref{prop:isomorphic}, $\forall 1 \leq i \leq k, \forall \gamma \in \Gamma$,
\begin{equation}
\mu_{\mathcal{C}}(e^{-\varphi_i} \gamma x_i) = \sigma_H(\partial \rho(\gamma x_i)) = \sigma_H(U_{i,\gamma}(\varphi)) = \sigma_M(W_i(\varphi)) = \omega_i(\varphi).
\end{equation}

For $ z \in \mathcal{C}_\varphi $, if $ z $ is not a vertex of $ \mathcal{C}_\varphi $, then the subnormal $ \partial \rho(z) $ corresponds to a a vertex or geodesic edge of $ \mathcal{D}_\varphi $, and thus $ \mu_{\mathcal{C}}(z) = \sigma_H(\partial \rho(z)) = 0 $.

The above shows that the existence of a $ \Gamma $-convex polyhedron with a given discrete Gaussian curvature measure in Minkowski spacetime is equivalent to finding a cell decomposition of $ M $ with given cell measures. 
According to Theorem \ref{thm:3}, there exists $ \varphi = (\varphi_1, \varphi_2, \dots, \varphi_k) \in \mathbb{R}^k $ such that $ \omega_i(\varphi) = \nu_i, \forall 1 \leq i \leq k $. 
Then $ \mathcal{C}_\varphi $ is a $ \Gamma $-convex polyhedron with a Gauss curvature measure $ \nu $. 
Since such $ \varphi $ is unique up to adding a constant vector $ (c, c, \dots, c) $, $ \mathcal{C}_\varphi $ is unique up to homothety. 
In addition, the inverse of the radial projection $ p $ gives a homeomorphism $ \psi: \mathbb{H}^m \to \mathbb{R}^m_1 $. This completes the proof.
\end{proof}

\subsection{Hyperbolic Power Diagram}

In this section, we establish the theory of hyperbolic power diagrams and hyperbolic weighted Delaunay triangulations, in order to propose an efficient and stable numerical method for solving hyperbolic optimal transport maps. 

Let $ \alpha $, $ \beta $, $ \gamma $ be the three inner angles of a hyperbolic triangle $ T $ on $ \mathbb{H}^2 $, and let $ a $, $ b $, $ c $ be the hyperbolic lengths of the sides opposite these angles. The hyperbolic law of cosines states that:
\begin{equation}
\label{eq:cos-law}
\cosh c = \cosh a \cosh b - \sinh a \sinh b \cos \gamma
\end{equation}

The hyperbolic area formula for triangle $ T $ is given by:
\begin{equation}
\label{eq:triangle-area}
\text{Area}(T) = \pi - \alpha - \beta - \gamma.
\end{equation}

Given $ p \in \mathbb{H}^m $ and $ r > 0 $, a hyperbolic geodesic circle is defined as $ c = \{ q \in \mathbb{H}^m \mid d_H(q, p) = r \} $, where $ p $ is the center of the circle $ c $, and $ r $ is its radius. 
Let $ q \in \mathbb{H}^m $ be a point outside the circle $ c $. We can draw two geodesics that are tangent to the circle $ c $. 
The distance from $ q $ to each tangent point is called the hyperbolic power distance from $ q $ to $ c $, denoted as $ \text{pow}(q, c) $. 
According to the hyperbolic law of cosines \ref{eq:cos-law}, the radius $ r $ of the circle $ c $, the hyperbolic distance $ d_H(q, p) $ from $ p $ to $ q $, and the hyperbolic power distance $ \text{pow}(q, c) $ satisfy the following relation:
\begin{equation}
\label{eq:power-distance}
\text{pow}(q, c) = \frac{\cosh d_H(q, p)}{\cosh r} = - \frac{\langle q,p \rangle_H}{\cosh r}.
\end{equation}

Let $ C = \{ c_i(p_i, r_i) \mid i \in I \} $ be a set of hyperbolic geodesic circles on $ \mathbb{H}^m $. The hyperbolic space $ \mathbb{H}^m $ can be divided into cells based on the hyperbolic power distance, leading to the construction of a hyperbolic power diagram.

\begin{definition}
Given a set of hyperbolic geodesic circles $ C = \{ c_i(p_i, r_i) \mid i \in I \} \subset \mathbb{H}^m $, the hyperbolic power diagram is a cell decomposition on the hyperbolic space $ \mathbb{H}^m $,
\begin{equation}
\mathbb{H}^m = \bigcup_{i \in I} U_i(C),
\end{equation}
where each hyperbolic power cell $U_i(C)$ is defined as
\begin{equation}
U_i(C) = \{ p \in \mathbb{H}^m \mid \text{pow}(p, c_i) \leq \text{pow}(p, c_j), \forall j \in I \}.
\end{equation}

For the geodesic circles $ c_i $ and $ c_j $, the Laguerre bisector is defined as:
\begin{equation}
LB(c_i, c_j) = \{ p \in \mathbb{H}^m \mid \text{pow}(p, c_i) = \text{pow}(p, c_j) \}.
\end{equation}
\end{definition}

\begin{proposition}
Given hyperbolic geodesic circles $ c_i $ and $ c_j $, the Laguerre bisector $ LB(c_i, c_j) $ is a geodesic on the hyperbolic space $ \mathbb{H}^m $, and it is perpendicular to the geodesic $ L_{ij} $ passing through $ p_i $ and $ p_j $.
\end{proposition}

\begin{proof}
From equation \ref{eq:power-distance}, we have
\begin{equation}
p \in LB(c_i, c_j) \Leftrightarrow \text{pow}(p, c_i) = \text{pow}(p, c_j) \Leftrightarrow \frac{\langle p,p_i \rangle_H}{\cosh r_i} = \frac{\langle p,p_j \rangle_H}{\cosh r_j}.
\end{equation}

This implies that the Laguerre bisector $ LB(c_i, c_j) $ is a geodesic on $ \mathbb{H}^m $.

Let $ p_{ij} $ be the intersection point of the Laguerre bisector $ LB(c_i, c_j) $ and the geodesic $ L_{ij} $ passing through $ p_i $ and $ p_j $, then
\begin{equation}
\frac{\langle p_{ij},p_i \rangle_H}{\cosh r_i} = \frac{\langle p_{ij},p_j \rangle_H}{\cosh r_j}.
\end{equation}

Thus, for $ p \in LB(c_i, c_j) $, we have:
\begin{equation}
\frac{\langle p, p_i \rangle_H}{\langle p_{ij}, p_i \rangle_H} = \frac{\langle p, p_j \rangle_H}{\langle p_{ij}, p_j \rangle_H} \quad \Rightarrow \quad \frac{d_H(p, p_i)}{d_H(p_{ij}, p_i)} = \frac{d_H(p, p_j)}{d_H(p_{ij}, p_j)}.
\end{equation}

Thus, the geodesic connecting $ p $ and $ p_{ij} $ is perpendicular to $ L_{ij} $, and so $ LB(c_i, c_j) $ is perpendicular to $ L_{ij} $.
\end{proof}

The above proposition shows that the $ \sigma_H $-measure of each Laguerre bisector is zero. For convenience, for a set of geodesic circles $ \{(p_i, r_i)\} $, we use $ \text{pow}(p, p_i) $ to denote $ \text{pow}(p, c_i) $.

\begin{definition}
Given a set of hyperbolic geodesic circles $ C = \{(p_i, r_i) \mid i \in I \} $, a triangulation with vertices at the centers of the geodesic circles and dual to the hyperbolic power diagram is called the hyperbolic weighted Delaunay triangulation. 
Points $ p_i $ and $ p_j $ are connected in the hyperbolic weighted Delaunay triangulation if and only if the cells $ U_i(C) $ and $ U_j(C) $ are adjacent in the hyperbolic power diagram.
\end{definition}

Next, we explore the differential properties of the cell areas in the power diagram on the two-dimensional hyperbolic plane $ \mathbb{H}^2 $. 
Let $ \triangle p_i p_j p_k $ be a hyperbolic weighted triangle on $ \mathbb{H}^2 $, where $ p_i, p_j, p_k $ are the vertices and $ r_i, r_j, r_k $ are the corresponding radii of the geodesic circles of the vertices. 
Let $ o \in \mathbb{H}^2 $ be the hyperbolic power center of the triangle, meaning that the hyperbolic power distances to the three vertices are equal:
\begin{equation}
\label{eq:power-center}
\text{pow}(o, p_i) = \text{pow}(o, p_j) = \text{pow}(o, p_k) = \cosh R_{ijk},
\end{equation}
where $ R_{ijk} $ is called the hyperbolic power radius of the triangle.

We can then draw geodesics perpendicular to the three sides of the triangle through the hyperbolic power center, with the base of the perpendiculars being $ q_i, q_j, q_k $ respectively, and the geodesic distances from the hyperbolic power center to the base of the perpendiculars being $ d_i, d_j, d_k $.

\begin{figure}[h]
\begin{center}
\fbox{\rule{0pt}{2in}
\includegraphics[width=0.8\linewidth]{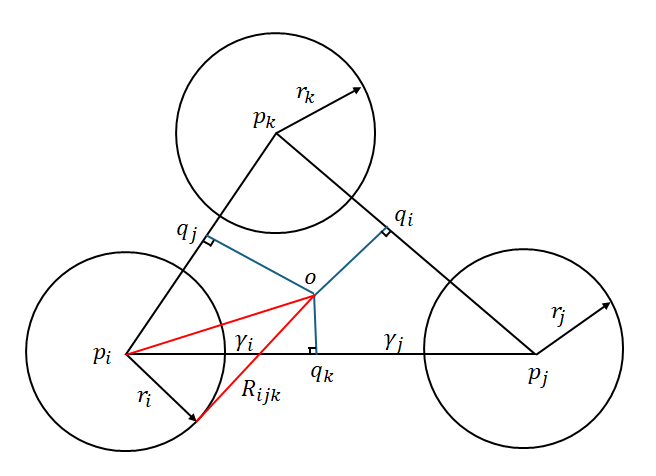}}
\end{center}
   \caption{Hyperbolic weighted triangle}
\label{fig:triangle}
\end{figure}

\begin{proposition}
Let $ \triangle p_i p_j p_k $ be a hyperbolic weighted triangle, as shown in Figure \ref{fig:triangle}. 
Then, the hyperbolic power center $ o $ and the hyperbolic power radius $ R_{ijk} $ are given by the following system of linear equations:
\begin{equation}
\label{eq:linear-system}
\begin{split}
o^\top J o &= -1, \\
(p_i, p_j, p_k)^\top J o &= - \cosh R_{ijk} (\cosh r_i, \cosh r_j, \cosh r_k)^\top
\end{split}
\end{equation}
where $J = diag(1,1,-1)$.
\end{proposition}

\begin{proof}
Since $ o \in \mathbb{H}^2 $, we have $ o^\top J o = \langle o, o \rangle_H = -1 $.

From equation \ref{eq:power-center}, we have
\begin{equation}
\cosh R_{ijk} = \text{pow}(o, p_i) = \frac{\cosh d_H(o, p_i)}{\cosh r_i} = - \frac{p_i^\top J o}{\cosh r_i}.
\end{equation}

Similarly,
\begin{equation}
p_j^\top J o = - \cosh R_{ijk} \cosh r_j, \quad p_k^\top J o = - \cosh R_{ijk} \cosh r_k.
\end{equation}
\end{proof}

\begin{proposition}
Let $ \triangle p_i p_j p_k $ be a hyperbolic weighted triangle, as shown in Figure \ref{fig:triangle}.
Then, the following partial derivatives hold:
\begin{equation}
\label{eq:dgamma-dh}
\frac{d\gamma_i}{dh_j} = \frac{d\gamma_j}{dh_i} = - \frac{1}{\tanh \gamma_i + \tanh \gamma_j},
\end{equation}
where $ h_i = \ln \cosh r_i $, $ h_j = \ln \cosh r_j $, $ \gamma_i = d_H(p_i, q_k) $, $ \gamma_j = d_H(p_j, q_k) $, and $ \gamma_{ij} = \gamma_i + \gamma_j $.
\end{proposition}

\begin{proof}
Since $ q_k \in LB(p_i, p_j) $, we have
\begin{equation}
\cosh \gamma_i e^{-h_i} = \cosh \gamma_j e^{-h_j}.
\end{equation}

By differentiating both sides of the above equation with respect to $ h_i $, we get
\begin{equation}
\sinh \gamma_i \frac{d\gamma_i}{dh_i} e^{-h_i} - \cosh \gamma_i e^{-h_i} = \sinh \gamma_j \frac{d\gamma_j}{dh_i} e^{-h_j}.
\end{equation}

Note that the equation $ \gamma_{ij} = \gamma_i + \gamma_j $ implies that
$\frac{d\gamma_i}{dh_i} + \frac{d\gamma_j}{dh_i} = 0$. Therefore, we have
\begin{equation}
\begin{split}
\frac{d\gamma_j}{dh_i} &= - \frac{\cosh \gamma_i e^{-h_i}}{\sinh \gamma_i e^{-h_i} + \sinh \gamma_j e^{-h_j}} \\
&= - \frac{\cosh \gamma_i e^{-h_i}}{\sinh \gamma_i e^{-h_i} + \tanh \gamma_j \cosh \gamma_j e^{-h_j}} \\ 
&= - \frac{\cosh \gamma_i e^{-h_i}}{\sinh \gamma_i e^{-h_i} + \tanh \gamma_j \cosh \gamma_i e^{-h_i}} \\ 
&= - \frac{1}{\tanh \gamma_i + \tanh \gamma_j}.
\end{split}
\end{equation}
\end{proof}

To compute the partial derivatives of the area of each cell $ U_i(C) $, we need to calculate the area of a hyperbolic quadrilateral. 
Now consider the upper half-plane model, $ \mathbb{U} = \{ z = x + iy \in \mathbb{C} \mid y > 0 \} $, whose hyperbolic metric is $ ds = \frac{|dz|}{y} $. 
For $ a > 1 $, the hyperbolic distance between the points $ i $ and $ ia $ is given by
\begin{equation*}
d_U(i, ia) = \int_1^a \frac{dy}{y} = \ln a.
\end{equation*}

A Saccheri quadrilateral is a hyperbolic quadrilateral in $ \mathbb{H}^2 $ that has two sides of equal hyperbolic length and is perpendicular to the third side. We refer to the third side as the base of the quadrilateral. Below, we provide the formula for the area of a Saccheri quadrilateral.

\begin{figure}
\centering
\begin{minipage}{.5\textwidth}
  \centering
  \captionsetup{width=.9\linewidth}
  \fbox{\rule{0pt}{2in}
  \includegraphics[width=.8\linewidth]{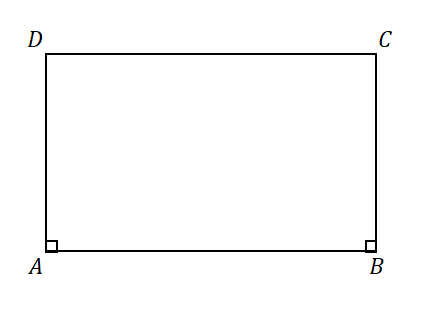}}
  \caption{Saccheri Quadrilateral ABCD}
  \label{fig:quad-abcd}
\end{minipage}%
\begin{minipage}{.5\textwidth}
  \centering
  \captionsetup{width=.9\linewidth}
  \fbox{\rule{0pt}{2in}
  \includegraphics[width=.8\linewidth]{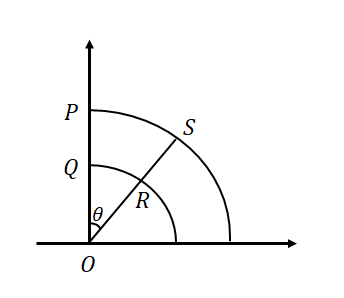}}
  \caption{Saccheri Quadrilateral PQRS}
  \label{fig:quad-pqrs}
\end{minipage}
\end{figure}

\begin{proposition}
\label{prop:quad-area}
Let $ ABCD $ be a Saccheri quadrilateral, where the base is $ AB $, as shown in Figure \ref{fig:quad-abcd}. 
Let the hyperbolic lengths of sides $ AB $ and $ BC $ be $ a $ and $ b $, respectively. Then, the area of the quadrilateral $ \overline{ABCD} $ is given by
\begin{equation}
\text{Area}(ABCD) = a \sinh b.
\end{equation}
\end{proposition}

\begin{proof}
For any Saccheri quadrilateral in $ \mathbb{H}^2 $, it is isometric to some Saccheri quadrilateral in $ \mathbb{U} $. Therefore, we prove this proposition in the upper half-plane model.

Let $ P = (0, e^a) $, $ Q = (0, 1) \in \mathbb{C} $. Then we have $d_U(P, Q) = \ln e^a = a$.
This shows that there exists a transformation $ \gamma \in \text{Isom}(\mathbb{U}) $ such that $ \gamma $ maps the geodesic segment $ AB $ to the geodesic segment $ PQ $, where $ PQ $ is the vertical line segment from $ P $ to $ Q $. In addition, $ \gamma(A) = P $ and $ \gamma(B) = Q $. 
Let $ R $ lie on the unit circle centered at $ O $, satisfying $ \text{Re}(R) > 0 $ and $ d_U(Q, R) = b $. 
Let $ S $ lie on the circle centered at $ O $ with radius $ e^a $, satisfying $ \text{Re}(S) > 0 $ and $ d_U(P, S) = b $. 
Thus, the points $ O $, $ R $, and $ S $ are collinear, as shown in Figure \ref{fig:quad-pqrs}.
Since the two opposite sides $ AD $ and $ BC $ are both perpendicular to $ AB $, $ \gamma $ maps $ AD $ to $ PS $ and $ BC $ to $ QR $. 
Therefore, $ \gamma $ isometrically maps the quadrilateral $ ABCD $ to the quadrilateral $ PQRS $. Hence, $\text{Area}(ABCD) = \text{Area}(PQRS)$.

Let $ \theta = \angle QOR $, then we have
\begin{equation*}
\int_0^\theta \frac{1}{\cos x} \, dx = b.
\end{equation*}

Therefore, $ b = \ln(\tan \theta + \sec \theta) $. Thus,
\begin{equation}
\begin{split}
\text{Area}(ABCD) &= \text{Area}(PQRS) = \int_0^\theta \int_1^{e^a} \frac{1}{r^2 \cos^2 x} r \, dr \, d\theta \\
&= a \tan \theta = a \sinh b.
\end{split}
\end{equation}
\end{proof}

\begin{figure}[h]
\begin{center}
\fbox{\rule{0pt}{2in}
\includegraphics[width=0.8\linewidth]{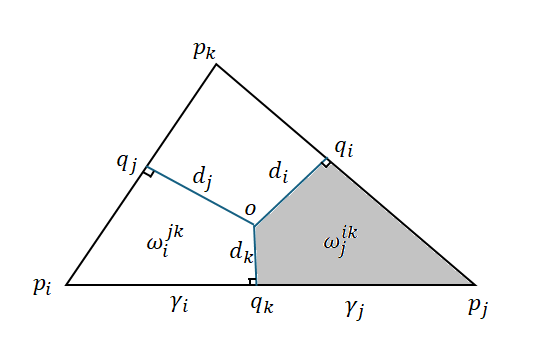}}
\end{center}
   \caption{Computing cell area in a hyperbolic weighted triangle}
\label{fig:triangle-2}
\end{figure}

\begin{proposition}
Let $ w_j^{ik} := \sigma_{\mathbb{H}^2}(\triangle p_ip_jp_k \cap U_j(C)) $, as shown in Figure \ref{fig:triangle-2}.
Then, the partial derivative of $ w_j^{ik} $ with respect to $ h_i = \ln \cosh r_i $ is given by:
\begin{equation}
\label{eq:cell-grad1}
\frac{\partial}{\partial h_i} w_j^{ik} = - \frac{\sinh d_k}{\tanh \gamma_i + \tanh \gamma_j}.
\end{equation}
\end{proposition}

\begin{proof}
When $ h_i $ changes to $ h_i + \delta h_i $ and $ \gamma_j $ shrinks to $ \gamma_j - \delta \gamma_j $, the power center changes from $ o $ to $ o_1 $. 
The region of change for $ U_j(C) $ can be represented by a Saccheri quadrilateral of base length $\delta \gamma_j$ and side length $d_k$, and a higher-order infinitesimal hyperbolic triangle. 
Thus, based on Proposition \ref{prop:quad-area} and equation \ref{eq:dgamma-dh}, we have
\begin{equation*}
\delta \gamma_j \sinh d_k = -\sinh d_k \frac{1}{\tanh \gamma_i + \tanh \gamma_j} \delta h_i.
\end{equation*}

Therefore, we obtain
\begin{equation}
\frac{\partial}{\partial h_i} w_j^{ik} = - \frac{\sinh d_k}{\tanh \gamma_i + \tanh \gamma_j}.
\end{equation}
\end{proof}

\begin{figure}[h]
\begin{center}
\fbox{\rule{0pt}{2in}
\includegraphics[width=0.8\linewidth]{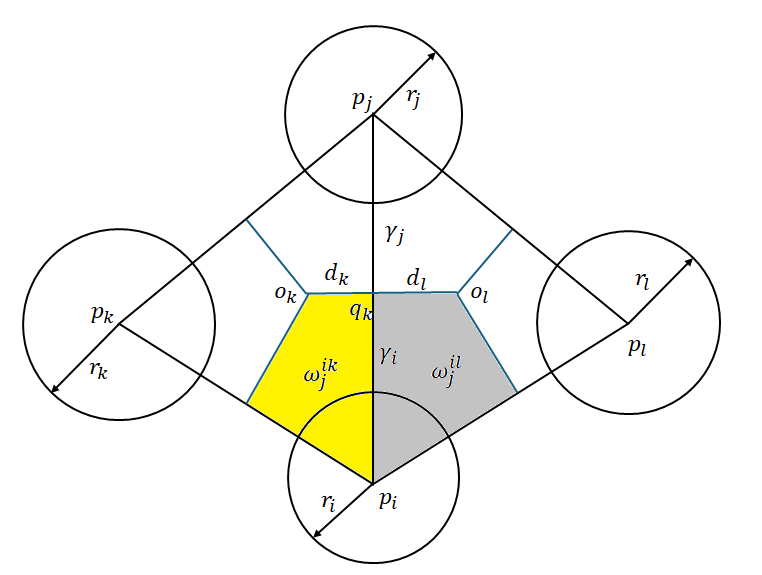}}
\end{center}
   \caption{Computing the hyperbolic power cell area}
\label{fig:power-diagram}
\end{figure}

Based on the above proposition, we can now prove the following theorem.
\begin{theorem}
\label{thm:5}
The partial derivatives of the hyperbolic power cell area are given by the following:
\begin{equation}
\begin{split}
\frac{\partial w_i}{\partial h_i} &= -\sum_{j \neq i} \frac{\partial w_j}{\partial h_i}, \\
\frac{\partial w_i}{\partial h_j} &= \frac{\partial w_j}{\partial h_i} = -\frac{\sinh d_k + \sinh d_l}{\tanh \gamma_i + \tanh \gamma_j}.
\end{split}
\end{equation}
\end{theorem}

\begin{proof}
Let $ \Omega $ be a compact convex region on $ \mathbb{H}^2 $, and $ C = \{(p_i, r_i) | p_i \in \Omega, r_i > 0, 1 \leq i \leq k\} $ be a family of hyperbolic geodesic circles. 
The hyperbolic power diagram formed by $ C $ is symmetric, as shown in Figure \ref{fig:power-diagram}. 
Let $ w_i = \sigma_{\mathbb{H}^2}(U_i(C) \cap \Omega) $ and $ h_i = \ln \cosh r_i, \forall 1 \leq i \leq k $. 

From equation \ref{eq:cell-grad1}, we have
\begin{equation}
\begin{split}
\frac{\partial w_i}{\partial h_j} = \frac{\partial w_j^{ik}}{\partial h_i}  + \frac{\partial w_j^{il}}{\partial h_i} &= - \frac{\sinh d_k}{\tanh \gamma_i + \tanh \gamma_j} - \frac{\sinh d_l}{\tanh \gamma_i + \tanh \gamma_j} \\
&= - \frac{\sinh d_k + \sinh d_l}{\tanh \gamma_i + \tanh \gamma_j}.
\end{split}
\end{equation}

Note that $ \sum_{i=1}^{k} \omega_i = \sum_{i=1}^{k} \sigma_{\mathbb{H}^2}(U_i(C) \cap \Omega) = \sigma_{\mathbb{H}^2}(\Omega) $. Thus, we have
\begin{equation}
\frac{\partial w_i}{\partial h_i} = -\sum_{j \neq i} \frac{\partial w_j}{\partial h_i}.
\end{equation}
\end{proof}

\subsection{Computational Algorithm}
In this section, based on the geometric variational principle established in Theorem \ref{thm:3}, we propose a computational algorithm to compute the semi-discrete optimal transport map on a compact hyperbolic surface. 
The energy $ E(\varphi) $ is strictly convex on the admissible height space $ \Phi_0 $, so we can optimize and solve it using Newton's method.

Let $ M = \mathbb{H}^2 / \Gamma $ be a compact hyperbolic surface, $ \sigma_M $ be the Riemannian measure on $ M $, and $ \nu = \sum_{i=1}^{k} \nu_i \delta_{p_i} $ be a discrete measure that satisfies $ \min_{1 \leq i \leq k} \nu_i > 0 $, $ \{ p_1, p_2, \dots, p_k \} \subset M $, and $ \sigma_M(M) = \sum_{i=1}^{k} \nu_i $. 
By selecting a point $ p \in M $, we can compute a Dirichlet region $ D(p) $ on $ \mathbb{H}^2 $ such that $ \pi_{\Gamma}((0, 0, 1)) = p $, so $ D(p) $ is a fundamental domain for the group action of $ \Gamma $ on $ \mathbb{H}^2 $. 
Let $ x_1, x_2, \dots, x_k \in \overline{D(p)} $ such that $ \pi_{\Gamma}(x_i) = p_i, \forall 1 \leq i \leq k $. 
Then, for any $ \varphi = (\varphi_1, \varphi_2, \dots, \varphi_k) \in \mathbb{R}^k $, we define the following notation:

\begin{enumerate}
\item Radial vector $ \rho = (\rho_1, \rho_2, \dots, \rho_k) $, $ \rho_i = e^{-\varphi_i}, \forall 1 \leq i \leq k $, and geodesic radius vector $ r = (r_1, r_2, \dots, r_k) $, $ \varphi_i = \ln \cosh r_i, \forall 1 \leq i \leq k $;

\item Convex hull $ \mathcal{C}_{\varphi} = \text{Conv}\{\rho_i \gamma x_i | 1 \leq i \leq k, \gamma \in \Gamma\} $;

\item Upper envelope $ \mathcal{U}_{\varphi} = \text{Env}\{\pi_{i,\gamma}(y) = \frac{-1}{\rho_i \langle \gamma x_i, y \rangle_H}  | 1 \leq i \leq k, \gamma \in \Gamma \} $;

\item Hyperbolic weighted Delaunay triangulation $ \mathcal{T}_{\varphi} = \{\gamma x_i | 1 \leq i \leq k, \gamma \in \Gamma\} $;

\item Subnormal cell decomposition $ \mathcal{S}_{\varphi} = \bigcup_{i=1}^k \bigcup_{\gamma \in \Gamma} \gamma \partial \rho(x_i) $;

\item Hyperbolic power diagram $ \mathcal{D}_{\varphi} = \bigcup_{i=1}^k \bigcup_{\gamma \in \Gamma} \gamma U_i(\varphi) $, where $ U_i(\varphi) = \{ y \in \mathbb{H}^m \mid \rho_i \langle x_i, y \rangle_H \geq \rho_j \langle \gamma x_j, y \rangle_H, \forall 1 \leq j \leq k, \forall \gamma \in \Gamma \} $, and the corresponding hyperbolic geodesic circle set is $ \{(\gamma x_i, r_i) \mid 1 \leq i \leq k, \gamma \in \Gamma \} $;

\item Surface cell decomposition $ \mathcal{W}_{\varphi} = \bigcup_{i=1}^k W_i(\varphi) $, where $ W_i(\varphi) = \pi_{\Gamma}(U_i(\varphi)), \forall 1 \leq j \leq k $;

\item Cell measure vector $ \omega(\varphi) = (\omega_1(\varphi), \omega_2(\varphi), \dots, \omega_k(\varphi)) $, where $ \omega_i(\varphi) = \sigma_{\mathbb{H}^2}(U_i(\varphi)) = \sigma_M(W_i(\varphi)), \forall 1 \leq j \leq k $.
\end{enumerate}

\subsubsection{Hyperbolic Power Diagram Algorithm}
We now present an algorithm to compute the hyperbolic power diagram $ \mathcal{D} $ and the hyperbolic weighted Delaunay triangulation $ \mathcal{T} $.

Given a family of geodesic circles $ \{(x_i, r_i) \mid x_i \in \mathbb{H}^2, r_i > 0, i \in I\} $, we first use the Lawson edge-flip algorithm \cite{lawson1977software} to compute the convex hull $ \mathcal{C} = \text{Conv} \{ \rho_i x_i \mid \rho_i = \cosh^{-1} r_i, i \in I \} $ in $ \mathbb{R}^3 $. 

In the second step, we use the Legendre dual algorithm to compute the dual mesh $ \mathcal{U} $ of $ \mathcal{C} $, which is the upper envelope of a family of hyperplanes. 
Each vertex $ \rho_i x_i \in \mathcal{C} $ corresponds to a hyperplane $ \pi_i(y) = -(\rho_i \langle x_i, y \rangle_H)^{-1} $ in $ \mathcal{U} $.
Each face $ [\rho_i x_i, \rho_j x_j, \rho_k x_k] \in \mathcal{C} $ corresponds to a vertex $ -(\rho_i \langle x_i, y \rangle_H)^{-1} y $ in $ \mathcal{U} $, where $ y $ is the inward unit normal vector, and it satisfies the following system of linear equations:
\begin{equation}
\label{eq:dual-normal}
\langle y, y \rangle_H = -1, \quad \rho_i \langle x_i, y \rangle_H = \rho_j \langle x_j, y \rangle_H = \rho_k \langle x_k, y \rangle_H.
\end{equation}

Two vertices $ v_1, v_2 \in \mathcal{U} $ are connected by an edge if and only if their dual faces in $ \mathcal{C} $ intersect in an edge.

The third step involves computing the subnormal cell decomposition $ \mathcal{S} $ by radially projecting $ \mathcal{U} $ onto $ \mathbb{H}^2 $. 
According to Proposition \ref{prop:isomorphic}, the hyperbolic power diagram $ \mathcal{D} $ and the subnormal cell decomposition $ \mathcal{S} $ are isomorphic. Therefore, we obtain the hyperbolic power diagram $ \mathcal{D} $. 
Finally, by Proposition \ref{prop:isomorphic-TC}, we compute the hyperbolic weighted Delaunay triangulation $ \mathcal{T} $ by radially projecting $ \mathcal{C} $ onto $ \mathbb{H}^2 $.

\begin{algorithm}
\caption{Hyperbolic Power Diagram}\label{alg:hpd}
\begin{algorithmic}[1]
\Require A set of geodesic circles $ \{ (x_i, r_i) | x_i \in \mathbb{H}^2, r_i > 0, i \in I \} $.
\Ensure Hyperbolic power diagram $ \mathcal{D} $, hyperbolic weighted Delaunay triangulation $ \mathcal{T} $.
\State Compute the radial length set $ \{ \rho_i = \cosh^{-1} r_i | i \in I \} $.
\State Apply the Lawson edge-flip algorithm \cite{lawson1977software} to compute the convex hull $ \mathcal{C} = \text{Conv} \{ \rho_i x_i | i \in I \} $ in $ \mathbb{R}^3 $.
\State Apply the equation \ref{eq:dual-normal} to compute the normal vectors of all faces on $ \mathcal{C} $.
\State Apply the Legendre dual algorithm to compute the upper envelope $ \mathcal{U} $.
\State Compute the hyperbolic power diagram $ \mathcal{D} $ by radial projection.
\State Compute the hyperbolic weighted Delaunay triangulation $ \mathcal{T} $ by radial projection.
\State \Return hyperbolic power diagram $ \mathcal{D} $ and hyperbolic weighted Delaunay triangulation $ \mathcal{T} $.
\end{algorithmic}
\end{algorithm}

\subsubsection{Semi-Discrete Hyperbolic Optimal Transport Map Algorithm}
In this section, we propose an algorithm for the semi-discrete hyperbolic optimal transport map $ T : (M, \sigma_M) \to (M, \nu) $ on compact hyperbolic surfaces. 
Furthermore, this algorithm can be extended to numerical algorithms for hyperbolic optimal transport maps on compact hyperbolic manifolds in higher dimensions.

We first compute a fundamental domain $ D $ on $ \mathbb{H}^2 $ under the action of the group $ \Gamma $ and the generators of $ \Gamma $. 
Let $ p \in M $, we compute a Dirichlet region $ D(p) $ on $ \mathbb{H}^2 $ such that $ \pi_{\Gamma}((0, 0, 1)) = p $. 
We consider a discrete triangulated surface $ M = (V, E, F) $ , and let $ l: E \to \mathbb{R} $ be the edge length function discretized from the Riemannian metric on $ M $. 
For a face $ [v_0, v_1, v_2] \in M $, we parameterize the coordinates of the three vertices $ v_0, v_1, v_2 $ using the edge lengths. Let $ \tau(v_0) = (0, 0, 1) $, and then we calculate the embeddings of the other two vertices in the hyperbolic plane using the following formulas:
\begin{equation}
\begin{split}
\tau(v_1) &= (\sinh l_{01}, 0, \cosh l_{01}), \\
\tau(v_2) &= (\sinh l_{02} \cos \theta_0^{12}, \sinh l_{02} \sin \theta_0^{12}, \cosh l_{02}).
\end{split}
\end{equation}
where $ l_{01} $, $ l_{02} $ are the lengths of the edges $ [v_0v_1] $, $ [v_0v_2] $, respectively, and $ \theta_0^{12} $ is the interior angle at the vertex $ v_0 $. 

If two vertices $ v_i, v_j $ in the face $ [v_i, v_j, v_k] \in M $ have already been embedded, we compute the two intersection points of the two hyperbolic geodesics $ (\tau(v_i), l_{ik}) $ and $ (\tau(v_j), l_{jk}) $ to obtain $ \tau(v_k) $, where $ \tau(v_k) $ should satisfy the orientation condition $ (\tau(v_j) - \tau(v_i)) \otimes (\tau(v_k) - \tau(v_i)) > 0 $, where $ \otimes $ denotes the Lorentzian cross product (p.60 \cite{ratcliffe1994foundations}). 
Therefore, we embed all faces of $ M $ in the appropriate order, thereby obtaining a fundamental domain $ D $ under the action of the group $ \Gamma $ on $ \mathbb{H}^2 $.

Let $ \{ a_1, b_1, \dots, a_g, b_g \} $ be the standard generators of the fundamental group of $ M $, where $ g $ is the genus of $ M $. 
Then, the fundamental domain $ D $ has $ 4g $ geodesic edges: $ \tau(a_1), \tau(b_1), \tau(a_1^{-1}), \tau(b_1^{-1}), \dots, \tau(a_g), \tau(b_g), \tau(a_g^{-1}), \tau(b_g^{-1}) $. 
These edges induce $ 2g $ rigid motions $ \{ \alpha_1, \beta_1, \dots, \alpha_g, \beta_g \} $, where each $ \alpha_i, \beta_i $ maps $ \tau(a_i), \tau(b_i) $ to $ \tau(a_i^{-1}), \tau(b_i^{-1}) $. These $ 2g $ transformations form the set of generators of $ \Gamma $.

Note that the unit disk $ \mathbb{D} = \{ z \in \mathbb{C} : |z| < 1 \} $, equipped with the metric $ ds = \frac{2|dz|}{1 - |z|^2} $, is the Poincaré disk model of the hyperbolic plane. Moreover, it is isometric to $ \mathbb{H}^2 $ through the following stereographic projection:
\begin{equation}
\zeta(z) = \left( \frac{2x}{1 - |z|^2}, \frac{2y}{1 - |z|^2}, \frac{1 + |z|^2}{1 - |z|^2} \right), \quad \forall z = x + iy \in \mathbb{D}.
\end{equation}

We apply the algorithm from \cite{jin2008discrete} to compute the generators of the group $ \Gamma $ in the Poincaré disk model, and then obtain the generators of $ \Gamma $ in $ \mathbb{H}^2 $ through stereographic projection.

Before proceeding to the next step, it is worth mentioning that we do not need to use all the vertices $ \{ \gamma x_i \mid 1 \leq i \leq k, \gamma \in \Gamma \} $. 
According to the universal covering theory, we have $ \mathbb{H}^2 = \bigcup_{\gamma \in \Gamma} \gamma D $, so by transforming the fundamental domain $ D $, we can obtain a finite domain $ D_0 $ on $ \mathbb{H}^2 $, such that $\overline{D} \subset D_0 $, and $ D_0 = \bigcup_{\gamma \in \Gamma_0} \gamma D $, where $ \Gamma_0 $ is a finite subset of $ \Gamma $. Thus, $ E_i^j \subset \Gamma_0, \forall 1 \leq i, j \leq k $. 
Therefore, we only need to compute the hyperbolic power diagram $ \mathcal{D}_\varphi $ and the hyperbolic weighted Delaunay triangulation $ \mathcal{T}_\varphi $ on the finite set of geodesic circles $ \{ (\gamma x_i, r_i) \mid 1 \leq i \leq k, \gamma \in \Gamma_0 \} $.

Next, we apply Newton's method to optimize the energy $ E(\varphi) $. Given the initial vector $ \varphi = (0, 0, \ldots, 0) $, we apply Algorithm \ref{alg:hpd} to compute the hyperbolic power diagram $ \mathcal{D}_\varphi $ and the hyperbolic weighted Delaunay triangulation $ \mathcal{T}_\varphi $ on the geodesic circle set $ \{ (\gamma x_i, r_i) \mid \varphi_i = \ln \cosh r_i, 1 \leq i \leq k, \gamma \in \Gamma_0 \} $. 
We use equation \ref{eq:grad-E} to compute the gradient $ \nabla E $, and equations \ref{eq:omega-grad2d-ij} and \ref{eq:omega-grad2d-i} to compute the Hessian matrix, and solve the linear equation $ H(E) \cdot h = \nabla E $, subject to the constraint $ \sum_{i=1}^k h_i = 0 $. 
Then we update $ \varphi $ as follows:
\begin{equation}
\varphi \leftarrow \varphi + \lambda h.
\end{equation}

The step size parameter $ \lambda $ needs to be chosen as a suitable positive number such that all cells of the hyperbolic power diagram $ \mathcal{D}_{\varphi + \lambda h} $ are non-degenerate.
In the numerical experiments, we first set $ \lambda = 1 $ and compute $ \mathcal{D}_{\varphi + \lambda h} $. 
If any cells of $ \mathcal{D}_{\varphi + \lambda h} $ are degenerate, we halve $ \lambda $, \ie $ \lambda \leftarrow \frac{1}{2} \lambda $, and recompute $ \mathcal{D}_{\varphi + \lambda h} $ until all cells are non-degenerate. 
We repeat this process and stop the iteration when $ \| \omega(\varphi) - \nu \| < \epsilon $, where $ \epsilon > 0 $ is the given error threshold.

Finally, through the covering map $ \pi_\Gamma $, we obtain the surface cell decomposition $ M = \bigcup_{i=1}^k W_i $, where $ W_i = \pi_\Gamma(U_i), \forall U_i \in \mathcal{D}_\varphi $. 
Thus, $ T: W_i \mapsto p_i $ gives the semi-discrete hyperbolic optimal transport map from $ (M, \sigma_M) $ to $ (M, \nu) $.

This algorithm is similar to the damped Newton algorithm proposed by Kitagawa et al. \cite{kitagawa2019convergence}. According to Theorem 4.1 \cite{kitagawa2019convergence}, if the areas of all the cells in the hyperbolic power diagram are greater than or equal to a positive constant in each iteration, then the Kantorovich functional is $ C^{2, \alpha} $, where $ \alpha $ depends on $ \delta $ and other constants. This guarantees the convergence of Newton's method.

\begin{algorithm}
\caption{Semi-discrete Hyperbolic Optimal Transport Map}\label{alg:hotm}
\begin{algorithmic}[1]
\Require A compact hyperbolic surface $ (M, \sigma_M) $, with the target measure $ \nu = \sum_{i=1}^k \nu_i \delta_{p_i} $ satisfying $ \sum_{i=1}^k \nu_i = \sigma_M(M) $, initial step size $\lambda_0$, and the error threshold $ \epsilon > 0 $.
\Ensure The semi-discrete hyperbolic optimal transport map $ T : (M, \sigma_M) \to (M, \nu) $.
\State Compute a fundamental domain $ D $ of the group $ \Gamma $ acting on $ \mathbb{H}^2 $, and the generators of $ \Gamma $.
\State Compute $ x_i \in \overline{D} $, such that $ \pi_\Gamma(x_i) = p_i, \forall 1 \leq i \leq k $.
\State Initialize $ \varphi = (0, 0, \dots, 0) $.
\Repeat
	\State Apply Algorithm \ref{alg:hpd} to compute the hyperbolic power diagram $ \mathcal{D}_\varphi $ and the hyperbolic weighted Delaunay triangulation $ \mathcal{T}_\varphi $ for the set of geodesic circles $ \{ (\gamma x_i, r_i) \mid \varphi_i = \ln \cosh r_i, 1 \leq i \leq k, \gamma \in \Gamma \} $.
	\State Apply equation \ref{eq:triangle-area} to calculate the cell area vector $ \omega(\varphi) = (\omega_1(\varphi), \omega_2(\varphi), \dots, \omega_k(\varphi)) $.
	\State Apply equation \ref{eq:grad-E} to compute the gradient $ \nabla E $.
	\State Apply equations \ref{eq:omega-grad2d-ij} and \ref{eq:omega-grad2d-i} to compute the Hessian matrix $ H(E) $.
	\State Solve the linear system $ H(E) \cdot h = \nabla E $, subject to the constraint $ h_1 + h_2 + \dots + h_k = 0 $.
	\State $\lambda \gets \lambda_0$
	\While{$\exists U_i \in \mathcal{D}_{\varphi + \lambda h}$ that is non-degenerate}
		\State Apply Algorithm \ref{alg:hpd} to compute the hyperbolic power diagram $ \mathcal{D}_{\varphi + \lambda h} $
		\State $\lambda \gets \frac{1}{2} \lambda$
	\EndWhile
	\State $\varphi \gets \varphi + \lambda h$
\Until{$ \| \omega(\varphi) - \nu \| < \epsilon $}
\State \Return semi-discrete hyperbolic optimal transport map $ T : W_i \mapsto p_i $.
\end{algorithmic}
\end{algorithm}

\subsubsection{Hyperbolic Area-Preserving Parametrization Algorithm}

In this section, we propose an algorithm to compute an area-preserving parametrization from a compact surface with genus greater than one to the hyperbolic plane, as described in Algorithm \ref{alg:hap}.  
Let $(\Sigma, d_{\Sigma})$ be a compact metric surface with genus $g > 1$. According to the Uniformization theorem (Theorem 4.4.1 \cite{jost2006compact}), there exists a conformal map $\varphi$ that maps the surface $\Sigma$ conformally onto a compact hyperbolic surface $M$. 
The metric $d_{\Sigma}$ on surface $\Sigma$ is conformal to the hyperbolic metric $d_{M}$ on surface $M$, and the conformal factor of $\varphi$ provides the surface area element measure on $M$.  
From the Gauss-Bonnet theorem (p. 274 \cite{do2016differential}), it is known that the total area of $M$ is $-2\pi\chi(M) = 2\pi(2g - 2)$. To compute an area-preserving parametrization from $\Sigma$ to the hyperbolic plane $\mathbb{H}^2$, we scale the surface $\Sigma$ such that its total area equals $2\pi(2g - 2)$.

Now, let $\Sigma = (V, E, F)$ be a triangular mesh surface, where the vertex set is $V = \{v_1, v_2, \dots, v_k\}$. The conformal map $\varphi : \Sigma \to M$ and the hyperbolic metric $d_M$ on $M$ can be computed using the discrete surface Ricci flow algorithm \cite{jin2008discrete}. Then, we assign an area element to each vertex $v_i$ as follows:
\begin{equation}
\label{eq:point-area-elem}
\nu_i = \frac{1}{3} \sum_{[v_i, v_j, v_k] \in \Sigma} \text{area}([v_i, v_j, v_k]),
\end{equation}
where $[v_i, v_j, v_k]$ is a face of $\Sigma$ that contains the vertex $v_i$.

This gives the discrete surface area element measure on $M$:
\begin{equation}
\label{eq:surface-area-elem}
\nu = \sum_{i=1}^{k} \nu_i \delta_{p_i}, \quad p_i = \varphi(v_i), \quad \forall 1 \leq i \leq k.
\end{equation}

Next, we apply Algorithm \ref{alg:hotm} to compute the semi-discrete hyperbolic optimal transport map $T : (M, d_M) \to (M, \nu)$. 
This gives a cell decomposition of surface $M$, $M = \bigcup_{i=1}^k W_i$, such that the area of each cell is equal to the area element of the corresponding vertex on $\Sigma$. 
By computing the center $q_i$ of each cell $W_i$ and $x_i \in \mathbb{H}^2$, and ensuring that $\pi_\Gamma(x_i) = q_i$, we obtain an area-preserving parametrization from $\Sigma$ to $\mathbb{H}^2$, $T : v_i \mapsto x_i, \forall 1 \leq i \leq k$.  
Furthermore, by applying spherical projection, we can obtain an area-preserving parametrization from $\Sigma$ to the Poincaré disk model $\mathbb{D}$.  In this way, we can compute an area-preserving parametrization from any compact surface of genus greater than 1 to any conformal model on the hyperbolic plane.

\begin{algorithm}
\caption{Hyperbolic Area-Preserving Parametrization}\label{alg:hap}
\begin{algorithmic}[1]
\Require Triangular mesh surface $\Sigma$, vertex set $\{v_1, v_2, \dots, v_k\}$, genus $g > 1$.
\Ensure Hyperbolic area-preserving parametrization.
\State Scale the surface $\Sigma$ such that its total area equals $2\pi(2g - 2)$.
\State Apply the discrete surface Ricci flow algorithm \cite{jin2008discrete} to compute the conformal map $\varphi : \Sigma \to M$ and the hyperbolic metric $d_M$ on $M$.
\State Use equations \ref{eq:point-area-elem} and \ref{eq:surface-area-elem} to compute the discrete measure $\nu = \sum_{i=1}^{k} \nu_i \delta_{p_i}$.
\State Apply Algorithm \ref{alg:hotm} to compute the semi-discrete hyperbolic optimal transport map $T : (M, d_M) \to (M, \nu)$.
\State Compute the center $q_i$ of each cell $W_i$ and $x_i \in \mathbb{H}^2$, such that $\pi_\Gamma(x_i) = q_i$.
\State \Return Hyperbolic area-preserving parametrization $T : v_i \mapsto x_i$.
\end{algorithmic}
\end{algorithm}

\section{Experiments}
\label{experiments} 

\subsection{Experiment Details}
In this section, we present the experiments details and results for evaluation of our method.
We implemented our algorithms in C++ code and all our experiments are conducted on a laptop with an Intel i5 CPU at 2.4 GHz with 16GB RAM. Our method do not require the use of specialized hardware such as GPUs. 
For optimization using Newton's method, we use an initial step size of 0.5 and error threshold of $1e^{-6}$.

\subsection{Results and Discussion}
\subsubsection{Synthetic Data}
We first evaluate our method on some simple synthetic data for easy visualization and comparison. 
The synthetic dataset consists of 61 evenly spaced points on the Poincaré disk as shown in Figure \ref{fig:toy-data}.
We take the hyperbolic surface $M$ to the the convex hull region formed by the set of points $\{ p_i \}$ and let $\Sigma = (V, E, F)$ be a triangular mesh surface of $M$.
We define the source measure $\mu$ to be $\mu(E) = \text{area}_h(E)$ for $E \subseteq M$, where $\text{area}_h(\cdot)$ represents the hyperbolic area function.
The target measure is defined as $\nu = \sum_{i} \nu_i \delta_{p_i}$, where 
\[
\nu_i = \frac{1}{3} \sum_{[p_i, p_j, p_k] \in \Sigma} \text{area}_e([p_i, p_j, p_k]),
\]
$[p_i, p_j, p_k]$ is a face of $\Sigma$ that contains $p_i$, and $\text{area}_e(\cdot)$ is the euclidean area function.

\begin{figure}[h]
\begin{center}
\fbox{\rule{0pt}{2in}
\includegraphics[width=0.8\linewidth]{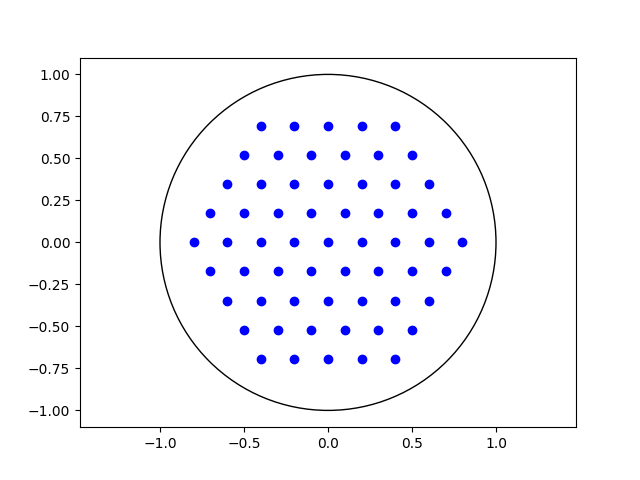}}
\end{center}
   \caption{Synthetic dataset}
\label{fig:toy-data}
\end{figure}

We compute the hyperbolic optimal transport map on this synthetic dataset and obtain the cell decomposition $M = \bigcup_{i=1}^k W_i$ using algorithm \ref{alg:hotm}. 
In order to visualize the effect of computing the OT map in hyperbolic space, we also compute the Euclidean OT map and cell decomposition for the same dataset and compare the results, shown in Figure \ref{fig:hyperbolic-ot} and \ref{fig:euclidean-ot}. 
The cell decomposition is marked by the blue edges and the centroids of each cell are marked in green.
The hyperbolic OT map is then given by the map from each cell  $W_i$  to its corresponding point $p_i$.

\begin{figure}[h!]
\centering
\begin{minipage}{.5\textwidth}
  \centering
  \captionsetup{width=.9\linewidth}
  \fbox{\rule{0pt}{2in}
  \includegraphics[width=.8\linewidth]{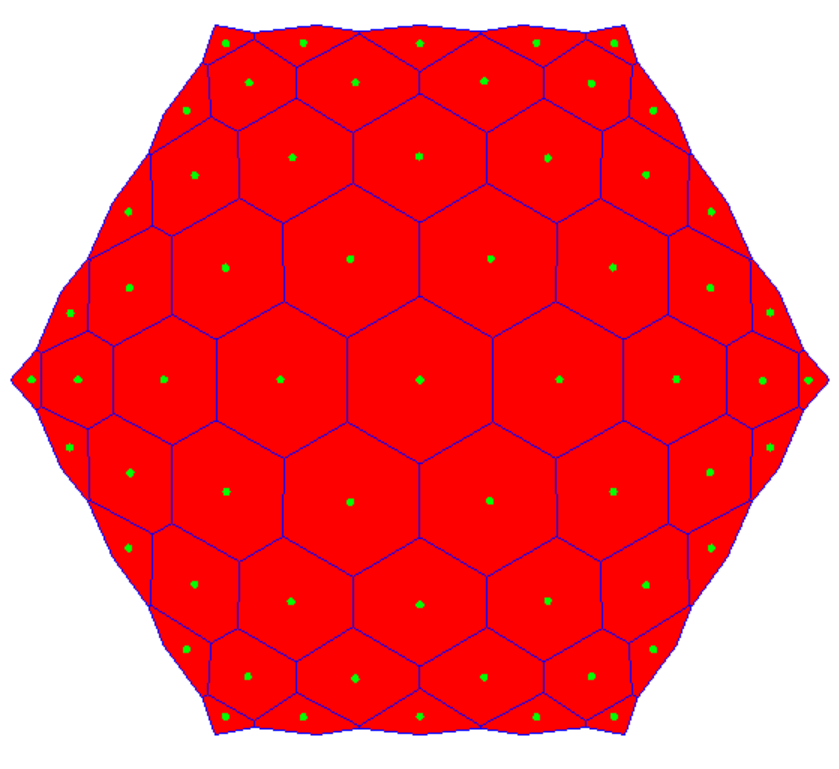}}
  \caption{Hyperbolic OT}
  \label{fig:hyperbolic-ot}
\end{minipage}%
\begin{minipage}{.5\textwidth}
  \centering
  \captionsetup{width=.9\linewidth}
  \fbox{\rule{0pt}{2in}
  \includegraphics[width=.8\linewidth]{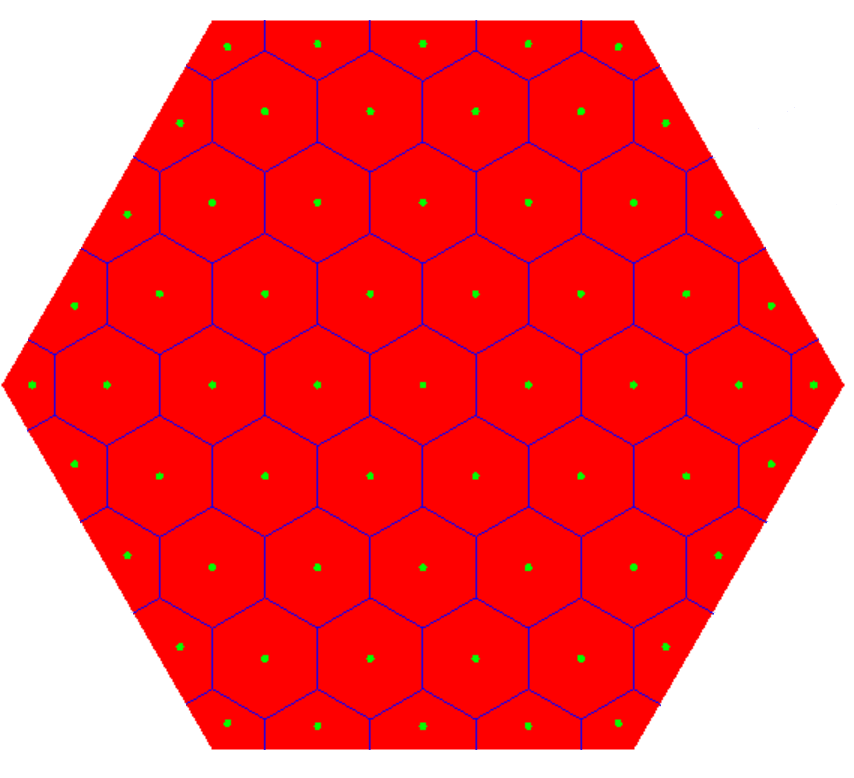}}
  \caption{Euclidean OT}
  \label{fig:euclidean-ot}
\end{minipage}
\end{figure}

The result from the hyperbolic OT map shows a cell decomposition where cells near the origin are larger than those farther away. This effect arises due to the scaling of the hyperbolic metric on the Poincaré disk, which varies with the distance from the origin. Despite their apparent differences in size, each of the 6-sided cells maintains an equal area.
Note that the seemingly non-convex boundary depicted in the figure is an artifact of the metric on the Poincaré disk, where geodesics are circular arcs perpendicular to the boundary of the disk. The boundary sides are in fact hyperbolic geodesics.
In contrast, the cell decomposition resulting from the Euclidean OT map exhibits uniformity throughout the domain, with the exception of the boundary, consistent with the expected behavior according to the definition of the target measure.

To see the effect of a different target measure, we perform another experiment on the synthetic dataset by setting the target measure to use the hyperbolic area instead.
Similarly, we compute the hyperbolic OT map on the dataset and obtain the cell decomposition on the source domain. The results are shown in Figure \ref{fig:diff-targets} where the result using the Euclidean area target measure is shown on the left while the result using hyperbolic area target measure is shown on the right. As before, the centroids of each cell are marked in green.

\begin{figure}[h!]
\begin{center}
\fbox{\rule{0pt}{2in}
\includegraphics[width=0.8\linewidth]{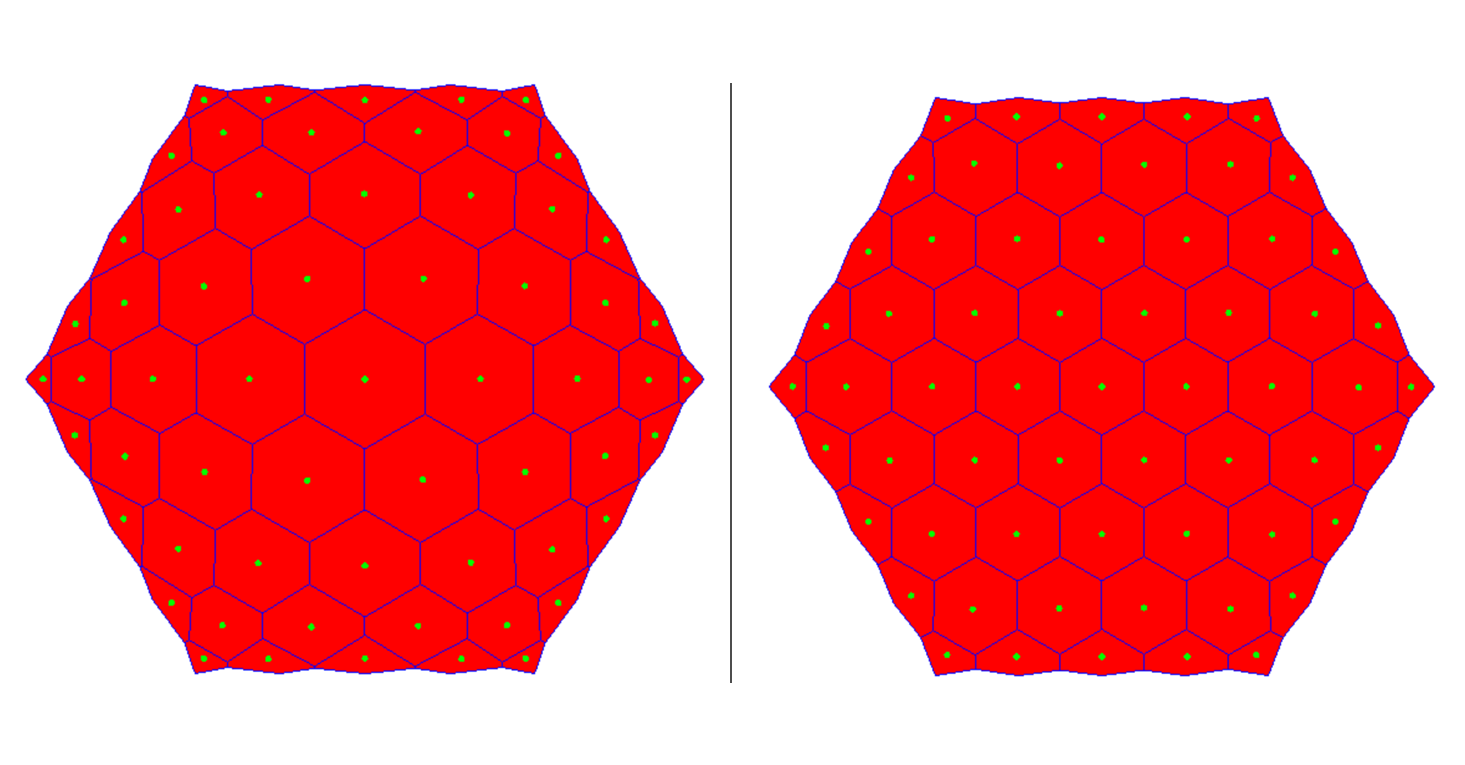}}
\end{center}
   \caption{Comparing different target measures}
\label{fig:diff-targets}
\end{figure}

The results demonstrate that the cells appear more uniform in size, which contrasts with the outcome obtained using the hyperbolic OT map with the Euclidean area function. 
However, this pattern is consistent with the behavior observed in the Euclidean OT case when the Euclidean area function is used. 
This aligns with the expected theoretical results.

\subsubsection{Multi-genus Surfaces}
We also evaluate our method on 3D models to demonstrate the computation of OT maps on multi-genus Riemann surfaces, following the procedure outlined in Algorithm \ref{alg:hap}.
To prepare the data for OT map computation, we first calculate the hyperbolic metric on the surface using the discrete surface Ricci flow algorithm \cite{jin2008discrete}. The process is illustrated in Figure \ref{fig:ricci-flow} using the figure eight model, which is a genus 2 surface. 
We start with a triangular mesh of a compact Riemann surface (top left). 
Next, we apply discrete Ricci flow to compute the hyperbolic metric on the surface, which is conformal to the unit disk (top right). 
We then cut along a set of fundamental group generators of the surface and embed the surface onto the unit disk (bottom left). 
Subsequently, we compute the generators of the Fuchsian group to determine the Möbius transformations that map each side of the boundary to its inverse. Using these transformations, we apply them to the fundamental domain mesh and replicate the mesh to cover the universal covering space (bottom right).

\begin{figure}[h!]
\begin{center}
\fbox{\rule{0pt}{2in}
\includegraphics[width=0.8\linewidth]{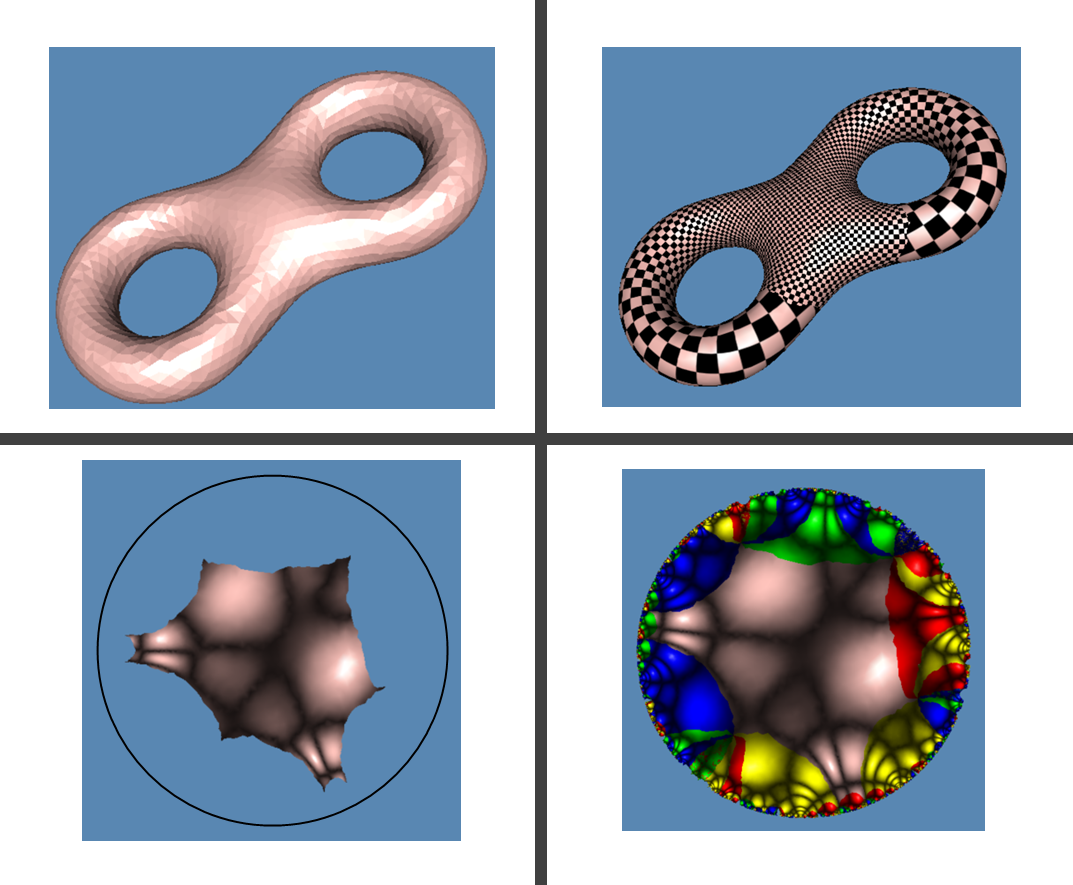}}
\end{center}
   \caption{Discrete surface ricci flow}
\label{fig:ricci-flow}
\end{figure}

Based on the output above, we now apply Algorithm \ref{alg:hotm} to compute the OT map on the surface and obtain the final cell decomposition on the embedded surface and on the universal covering space. 
The source and target measures are defined similarly as in the first synthetic data experiment.
The results are shown in Figure \ref{fig:ucs-cells}, where we have the cell decomposition on the fundamental domain on the left and on the universal covering space on the right. 

\begin{figure}[h!]
\begin{center}
\fbox{\rule{0pt}{2in}
\includegraphics[width=0.8\linewidth]{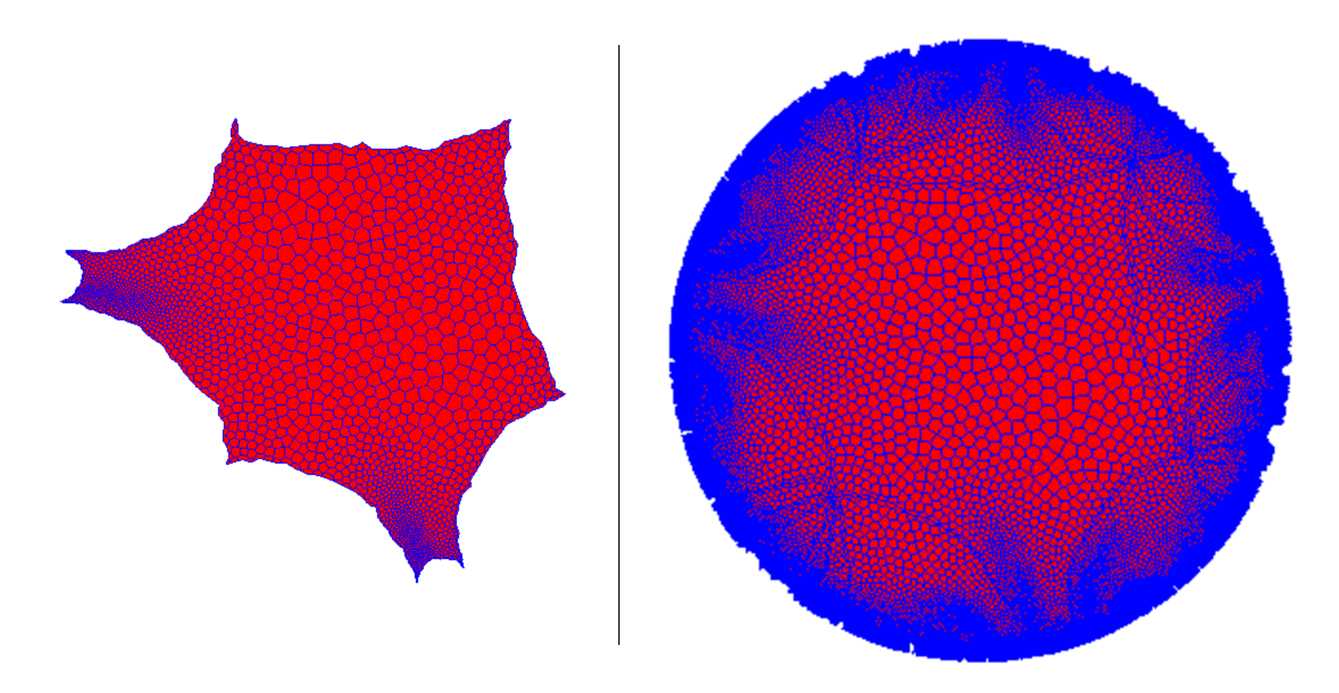}}
\end{center}
   \caption{Cell decomposition on the fundamental domain and the universal covering space}
\label{fig:ucs-cells}
\end{figure}

We then examine the convergence of the error for Newton's method applied to the figure eight model, as illustrated in Figure \ref{fig:error-graph}. In this figure, we plot both the logarithm of the total squared error and the maximum relative error as functions of the iteration number. 
The graph reveals that as the number of iterations increases, the errors decrease at an exponential rate. 
The graph demonstrates that, with increasing iterations, both errors decrease exponentially. Specifically, both the total squared error and the maximum relative error exhibit rapid reductions, highlighting the efficiency of Newton's method in achieving fast convergence. 
This behavior aligns with the theoretical results presented in \cite{kitagawa2019convergence}.
This behavior underscores the ability of the proposed method to quickly refine the solution with each iteration, making it highly efficient for solving the hyperbolic OT problem.

\begin{figure}[h!]
\begin{center}
\fbox{\rule{0pt}{2in}
\includegraphics[width=0.8\linewidth]{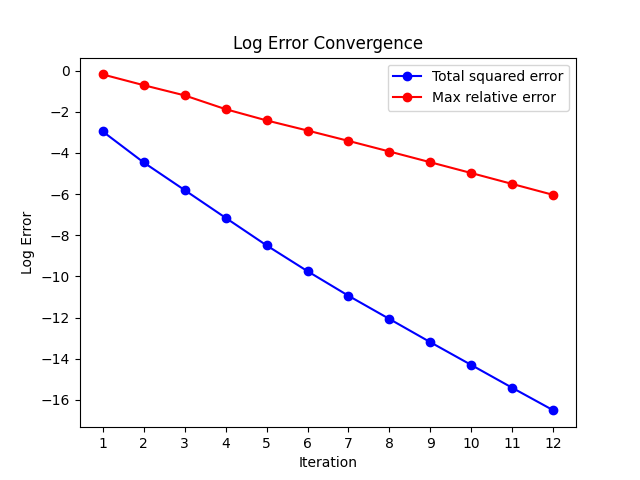}}
\end{center}
   \caption{Log Error Convergence}
\label{fig:error-graph}
\end{figure}

We also compute the hyperbolic OT map on additional surface models, including the amphora model, which is another genus-2 surface with a larger number of vertices, as well as a genus-3 surface to demonstrate the method on models with a higher genus. 
The results are presented in Figure \ref{fig:amphora-genus3}, where the left image shows the original model, the middle image shows the embedded surface on the Poincaré disk, and the right image shows the cell decomposition obtained after computing the hyperbolic OT map.

\begin{figure}[h!]
\begin{center}
\fbox{\rule{0pt}{2in}
\includegraphics[width=0.8\linewidth]{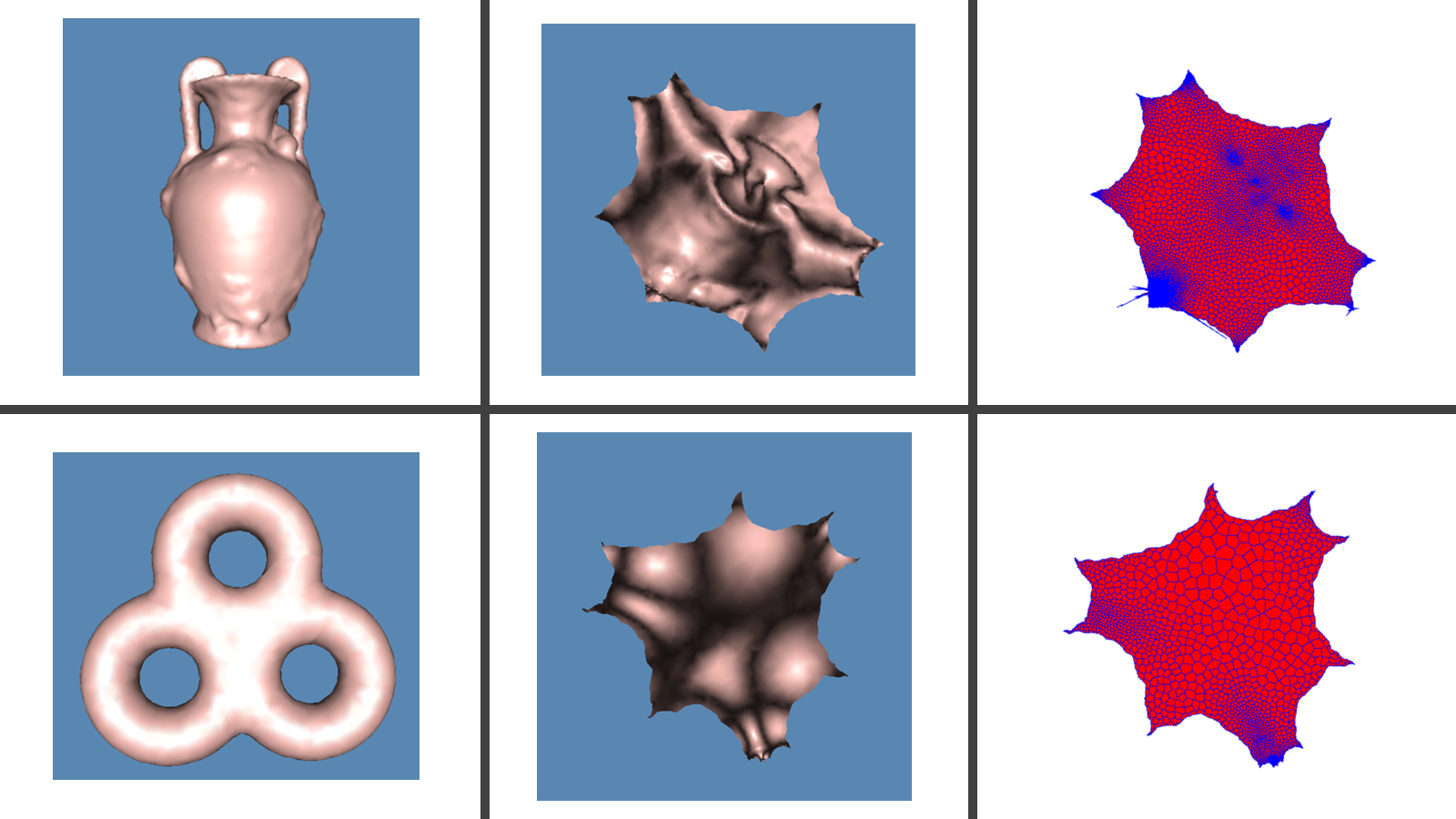}}
\end{center}
   \caption{Results on amphora and genus 3 models}
\label{fig:amphora-genus3}
\end{figure}

In order to investigate the computational efficiency of our proposed method, we evaluate the average execution time per iteration for the method on the three different models. 
We also evaluate the execution times for the Euclidean OT method based on \cite{gu2013variational} on the same data for comparison.
The results are shown in Table \ref{table:1}.

\begin{table}[h!]
\centering
\caption{Execution times per iteration}
\begin{tabular}{c c c c c} 
 \hline
 \multirow{2}{*}{Model}  & \multirow{2}{*}{Genus}  & \multirow{2}{*}{Num vertices} & \multicolumn{2}{c}{Time per iteration (ms)}  \\ [0.5ex] 
 \cline{4-5}
 &&& Hyperbolic OT & Euclidean OT \\
 \hline\hline
 Eight & 2 & 2213 & 19.6 & 155 \\ 
 Amphora & 2 & 10313 & 103 & 1370 \\
 Genus 3 & 3 & 1931 & 16.6 & 146 \\ [1ex] 
 \hline
\end{tabular}
\label{table:1}
\end{table}

The table presents the execution times for both methods across models with different numbers of vertices, all of which are in the order of milliseconds. 
This indicates that our proposed method is highly efficient, even when applied to larger and more complex models. 
The result also reveals a roughly linear relationship between the execution time and the number of vertices for hyperbolic OT. 
Interestingly, the execution times for Euclidean OT are comparable to those of hyperbolic OT, suggesting that both methods exhibit similar scalability when applied to these 3D models. 
This shows that our method for hyperbolic OT does not incur additional computational cost when compared to Euclidean methods.

\section{Conclusion}
\label{conclusion}
In this paper, we presented the semi-discrete hyperbolic optimal transport problem and proposed a method for computing the optimal transport map in hyperbolic space based on the geometric variational principle.
We also implemented a numerical algorithm to compute the hyperbolic OT map and demonstrated its efficacy through experiments on synthetic toy data and 3D mesh data.
Our method is able to compute the OT map without sacrificing precision using entropic regularization and does not require the use of GPUs unlike neural network-based approaches.

We believe our research findings in this paper is potentially useful in many practical applications such as modeling hierarchical data or computing OT maps on multi-genus Riemann surfaces.

Future work on the hyperbolic OT problem can be further explored in the following areas:

\begin{itemize}
\item Implement the algorithm for higher dimensions so that it can be integrated with machine learning models for other applications.

\item Improving the robustness of the numerical algorithm so that it is less sensitive to the input data or the results from the discrete Ricci flow algorithm.

\item Investigate other transport cost functions on general Riemannian manifolds that are compatible with the generalized Legendre duality theory and the geometric variational principle for the OT problem.

\end{itemize}

\bibliographystyle{unsrtnat}
\bibliography{references}

\begin{thebibliography}{52}
\providecommand{\natexlab}[1]{#1}
\providecommand{\url}[1]{\texttt{#1}}
\expandafter\ifx\csname urlstyle\endcsname\relax
  \providecommand{\doi}[1]{doi: #1}\else
  \providecommand{\doi}{doi: \begingroup \urlstyle{rm}\Url}\fi

\bibitem[Galichon(2018)]{galichon2018optimal}
Alfred Galichon.
\newblock \emph{Optimal transport methods in economics}.
\newblock Princeton University Press, 2018.

\bibitem[Peyr{\'e} et~al.(2019)Peyr{\'e}, Cuturi,
  et~al.]{peyre2019computational}
Gabriel Peyr{\'e}, Marco Cuturi, et~al.
\newblock Computational optimal transport: With applications to data science.
\newblock \emph{Foundations and Trends{\textregistered} in Machine Learning},
  11\penalty0 (5-6):\penalty0 355--607, 2019.

\bibitem[Montesuma et~al.(2024)Montesuma, Mboula, and
  Souloumiac]{montesuma2024recent}
Eduardo~Fernandes Montesuma, Fred Maurice~Ngole Mboula, and Antoine Souloumiac.
\newblock Recent advances in optimal transport for machine learning.
\newblock \emph{IEEE Transactions on Pattern Analysis and Machine
  Intelligence}, 2024.

\bibitem[Bonneel and Digne(2023)]{bonneel2023survey}
Nicolas Bonneel and Julie Digne.
\newblock A survey of optimal transport for computer graphics and computer
  vision.
\newblock In \emph{Computer Graphics Forum}, volume~42, pages 439--460. Wiley
  Online Library, 2023.

\bibitem[Chen et~al.(2021)Chen, Georgiou, and Pavon]{chen2021optimal}
Yongxin Chen, Tryphon~T Georgiou, and Michele Pavon.
\newblock Optimal transport in systems and control.
\newblock \emph{Annual Review of Control, Robotics, and Autonomous Systems},
  4\penalty0 (1):\penalty0 89--113, 2021.

\bibitem[Gromov(1987)]{gromov1987hyperbolic}
Mikhael Gromov.
\newblock Hyperbolic groups.
\newblock In \emph{Essays in group theory}, pages 75--263. Springer, 1987.

\bibitem[Jost(2006)]{jost2006compact}
Jurgen Jost.
\newblock \emph{Compact riemann surfaces}.
\newblock Springer, 2006.

\bibitem[Zeng(2022)]{zeng2022geometric}
Peng Zeng.
\newblock \emph{Geometric variational principles and numerical methods for
  optimal transport}.
\newblock PhD thesis, Tsinghua University, 2022.

\bibitem[Gu et~al.(2016)Gu, Luo, Sun, and Yau]{gu2013variational}
Xianfeng Gu, Feng Luo, Jian Sun, and Shing-Tung Yau.
\newblock Variational principles for minkowski type problems, discrete optimal
  transport, and discrete monge-ampere equations.
\newblock \emph{Asian Journal of Mathematics}, 2016.

\bibitem[Cui et~al.(2019)Cui, Qi, Wen, Lei, Li, Zhang, and
  Gu]{cui2019spherical}
Li~Cui, Xin Qi, Chengfeng Wen, Na~Lei, Xinyuan Li, Min Zhang, and Xianfeng Gu.
\newblock Spherical optimal transportation.
\newblock \emph{Computer-Aided Design}, 115:\penalty0 181--193, 2019.

\bibitem[Alexandrov(2005)]{alexandrov2005convex}
Alexandr~D Alexandrov.
\newblock \emph{Convex polyhedra}, volume 109.
\newblock Springer, 2005.

\bibitem[Santambrogio(2015)]{santambrogio2015optimal}
Filippo Santambrogio.
\newblock \emph{Optimal transport for applied mathematicians}, volume~87.
\newblock Springer, 2015.

\bibitem[Kantorovich(1948)]{kantorovich1948on}
L.~V. Kantorovich.
\newblock On a problem of monge.
\newblock \emph{Uspekhi Mat. Nauk}, 3:\penalty0 225–226, 1948.

\bibitem[Brenier(1991)]{brenier1991polar}
Yann Brenier.
\newblock Polar factorization and monotone rearrangement of vector-valued
  functions.
\newblock \emph{Communications on pure and applied mathematics}, 44\penalty0
  (4):\penalty0 375--417, 1991.

\bibitem[Kantorovich(2006)]{kantorovich2006problem}
LK~Kantorovich.
\newblock On a problem of monge.
\newblock \emph{Journal of Mathematical Sciences}, 133\penalty0 (4), 2006.

\bibitem[Benamou and Brenier(2000)]{benamou2000computational}
Jean-David Benamou and Yann Brenier.
\newblock A computational fluid mechanics solution to the monge-kantorovich
  mass transfer problem.
\newblock \emph{Numerische Mathematik}, 84\penalty0 (3):\penalty0 375--393,
  2000.

\bibitem[Arjovsky et~al.(2017)Arjovsky, Chintala, and
  Bottou]{arjovsky2017wasserstein}
Martin Arjovsky, Soumith Chintala, and L{\'e}on Bottou.
\newblock Wasserstein generative adversarial networks.
\newblock In \emph{International conference on machine learning}, pages
  214--223. PMLR, 2017.

\bibitem[Xie et~al.(2020)Xie, Wang, Wang, and Zha]{xie2020fast}
Yujia Xie, Xiangfeng Wang, Ruijia Wang, and Hongyuan Zha.
\newblock A fast proximal point method for computing exact wasserstein
  distance.
\newblock In \emph{Uncertainty in artificial intelligence}, pages 433--453.
  PMLR, 2020.

\bibitem[Cuturi(2013)]{cuturi2013sinkhorn}
Marco Cuturi.
\newblock Sinkhorn distances: Lightspeed computation of optimal transport.
\newblock \emph{Advances in neural information processing systems}, 26, 2013.

\bibitem[Pooladian et~al.(2023)Pooladian, Divol, and
  Niles-Weed]{pooladian2023minimax}
Aram-Alexandre Pooladian, Vincent Divol, and Jonathan Niles-Weed.
\newblock Minimax estimation of discontinuous optimal transport maps: The
  semi-discrete case.
\newblock In \emph{International Conference on Machine Learning}, pages
  28128--28150. PMLR, 2023.

\bibitem[Rabin et~al.(2011)Rabin, Peyr{\'e}, Delon, and
  Bernot]{rabin2011wasserstein}
Julien Rabin, Gabriel Peyr{\'e}, Julie Delon, and Marc Bernot.
\newblock Wasserstein barycenter and its application to texture mixing.
\newblock In \emph{International conference on scale space and variational
  methods in computer vision}, pages 435--446. Springer, 2011.

\bibitem[Bonneel et~al.(2015)Bonneel, Rabin, Peyr{\'e}, and
  Pfister]{bonneel2015sliced}
Nicolas Bonneel, Julien Rabin, Gabriel Peyr{\'e}, and Hanspeter Pfister.
\newblock Sliced and radon wasserstein barycenters of measures.
\newblock \emph{Journal of Mathematical Imaging and Vision}, 51:\penalty0
  22--45, 2015.

\bibitem[Liutkus et~al.(2019)Liutkus, Simsekli, Majewski, Durmus, and
  St{\"o}ter]{liutkus2019sliced}
Antoine Liutkus, Umut Simsekli, Szymon Majewski, Alain Durmus, and
  Fabian-Robert St{\"o}ter.
\newblock Sliced-wasserstein flows: Nonparametric generative modeling via
  optimal transport and diffusions.
\newblock In \emph{International Conference on machine learning}, pages
  4104--4113. PMLR, 2019.

\bibitem[Makkuva et~al.(2020)Makkuva, Taghvaei, Oh, and
  Lee]{makkuva2020optimal}
Ashok Makkuva, Amirhossein Taghvaei, Sewoong Oh, and Jason Lee.
\newblock Optimal transport mapping via input convex neural networks.
\newblock In \emph{International Conference on Machine Learning}, pages
  6672--6681. PMLR, 2020.

\bibitem[Amos et~al.(2017)Amos, Xu, and Kolter]{amos2017input}
Brandon Amos, Lei Xu, and J~Zico Kolter.
\newblock Input convex neural networks.
\newblock In \emph{International Conference on Machine Learning}, pages
  146--155. PMLR, 2017.

\bibitem[Korotin et~al.(2021{\natexlab{a}})Korotin, Egiazarian, Asadulaev,
  Safin, and Burnaev]{korotin2021wasserstein}
Alexander Korotin, Vage Egiazarian, Arip Asadulaev, Alexander Safin, and Evgeny
  Burnaev.
\newblock Wasserstein-2 generative networks.
\newblock In \emph{International Conference on Learning Representations},
  2021{\natexlab{a}}.
\newblock URL \url{https://openreview.net/forum?id=bEoxzW_EXsa}.

\bibitem[Rout et~al.(2021)Rout, Korotin, and Burnaev]{rout2021generative}
Litu Rout, Alexander Korotin, and Evgeny Burnaev.
\newblock Generative modeling with optimal transport maps.
\newblock In \emph{International Conference on Learning Representations}, 2021.

\bibitem[Gushchin et~al.(2023)Gushchin, Kolesov, Korotin, Vetrov, and
  Burnaev]{gushchin2023entropic}
Nikita Gushchin, Alexander Kolesov, Alexander Korotin, Dmitry~P Vetrov, and
  Evgeny Burnaev.
\newblock Entropic neural optimal transport via diffusion processes.
\newblock \emph{Advances in Neural Information Processing Systems},
  36:\penalty0 75517--75544, 2023.

\bibitem[Korotin et~al.(2021{\natexlab{b}})Korotin, Li, Genevay, Solomon,
  Filippov, and Burnaev]{korotin2021neural}
Alexander Korotin, Lingxiao Li, Aude Genevay, Justin~M Solomon, Alexander
  Filippov, and Evgeny Burnaev.
\newblock Do neural optimal transport solvers work? a continuous wasserstein-2
  benchmark.
\newblock \emph{Advances in neural information processing systems},
  34:\penalty0 14593--14605, 2021{\natexlab{b}}.

\bibitem[An et~al.(2019)An, Guo, Lei, Luo, Yau, and Gu]{an2019ae}
Dongsheng An, Yang Guo, Na~Lei, Zhongxuan Luo, Shing-Tung Yau, and Xianfeng Gu.
\newblock Ae-ot: A new generative model based on extended semi-discrete optimal
  transport.
\newblock \emph{International Conference on Learning Representations}, 2019.

\bibitem[An et~al.(2020)An, Guo, Zhang, Qi, Lei, and Gu]{an2020ae}
Dongsheng An, Yang Guo, Min Zhang, Xin Qi, Na~Lei, and Xianfang Gu.
\newblock Ae-ot-gan: Training gans from data specific latent distribution.
\newblock In \emph{European Conference on Computer Vision}, pages 548--564.
  Springer, 2020.

\bibitem[Hamfeldt and Turnquist(2022)]{hamfeldt2022convergence}
Brittany~Froese Hamfeldt and Axel~GR Turnquist.
\newblock A convergence framework for optimal transport on the sphere.
\newblock \emph{Numerische Mathematik}, 151\penalty0 (3):\penalty0 627--657,
  2022.

\bibitem[Quellmalz et~al.(2023)Quellmalz, Beinert, and
  Steidl]{quellmalz2023sliced}
Michael Quellmalz, Robert Beinert, and Gabriele Steidl.
\newblock Sliced optimal transport on the sphere.
\newblock \emph{Inverse Problems}, 39\penalty0 (10):\penalty0 105005, 2023.

\bibitem[Nickel and Kiela(2017)]{nickel2017poincare}
Maximillian Nickel and Douwe Kiela.
\newblock Poincar{\'e} embeddings for learning hierarchical representations.
\newblock \emph{Advances in neural information processing systems}, 30, 2017.

\bibitem[Balazevic et~al.(2019)Balazevic, Allen, and
  Hospedales]{balazevic2019multi}
Ivana Balazevic, Carl Allen, and Timothy Hospedales.
\newblock Multi-relational poincar{\'e} graph embeddings.
\newblock \emph{Advances in neural information processing systems}, 32, 2019.

\bibitem[Alvarez-Melis et~al.(2020)Alvarez-Melis, Mroueh, and
  Jaakkola]{alvarez2020unsupervised}
David Alvarez-Melis, Youssef Mroueh, and Tommi Jaakkola.
\newblock Unsupervised hierarchy matching with optimal transport over
  hyperbolic spaces.
\newblock In \emph{International Conference on Artificial Intelligence and
  Statistics}, pages 1606--1617. PMLR, 2020.

\bibitem[Ganea et~al.(2018)Ganea, B{\'e}cigneul, and
  Hofmann]{ganea2018hyperbolic}
Octavian Ganea, Gary B{\'e}cigneul, and Thomas Hofmann.
\newblock Hyperbolic neural networks.
\newblock \emph{Advances in neural information processing systems}, 31, 2018.

\bibitem[Khrulkov et~al.(2020)Khrulkov, Mirvakhabova, Ustinova, Oseledets, and
  Lempitsky]{khrulkov2020hyperbolic}
Valentin Khrulkov, Leyla Mirvakhabova, Evgeniya Ustinova, Ivan Oseledets, and
  Victor Lempitsky.
\newblock Hyperbolic image embeddings.
\newblock In \emph{Proceedings of the IEEE/CVF conference on computer vision
  and pattern recognition}, pages 6418--6428, 2020.

\bibitem[Ermolov et~al.(2022)Ermolov, Mirvakhabova, Khrulkov, Sebe, and
  Oseledets]{ermolov2022hyperbolic}
Aleksandr Ermolov, Leyla Mirvakhabova, Valentin Khrulkov, Nicu Sebe, and Ivan
  Oseledets.
\newblock Hyperbolic vision transformers: Combining improvements in metric
  learning.
\newblock In \emph{Proceedings of the IEEE/CVF conference on computer vision
  and pattern recognition}, pages 7409--7419, 2022.

\bibitem[Desai et~al.(2023)Desai, Nickel, Rajpurohit, Johnson, and
  Vedantam]{desai2023hyperbolic}
Karan Desai, Maximilian Nickel, Tanmay Rajpurohit, Justin Johnson, and
  Shanmukha~Ramakrishna Vedantam.
\newblock Hyperbolic image-text representations.
\newblock In \emph{International Conference on Machine Learning}, pages
  7694--7731. PMLR, 2023.

\bibitem[Chami et~al.(2020)Chami, Wolf, Juan, Sala, Ravi, and
  R{\'e}]{chami2020low}
Ines Chami, Adva Wolf, Da-Cheng Juan, Frederic Sala, Sujith Ravi, and
  Christopher R{\'e}.
\newblock Low-dimensional hyperbolic knowledge graph embeddings.
\newblock In \emph{Proceedings of the 58th Annual Meeting of the Association
  for Computational Linguistics}, pages 6901--6914. Association for
  Computational Linguistics, 2020.

\bibitem[Benedetti and Petronio(1992)]{benedetti1992lectures}
Riccardo Benedetti and Carlo Petronio.
\newblock \emph{Lectures on hyperbolic geometry}.
\newblock Springer Science \& Business Media, 1992.

\bibitem[Fillastre(2013)]{fillastre2013fuchsian}
Fran{\c{c}}ois Fillastre.
\newblock Fuchsian convex bodies: basics of brunn--minkowski theory.
\newblock \emph{Geometric and Functional Analysis}, 23\penalty0 (1):\penalty0
  295--333, 2013.

\bibitem[Bertrand(2014)]{bertrand2014prescription}
J{\'e}r{\^o}me Bertrand.
\newblock Prescription of gauss curvature on compact hyperbolic orbifolds.
\newblock \emph{Discrete and Continuous Dynamical Systems-Series A}, 2014.

\bibitem[Schneider(2013)]{schneider2013convex}
Rolf Schneider.
\newblock \emph{Convex bodies: the Brunn--Minkowski theory}, volume 151.
\newblock Cambridge university press, 2013.

\bibitem[Lee(2018)]{lee2018introduction}
John~M Lee.
\newblock \emph{Introduction to Riemannian manifolds}, volume~2.
\newblock Springer, 2018.

\bibitem[Boumal(2023)]{boumal2023introduction}
Nicolas Boumal.
\newblock \emph{An introduction to optimization on smooth manifolds}.
\newblock Cambridge University Press, 2023.

\bibitem[Lawson(1977)]{lawson1977software}
Charles~L Lawson.
\newblock Software for c1 interpolation.
\newblock In \emph{Symp. on Mathematical Software}, number JPL-PUBL-77-30,
  1977.

\bibitem[Ratcliffe et~al.(1994)Ratcliffe, Axler, and
  Ribet]{ratcliffe1994foundations}
John~G Ratcliffe, Sheldon Axler, and Kenneth~A Ribet.
\newblock \emph{Foundations of hyperbolic manifolds}, volume 149.
\newblock Springer, 1994.

\bibitem[Jin et~al.(2008)Jin, Kim, Luo, and Gu]{jin2008discrete}
Miao Jin, Junho Kim, Feng Luo, and Xianfeng Gu.
\newblock Discrete surface ricci flow.
\newblock \emph{IEEE Transactions on Visualization and Computer Graphics},
  14\penalty0 (5):\penalty0 1030--1043, 2008.

\bibitem[Kitagawa et~al.(2019)Kitagawa, M{\'e}rigot, and
  Thibert]{kitagawa2019convergence}
Jun Kitagawa, Quentin M{\'e}rigot, and Boris Thibert.
\newblock Convergence of a newton algorithm for semi-discrete optimal
  transport.
\newblock \emph{Journal of the European Mathematical Society}, 21\penalty0
  (9):\penalty0 2603--2651, 2019.

\bibitem[Do~Carmo(2016)]{do2016differential}
Manfredo~P Do~Carmo.
\newblock \emph{Differential geometry of curves and surfaces: revised and
  updated second edition}.
\newblock Courier Dover Publications, 2016.

\end{thebibliography}

%
%

\end{document}